\documentclass[10pt]{article}
\usepackage{enumerate}
\usepackage{amsmath,amssymb}
\usepackage{caption}
\usepackage{subcaption}
\usepackage[usenames,dvipsnames]{color}
\usepackage{bm}
\usepackage{ying}
\usepackage{multirow}
\usepackage{rotating}
\usepackage{fullpage}
\usepackage{authblk}
\usepackage{enumerate}
\usepackage{geometry}
\usepackage{pifont}
\usepackage{makecell}
\usepackage{float} 

\usepackage{setspace}
\let\Algorithm\algorithm
\renewcommand\algorithm[1][]{\Algorithm[#1]\setstretch{1.05}}

\usepackage{mathtools}
\mathtoolsset{showonlyrefs}
\allowdisplaybreaks

\usepackage[indent=15pt]{parskip}

\usepackage{graphicx,amssymb}
\usepackage{xcolor}

\usepackage[colorlinks,
            linkcolor=red,
            anchorcolor=blue,
            citecolor=blue
            ]{hyperref}
\usepackage{algorithm}
\usepackage{algorithmic}

\let\hat\widehat
\let\tilde\widetilde

\def\crossmark{\textcolor{Maroon}{\ding{55}}}
 
\def \iid {\stackrel{\textnormal{i.i.d.}}{\sim}}

\def \defn {\,:=\,}

\def \ndim {{\textnormal{N-dim}}}
\def \hgamma {\widehat{\Gamma}}
\def \hmu {\widehat{\mu}}
\def \eps {\varepsilon}
\def \vs {V_{\rm s}}
\def \vp {V_{\rm p}}
\def \vh {V_{\rm h}}

\def \Gams {\Gamma^{\rm s}}
\def \Gamp {\Gamma^{\rm p}}

\def \dpi {{\pi^{\dagger}}}
\def \xat {\{X_t,A_t,Y_t\}_{t=1}^T}
\def \afactor {{(\log T)^{\alpha/2}}}
\def \worst {{\textnormal{worst}}}
\def \greedy {{\textnormal{greedy}}}

\usepackage{tikz,tabularx}
\usetikzlibrary{arrows}
\usetikzlibrary{tikzmark, positioning, fit, shapes.misc}

\makeatletter
\newcommand{\pushright}[1]{\ifmeasuring@#1\else\omit\hfill$\displaystyle#1$\fi\ignorespaces}
\newcommand{\pushleft}[1]{\ifmeasuring@#1\else\omit$\displaystyle#1$\hfill\fi\ignorespaces}
\makeatother
 
\def\##1\#{\begin{align}#1\end{align}}
\def\$#1\${\begin{align*}#1\end{align*}}

\definecolor{darkred}{rgb}{.75,0,0}
\usepackage{color-edits} 
\addauthor[Ying]{ying}{darkred}
\addauthor[Zhimei]{zr}{orange}

\makeatletter
\newcommand{\printfnsymbol}[1]{%
  \textsuperscript{\@fnsymbol{#1}}%
}
\makeatother

\usepackage{xcolor}

\title{Policy learning ``without'' overlap: \\ 
Pessimism and generalized empirical Bernstein's inequality}
\author[1]{Ying Jin}
\author[2]{Zhimei Ren\printfnsymbol{1}}
\author[3]{Zhuoran Yang}
\author[4]{Zhaoran Wang}
\affil[1]{Data Science Initiative and Department of Health Care Policy, Harvard University}
\affil[2]{Department of Statistics and Data Science, University of Pennsylvania}
\affil[3]{Department of Statistics and Data Science, Yale University}
\affil[4]{Department of Industrial Engineering \& Management Sciences, Northwestern University}
\date{}

\begin{document}

\maketitle

\begin{abstract}
  This paper studies offline policy learning, which aims at 
  utilizing observations collected a priori 
  (from either fixed or adaptively evolving behavior policies) to 
  learn an optimal individualized decision rule 
  that  achieves the best overall outcomes for a given population.   
  Existing policy learning methods  
  rely on a uniform overlap assumption, i.e., 
  the propensities of exploring 
  \emph{all} actions for \emph{all} individual characteristics 
  must be lower bounded; put differently,
  the performance of the existing methods depends
  on the worst-case propensity in the offline dataset.
  As one has no control over the data collection process, 
  this assumption can be unrealistic in many situations, 
  especially when  the behavior policies are allowed to evolve over time with diminishing 
  propensities for certain actions. 

  In this paper, we propose 
  Pessimistic Policy Learning (PPL), a new algorithm 
  that optimizes  lower confidence bounds (LCBs) --- instead of 
  point estimates --- of the  policy values. 
  The LCBs are constructed using knowledge of the behavior policies 
  for collecting the offline data. 
  Without assuming any uniform overlap condition, 
  we establish a data-dependent upper bound for the suboptimality 
  of our algorithm, which only depends on 
  (i) the overlap for the \emph{optimal} policy, and (ii) the complexity 
  of the policy class we optimize over.    
  As an implication, for adaptively collected data, 
  we ensure efficient policy learning as long as 
  the propensities for optimal actions are lower bounded over time, 
  while those for suboptimal ones are allowed to diminish arbitrarily fast. 
  In our theoretical analysis,
  we develop a new self-normalized type concentration inequality 
  for inverse-propensity-weighting estimators, generalizing
  the well-known empirical Bernstein's inequality 
  to unbounded and non-i.i.d.~data. 
  We complement our theory with an efficient 
  optimization algorithm via Majorization-Minimization and policy tree search, 
  as well as extensive simulation studies and real-world applications that 
  demonstrate the efficacy of PPL.
\end{abstract}


\section{Introduction} \label{sec:intro}
 
Policy learning aims at constructing an optimal individualized
decision rule  that achieves the best overall outcomes for a 
given population~\citep{manski2004statistical,hirano2009asymptotics,stoye2009minimax,stoye2012minimax}. 
Due to its 
broad application in healthcare~\citep{murphy2003optimal,kim2011battle,bertsimas2017personalized}, 
advertising~\citep{bottou2013counterfactual,farias2019learning}, 
recommendation systems~\citep{schnabel2016recommendations,li2011unbiased,swaminathan2017off}, etc.,
how to utilize pre-collected data 
for efficient policy learning 
has received extensive attention in various fields
 including economics, causal inference, and statistical machine learning. 

This paper studies policy learning from 
an offline  dataset $\cD = \{ ( X_t , A_t , Y_t )\}_{t=1}^T$  
collected by some agent a priori.
To be precise, $ X_t \sim \PP_X $ are i.i.d.~contexts 
(i.e., individual characteristics), and
$ A_t $ is the action taken according to 
the {\em behavior policy} (also known as the propensity) 
$e_t(x,a):=\PP(A_t=a\given X_t=x)$;
the outcome (reward) is $Y_t = \mu(X_t,A_t)+\epsilon_t$, where 
$\mu(\cdot,\cdot)$ is an unknown function,
and $\{\epsilon_t\}_{t=1}^T$ are independent mean-zero random variables.
We assume knowledge of $e_t(\cdot,\cdot)$ for all $t\in[T]$, but have no control over 
how it is specified by the experimenter. 
%
A policy $\pi$ is a mapping from
the context space to the action space,
and its performance is measured by the {\em policy value}
$Q(\pi)=\EE[\mu(X,\pi(X))]$, i.e., 
the expected reward by executing $\pi$ on 
future contexts. 
Given a policy class $\Pi$ and 
the offline dataset $\cD$, our goal  is to select a policy $\pi \in \Pi$  
whose  performance  is as close as possible 
to the optimal policy $\pi^* = \argmax_{\pi\in\Pi}Q(\pi)$. 
We consider two types of data collection processes that are ubiquitous in
practice. 
\begin{enumerate}[(a)]
\item \textbf{Batched data}: the dataset  is collected by executing a fixed and potentially suboptimal policy, 
so that $\{(X_t,A_t,Y_t)\}_{t=1}^T$ are i.i.d.~tuples. 
Such datasets may come from randomized controlled trials using a pre-fixed treatment rule~\citep{murphy2003optimal,athey2017econometrics}, A/B testing in online platforms~\citep{kohavi2020trustworthy}, 
online recommendation systems with a static rule~\citep{schnabel2016recommendations}, to name a few.
\item \textbf{Adaptively collected data}: the dataset is collected by running an adaptive learning algorithm, where the behavior policies are progressively 
modified to achieve certain goals of the experimenter, e.g., 
improving statistical efficiency in adaptive experiments~\citep{anscombe1963sequential,simon1977adaptive,murphy2005experimental,collins2007multiphase}, or maximizing 
cumulative performance  by bandit algorithms 
in production systems~\citep{lai1985asymptotically,thompson1933likelihood,agrawal2013thompson}. 
This setting strictly generalizes the first one 
by allowing for adaptivity. 
It is more challenging due to data dependence and possibly diminishing exploration, 
and is studied only very recently~\citep{zhan2021policy,bibaut2021risk}.
\end{enumerate}

Existing  algorithms for offline policy learning 
are based on the principle of welfare maximization~\citep{kitagawa2018should}, which 
typically consist of two components: 
evaluation and optimization~\citep[e.g.,][]{athey2021policy,zhou2022offline,zhan2021policy,bibaut2021risk}. 
They first construct policy value estimators $\hat Q(\pi)$ for all $\pi \in \Pi$, 
and then optimize the point estimates
to return $\hat \pi = \argmax_{\pi\in\Pi} \hat Q (\pi)$.  
As we shall discuss in Section~\ref{subsec:challenge}, 
a good performance of such ``greedy'' learners 
relies crucially on accurate 
evaluation of {\em all} policies in $\Pi$. 
This further
requires   that $e_t(x,\pi(x))$
is sufficiently large for all $\pi \in \Pi$: 
intuitively, all actions that are of interest 
are taken sufficiently many times 
in the offline dataset. 
We call it the  \emph{uniform}  overlap assumption, 
as $e_t(x,\pi(x))$ measures the \emph{overlap} between 
the behavior policies and $\pi$. 
Accordingly, 
the worst-case overlap $\inf_{\pi\in\Pi} e_t(x,\pi(x))$ determines 
the learning performance of those methods.

Although uniform overlap is commonly assumed in the policy learning literature, 
it can be violated in many 
practical scenarios.
When learning from  rich policy classes, 
it essentially posits that  $e_t(x,a)$ must be large for all
$(x,a) \in \cX \times \cA$. 
In the adaptive setting, 
imagining the data is collected by  a bandit algorithm
that gradually decreases exploration~\citep{thompson1933likelihood},
the propensities 
$e_t(x, a)$ will diminish  over time for some $(x,a)$ 
pairs, often at a rate faster
than  required by existing policy learning methods~\citep{bembom2008data,hadad2021confidence}.
In the batched setting,  physical constraints (e.g., some 
actions are expensive to take) may lead to   
highly unbalanced sampling of different actions. 
In these cases, the greedy algorithms may only provide vacuous guarantees
due to the lack of uniform overlap. 
It is worth noting several recent efforts to overcome the overlap issue broadly in the causal inference literature; 
they either modify the evaluated policy~\citep{kennedy2019nonparametric} 
or clip tiny propensities~\citep{swaminathan2015batch}, both introducing considerable bias. 
As such, the fundamental role of overlap in policy learning remains an open question.
To fill in this gap, we aim to advance the 
understanding of the fundamental limit of policy learning 
in the absence of uniform overlap.


\subsection{Pessimistic policy learning}

In this paper, we propose a new framework, Pessimistic Policy Learning (PPL),  
to demonstrate that uniform overlap, while sufficient, is 
not a necessary  condition for efficient policy learning from a given policy class. 
PPL learns a high-quality policy whose gap from the optimal policy $\pi^*$
scales with the overlap of $\pi^*$, thereby 
strictly improving upon existing methods and 
automatically adapting to the quality of the offline data. 

The crux of PPL is a 
first-of-the-kind instantiation of the \emph{pessimism} principle~\citep{jin2021pessimism} 
in policy learning.  
PPL optimizes a carefully designed class of 
confidence lower bounds (LCBs) -- instead of point estimates -- of the policy values, where  
the LCBs are constructed to reflect the estimation error 
of commonly-used inverse-propensity-weighting (IPW)-type estimators.

Table~\ref{tab:summary_up_lo} summarizes the key features of PPL 
in comparison to prior works.
The upper half of it contrasts PPL 
with representative works in the policy learning literature. 
Besides providing  a comprehensive treatment for both batched and adaptive data, PPL 
achieves a so-called oracle property where the performance adapts to the overlap of $\pi^*$, 
and the general guarantees do not rely on any overlap condition. By introducing pessimism, PPL uniformly improve upon these methods~\cite{kitagawa2018should,zhan2021policy,swaminathan2015batch} without paying a price.
We note the comparison to~\cite{swaminathan2015batch}  which relies on clipping tiny propensities: 
while clipping is central to facilitate their analysis of IPW-type estimators, 
it loses the oracle adaptivity provided by pessimism. 
This highlights the importance of careful treatment of unbounded propensities.

Moreover, PPL applies 
without structural assumptions on the regression function $\mu(x,a)$.  
This is the benefit of the unbiasedness and consistency of IPW-type estimators, 
which is in contrast to existing algorithms under the hood of pessimism listed 
in the lower half of Table~\ref{tab:summary_up_lo} (these works mostly address
offline reinforcement learning; see Section~\ref{subsec:related_work} for a summary).
This line of works 
exclusively rely on structural assumptions on quantities similar to $\mu(x,a)$ 
to obtain valid quantification of uncertainty. 
However, as pointed by~\citep{dimakopoulou2017estimation}, the risk of 
mis-specification is an important issue in policy learning especially in critical domains, 
where tradeoff between knowledge of behavior policy and structural assumptions becomes necessary.  
As such, our statistical techniques are in sharp contrast to these works.

\begin{table}[ht]\renewcommand{\arraystretch}{1.2}
    \centering
    \resizebox{\textwidth}{!}{\begin{tabular}{c|c|c|c|c|c|c} 
        \toprule
         & \makecell{Adaptive\\data} & \makecell{Structure\\free} & \makecell{Oracle\\property} &\makecell{No\\overlap}  & \makecell{Variance\\adaptive} & \makecell{Knowledge of\\behavior policy} \\
        \midrule
       \cite{kitagawa2018should}
       & \crossmark & \checkmark & \crossmark  & \crossmark & \crossmark &  {assumed} \\
        \cite{swaminathan2015batch}&\crossmark & \checkmark & \crossmark &  \textcolor{gray!40}{\checkmark} &  \checkmark &  {assumed}\\ 
        \cite{zhan2021policy} & \checkmark & \checkmark & \crossmark  & \crossmark & \crossmark &  {assumed}\\ 
        \hline
         \cite{jin2021pessimism}  & \checkmark &  \crossmark & \checkmark & \checkmark&  \checkmark & not required\\ 
         \cite{xie2021bellman} & \crossmark & \crossmark&  \checkmark &  \checkmark &  \checkmark & not required\\ 
         \hline
         \textbf{This work} & \checkmark & \checkmark&\checkmark &\checkmark & \checkmark &  {assumed}\\
        \bottomrule
    \end{tabular}}
    \caption{Summary of related works. In each column, we denote whether the method achieves the property by \checkmark = achievement, \crossmark = no achievement, and \textcolor{gray!40}{\checkmark} = partial achievement. 
    The first three methods are from the policy learning literature, all based on the greedy approach. 
    The fourth and fifth lines are from the offline reinforcement learning literature which implement the pessimism principle by  
    leveraging parametric/functional structures of the outcomes, highlighting the tradeoff between outcome modeling and design knowledge.
    }
    \label{tab:summary_up_lo}
\end{table}
 

\subsection{Preview of contributions}

We summarize the contributions of this work as follows.

\begin{enumerate}[(i)]
\item For both batched and adaptively collected data, we propose the generic recipe for PPL, 
which outputs 
\$
\hat\pi = \argmax_{\pi\in\Pi} \big\{ \hat{Q}(\pi) - R(\pi)\big\},
\$
where $\hat{Q}(\pi)$ is a commonly-used policy value estimator, 
and $R(\pi)$ is a policy-dependent, data-dependent quantity that obeys 
$|\hat{Q}(\pi)-Q(\pi)|\leq R(\pi)$ simultaneously for all $\pi\in\Pi$ with high probability.  

\item We then propose to construct uniformly-valid 
upper bound $R(\pi)$ (hence the LCBs)  via
a novel concentration inequality
    for offline policy evaluation (OPE), which roughly states that with high probability,
    \$
    \big| \hat{Q}(\pi) - Q(\pi) \big| \lesssim \textrm{Complexity}(\Pi) \cdot 
\sqrt{ \hat\Var(\hat{Q}(\pi)) }, \quad \textrm{for     all }\pi\in\Pi.
    \$
    Here $\hat{Q}(\pi)$ is
    the augmented-inverse-propensity-weighting (AIPW)  estimator for $Q(\pi)$,
    and $\hat\Var(\hat{Q}(\pi))$
    is an estimator for the variance of 
    $\hat{Q}(\pi)$ (the definition is given in Section~\ref{subsec:generic_algo}).
    We establish this result using tools from tree symmetrization~\citep{rakhlin2015sequential} and
    self-normalized processes~\citep{de2004self}.
    Its single-policy version can be viewed as
    a generalization of the well-known empirical Bernstein's inequality~\citep{bartlett2005local,bartlett2006convexity,koltchinskii2006local}
    to unbounded and dependent data, which may be of independent interest for other 
    reweighting-based methods~\citep{jin2023sensitivity,gui2022conformalized};
    we discuss the technical novelty in more details in Section~\ref{subsec:pess_idea} and compare our generalized empirical Bernstein's inequality with existing ones in Section~\ref{subsec:related_work}.

    \item With (i) as the key technical tool, 
    we establish the particular benefit of PPL in 
    terms of a generic data-dependent upper bound on its performance, in the form of 
    \$
    Q(\pi^*) - Q(\hat \pi) 
    ~\lesssim ~\textrm{Complexity}(\Pi) \cdot                    \textrm{EstErr}(\pi^*)
        ~\lesssim ~\textrm{Complexity}(\Pi) \cdot \sqrt{\hat\Var(\hat{Q}(\pi^*))},
        \$
    where $\hat\pi$ is the learned policy, and $\pi^*$ is the optimal policy. 
    The first term in 
    the upper bound is a complexity measure of $\Pi$ (see Section~\ref{subsec:nata} for its precise definition). For any $\pi\in\Pi$,
    $\hat\Var(\hat{Q}(\pi))$ is an estimator for the 
    policy evaluation error; 
    it is small as long as the overlap for $\pi$ is large. 
    Thus, the learned policy $\hat{\pi}$ 
    has a small suboptimality gap
    as long as the overlap for $\pi^*$ is sufficiently good, while no assumption is needed for suboptimal policies. This 
    shows an `oracle' property of our approach (see Figure~\ref{fig:overlap}), 
    i.e., it  
    adapts to the overlap of $\pi^*$ in the data, even if $\pi^*$ is not known a priori.

    \item When specialized to batched data, in Section~\ref{sec:fix}, 
    we achieve the standard $O(1/\sqrt{T})$ rate for the learning performance when
    $e(X,\pi^*(X))$ is lower bounded by a constant almost surely. This rate is 
    sharper than the
    existing results and hold under weaker conditions.
    We also provide a matching minimax lower bound in this setting, showing that
    pessimism is the best effort given the dataset.

    \item For adaptively collected data, in Section~\ref{sec:adaptive}, 
    we show that our algorithm achieves
    polynomial learning rate when the propensity for the \emph{optimal} policy
    is decaying at a polynomial rate.
    As a special case, we achieve $O(1/\sqrt{T})$ rate if $e_t(X,\pi^*(X))$ is
    uniformly lower bounded by a constant, while propensities of suboptimal
    policies are allowed to decay at a much faster (and potentially non-polynomial) rate.
    We complement our polynomial upper bound with a matching minimax
    lower bound; in the information-theoretic sense,
    the overlap of the \emph{optimal} policy
    captures the fundamental difficulty in offline policy learning with adaptively collected     data.

    \item  Section~\ref{sub:alg} addresses
    practical challenges in implementing pessimistic policy learning. 
    To tackle the challenge that our objective is non-convex, 
    we devise an efficient approximate optimization algorithm 
    that combines Majorization-Minimization~\citep{van2016gensvm} 
    and policy tree search~\citep{sverdrup2020policytree}. 
    Additionally, we develop a cross-validation strategy specialized to sequential data 
    to choose penalty parameters 
    with favorable performance. 
    Finally, in Section~\ref{sec:exp}, we evaluate our methods via extensive simulation studies 
    and application to 33 real-world datasets. 
    Our results confirm  the benefit of pessimism in 
    policy learning, as well as the efficacy of our optimization algorithms. 
\end{enumerate}

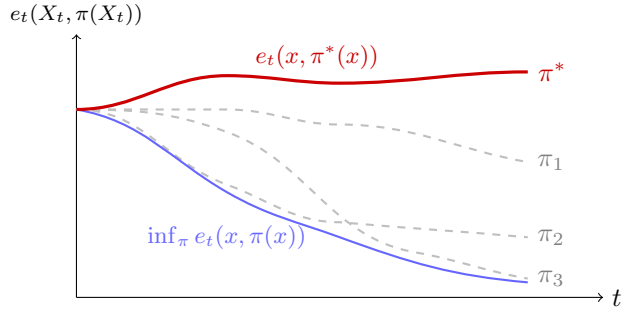
\begin{figure}[ht!] 
    \centering
    \begin{tikzpicture}
        \draw[-> ] (-1, -1) -- (6, -1) node[right] {$t$};
    \draw[-> ] (-1, -1) -- (-1, 2.5) node[above] {\footnotesize$e_t(X_t,\pi(X_t))$};
        
        \coordinate (c1) at (-1,1.5);
        \coordinate (c2) at (1,1.5);
        \coordinate (c3) at (2.5,1.3);
        \coordinate (c4) at (5,0.8);  
        \draw [thick,white!75!black,dashed] (c1) to[out=0,in=180] (c2) to[out=0,in=180](c3)
        to[out=0,in=170] (c4) node[right,white!50!black] { $\pi_1$}; 
    
        \coordinate (a1) at (-1,1.5);
        \coordinate (a2) at (1,0.5);
        \coordinate (a3) at (2.5,0);
        \coordinate (a4) at (5,-0.2);  
        \draw [thick,white!75!black,dashed] (a1) to[out=0,in=160] (a2) to[out=-20,in=180](a3)
        to[out=-5,in=175] (a4) node[right,white!50!black] { $\pi_2$}; 
    
        \coordinate (b1) at (-1,1.5);
        \coordinate (b2) at (1,1.2);
        \coordinate (b3) at (3.5,-0.4);
        \coordinate (b4) at (5,-0.75);  
        \draw [thick,white!75!black,dashed] (b1) to[out=0,in=160] (b2) to[out=-20,in=170](b3)
        to[out=-10,in=170] (b4) node[right,white!50!black] { $\pi_3$}; 
        \coordinate (g1) at (-1,1.5);
        \coordinate (g2) at (1,0.4);
        \coordinate (g3) at (2.5,-0.2);
        \coordinate (g4) at (5,-0.8);  
        \draw [thick,white!40!blue] (g1) to[out=-10,in=150] (g2) to[out=-30,in=160](g3)
        to[out=-20,in=175] (g4) ; 
    
        \coordinate (o1) at (-1,1.5);
        \coordinate (o2) at (1,1.95);
        \coordinate (o3) at (2.5,1.85);
        \coordinate (o4) at (5,2);  
        \draw [very thick,red!80!black] (o1) to[out=0,in=180] (o2) to[out=0,in=180](o3)
        to[out=0,in=180] (o4) node[right] {$ \pi^* $}; 
    
        \node[text=red!80!black] at (2.2,2.2) {\small $  e_t(x,\pi^*(x))$};
        \node[text=white!40!blue] at (1,-0.2) {\small $\inf_\pi e_t(x,\pi(x))$};
    \end{tikzpicture}  
    \caption{An illustration of the oracle property. The performance of 
    pessimistic policy learning only depends on the overlap for the optimal policy $\pi^*$ (red), while 
    the worst-case overlap $\inf_\pi e_t(x,\pi(x))$ is allowed to be small (or vanish over time). 
    See Sections~\ref{sec:fix} and~\ref{sec:adaptive} for more detailed discussion.}
    \label{fig:overlap}
    \end{figure}

\section{Background}
\label{sec:prelim}

Let $\cX$ denote the space of the contexts and
$\cY$ that of the reward.
A contextual bandit model $\cC$ 
is specified by an action set $\cA=\{a_1,\dots,a_K\}$ for $K=|\cA|$,
a context distribution $\PP_X$ and
a set of reward distributions 
$\{\PP_{Y \given X = x, A=a}\}_{x\in \cX, a\in \cA}$. 
We consider policy learning from an offline dataset that is 
collected from $\cC$. 
%

\subsection{Data collecting process}
 
At each time $t=1,2,\dots,T$, a context $X_t$ is independently drawn 
from $\PP_X$. The experimenter then takes an action 
$A_t\in \cA$ according to the behavior policy, and 
receives a reward $Y_t =\mu(X_t,A_t) + \epsilon_t$, 
where $\epsilon_t$ are independent mean-zero noise. 
Throughout, we denote $\cH_t=(X_1,A_1,Y_1,\dots,X_{t-1},A_{t-1},Y_{t-1})$ 
as the history {\em before} time $t$, 
so that $\PP(A_t=a\given X_t=x, \cH_t) = e_t(x,a\given \cH_t)$. 
We use $e_t(\cdot,\cdot\given \cH_t)$ to additionally emphasize  
that the behavior policy is potentially adaptive. 
In the context of causal inference~\citep{imbens2015causal}, $A_t$ can be understood as 
the treatment option taken for $X_t$, 
and the reward is  $Y_t=Y_t(A_t)$, 
where $\{Y_t(a)\}_{a\in \cA}$ are potential outcomes with 
$(X_t,Y_t(a)_{a\in \cA})$ following a population joint distribution, 
and we implicitly assume the conditional independence condition: $Y_t(a)\indep A_t\given X_t,\cH_t$. 
The two types of data generating processes we introduced in Section~\ref{sec:intro} 
are formally defined as follows. 
\begin{enumerate}[(a)]
    \item Batched   data: 
    the behavior policy is fixed and known, where 
    $\PP(A_t=a\given X_t = x,\cH_t) \equiv e(x,a)$ for all $(x,a)\in \cX \times \cA$ 
    and all $t\in [T]$.
    \item Adaptively collected data: 
    at each time $t\in [T]$, the action is taken according to 
    a history-dependent rule, where $\PP(A_t = a \given X_t = x,\cH_t) = e_t(x,a \given \cH_t)$
    for all $(x,a)\in \cX \times \cA$.
    As the behavior policy may depend on previous observations, 
    $\{(X_t,A_t,Y_t)\}_{t=1}^T$ are no longer i.i.d.  
\end{enumerate}

For simplicity, in both cases we write $e_t(x,a\given \cH_t)$ 
as the treatment probability 
at time $t\in[T]$, with the 
convention that $e_t(x,a\given \cH_t)\equiv e(x,a)$ for batched data. 
Throughout, We impose the following minimal assumption
on the data generating process. 

\begin{assumption}\label{assump:data}
The data generating process satisfies the following properties. 
\begin{enumerate}[(a)]
    \item Bounded reward: the rewards satisfy $Y_t \in [0,1]$ almost surely. 
    \item Known behavior policy: the behavior policies $e_t(x,a\given \cH_t)$ for all $(x,a)\in \cX\times \cA$
     are known to the learner. 
\end{enumerate}
\end{assumption}

The bounded reward assumption is standard in the literature  and 
the upper bound $1$ can be simply achieved by normalization.
Assumption~\ref{assump:data}(b) is typically satisfied for data
resulting from running bandit algorithms or online experiments, 
since the agent/experimenter often has access to the policy 
she  uses to collect the data (see e.g.,~\cite{collins2007multiphase,
li2010contextual,offer2021adaptive}).

\subsection{Policy learning and the performance metric}

Let $\Pi$ be a collection of policies, with 
each element in $\Pi$ being a mapping $\pi\colon \cX\to \cA$. 
For any $\pi\in\Pi$, we define its average policy value 
under the contextual bandit model $\cC$ as
\$
Q(\pi ;\cC) = \EE_{\cC}\big[ \mu(X,\pi(X)) \big] = \sum_{a\in \cA} \EE_{\cC}\big[  \ind\{\pi(X)=a\} Y(a)  \big],
\$
where the expectation is taken with respect to
the context distribution. 

The optimal policy $\pi^* \in \Pi$ 
maximizes the policy value, i.e., 
\#\label{eq:def_pi*}
\pi^*(\cC,\Pi) = \argmax_{\pi\in \Pi}~Q(\pi ;\cC).
\#
Our goal is to learn from the offline data a policy $\hat\pi\in \Pi$,
where the performace of $\hat \pi$ is measured by 
its {\em suboptimality}, defined as 
\$
\cL(\hat\pi ; \cC,\Pi) := Q(\pi^*(\cC,\Pi) ;\cC) - Q(\hat\pi ;\cC). 
\$
The above notations emphasize the dependence 
of our performance measure on the underlying contextual bandit $\cC$ 
and policy class $\Pi$. 
In the sequel, we omit $\cC$ and $\Pi$ when 
no confusion arises from the context. 

\subsection{Policy class complexity}
\label{subsec:nata}
We use the Natarajan dimension~\citep{natarajan1989learning} 
to quantify the complexity of a policy
class $\Pi$, which plays an important role 
in characterizing the performance of policy learning algorithms. 

\begin{definition}[Natarajan dimension]\label{def:nat}
Given a policy class $\Pi = \{\pi: \cX \mapsto \cA\}$,
we say $\{x_1,x_2,\ldots,x_m\}$
is shattered by $\Pi$ if there exist two functions 
$f_1,f_2$: $\{x_1,x_2,\ldots,x_m\} \mapsto \cA$
such that the following holds: 
\begin{enumerate}[(1)]
    \item For any $j\in [m]$, $f_{1}(x_j) \neq f_2(x_j)$;
    \item For any subset $S\subseteq [m]$, there exists some $\pi\in\Pi$ 
    such that  $\pi(x_j)=f_1(x_j)$ for all $j\in S$ and 
    $\pi(x_j) = f_2(x_j)$ for all $j\notin S$.   
\end{enumerate} 
The Natarajan dimension of $\Pi$, denoted as $\ndim(\Pi)$, 
is the largest size of a set of points  shattered by $\Pi$.
\end{definition}

The Natarajan dimension is a generalization of the 
well-known Vapnik-Chervonenkis (VC) dimension~\citep{vapnik2015uniform}
in the multi-action classification problem. 
Finite upper bounds for the Natarajan dimension  have been established
for many commonly used policy classes including
linear function classes, reduction trees, 
decision trees, random forests 
and neural networks~\citep{daniely2011multiclass,jin2023upper}.

\subsection{Related work}
\label{subsec:related_work}
 
\noindent\textit{Policy learning with a policy class.}  
This work adds to the literature of policy learning, 
%
where existing works have primarily focused on 
i.i.d.~data collected 
by a fixed behavior policy 
that is known~\citep{zhang2012estimating,zhao2012estimating,swaminathan2015batch,zhou2017residual,kitagawa2018should}, 
or  unknown but estimable~\citep{zhao2015doubly,athey2021policy,zhou2022offline,kallus2018balanced}.
These works all 
assume a uniform overlap condition, i.e., the propensity for all actions 
in the behavior policy is uniformly lower bounded by a constant. 
The overlap assumption is standard in causal inference for 
estimating average treatment effects~\citep{imbens2015causal}, 
and is essential for unknown behavior policy due to 
the necessity of controlling the error in its estimation~\citep{athey2021policy,zhou2022offline}. 
However, we show that when the 
behavior policy is known, the uniform overlap assumption is not necessary 
for efficient policy learning; our approach 
is based on an algorithmic change to the standard greedy approach in this literature. 
Recently, there has been an increasing interest in policy learning from
adaptively collected data~\citep{zhan2021policy,bibaut2021risk}. 
This setting is more challenging because the observations are no longer i.i.d., 
and further, adaptively collected data often 
have diminishing propensities of certain arms, which causes 
huge variance in the IPW-type estimators and translates to 
slow rates for the learning performance. 
To tackle these challenges, existing works assume the 
propensities for \emph{all} actions 
are deterministically lower bounded over time; as we introduced, this 
may be unrealistic for some commonly adopted 
algorithms. 
In contrast to the existing literature, we introduce a novel algorithm 
to policy learning and show  that  
the overlap of the \emph{optimal} policy is sufficient for efficient policy learning. 
Our results improve the learning performance 
and broaden the  scenarios where efficient policy learning is feasible in both settings. \\[-1ex]
 
\noindent\textit{Inference and learning without overlap.}  
This work adds to the literature 
on relaxing the overlap condition, a fundamental assumption in causal inference, 
policy evaluation and learning. 
Earlier works use strategies like sample trimming to rule out extreme 
propensities in treatment effect estimation and study asymptotic inference~\citep{yang2018asymptotic,branson2023causal}. 
There have also been recent works on policy evaluation and learning without overlap, 
concurrent with or after the first appearance of this work.  
For example, \cite{khan2023off} studies policy evaluation (i.e., without learning)  by extrapolating under smoothness assumptions; \cite{liu2023average}
shifts the target to overlap weighted average treatment effect on the treated;
\cite{mou2023kernel} studies  policy evaluation  
via regularized kernel regression, 
focusing on the minimax optimality under certain structural assumptions. 
For policy learning, \cite{zhao2023positivity} overcomes the overlap issue 
by considering an ``incremental'' policy that slightly deviates from the observational policy~\citep{kennedy2019nonparametric}, and study asymptotic, doubly robust policy learning among the class of incremental policies. In contrast, our work is not limited to the incremental policy class, and finite-sample analysis with pessimism and empirical Bernstein's inequality is also drastically different from the strategies in those works. 
\cite{chen2023steel} considers the problem under a Markov Decision Process, where the trajectory is assumed stationary, 
and induced by a fixed behavior policy (similar to our batched data setting).
The main problem therein is the unidentifiable part of policy value due to singularity, 
i.e., part of the space that has zero probability under the behavior policy; 
They employ distributionally robust optimization ideas to optimize the worst-case performance of policies, 
and handle the weak overlap in the estimation of the identifiable part via functional approximation of
the conditional mean function. Both the problem and method therein are different from ours. \\[-1ex]
 
\noindent\textit{Pessimism in offline RL and bandits.}  
The principle of pessimism was initially proposed 
in offline  
reinforcement learning (RL)~\citep{buckman2020importance,jin2021pessimism} 
for learning an optimal policy in Markov decision processes (MDPs) 
using pre-collected datasets~\citep{zanette2021provable,xie2021bellman,chen2022offline,rashidinejad2021bridging,yin2021towards,shi2022pessimistic,yan2022efficacy,xie2021policy,uehara2021pessimistic} 
and has been applied to  
contextual bandits as a special case of MDPs. 
To the best of our knowledge, 
this work is the first to  instantiate 
the pessimism principle 
for learning individualized decision rules
that requires no assumptions on the context space or
the mean outcome model.


To be specific, existing pessimism-based algorithms 
optimize LCBs for conditional value functions (which reduce to 
$\mu(x,a) := \EE[Y(a)\given X=x]$ in our setting),
whose construction either 
relies on finite cardinality of the context space~\citep{buckman2020importance,rashidinejad2021bridging}
or strong modeling assumptions~\citep{jin2021pessimism,xie2021bellman}. 
In contrast, our proposal 
does not rely on any modeling assumption on the conditional mean 
function.
It can be viewed as a `design-based' version 
of the pessimism principle which uses 
the knowledge of the sampling propensities.  
This is particularly suitable for 
continuously-valued contexts and 
for high-stakes applications where 
model misspecification is of concern. \\[-1ex]
 
\noindent\textit{Variance-regularized ERM and empirical Bernstein's inequality.}
The idea of pessimism  also appears in 
empirical risk minimization (ERM),
where the principle is implemented by taking into account the  
estimation uncertainty of empirical risks. 
Among them,~\cite{maurer2009empirical}  
proposed to optimize a variance-regularized empirical risk, 
and was extended by~\cite{duchi2016variance} to a computationally tractable 
formulation. 
The variance regularization is justified by 
empirical Bernstein's inequality  studied 
by a series of works~\citep{bartlett2005local,bartlett2006convexity,koltchinskii2006local}. 
Recently,~\cite{xu2020towards} studied 
instance-dependent risk without explicitly using a regularization term. 
These works all consider i.i.d.~data and bounded loss. 
Our method extends this line of work to 
\emph{unbounded} random variables 
and adaptively collected hence mutually \emph{dependent} 
data.
Closer to our setting,~\cite{swaminathan2015batch} studies 
counterfactual risk minimization using i.i.d.~data and unbounded weights. 
However, they used clipping to exclude extremely small weights 
and reduced the analysis to bounded random variables. The clipping  
causes bias, so the ``oracle'' property 
of pessimism (the sole dependence on the \emph{optimal} instance) may not hold.
Our work can also be of interest for providing a novel solution 
to counterfactual risk minimization with unbounded weights.

To deal with non-i.i.d.~data and unbounded weights, 
we develop a new self-normalized-type uniform concentration inequality 
for empirical risks.  
When applied to one single policy (i.e., without uniform concentration), 
it can be viewed as a generalization of 
the  empirical Bernstein's inequality~\citep{maurer2009empirical}. 
Several recent indepenendet works also generalize on empirical Bernstein's inequality, 
including~\cite{waudby2020estimating} for bounded data 
and~\cite{waudby2022anytime} for estimating a policy value 
using adaptively collected data.
While their methods and analysis are complementary to ours,
it will be interesting to see whether these bounds are applicable to 
providing uniform concentration over policy classes to enable policy learning in our approach.\\[-1ex]
 
\noindent\textit{Other reweighting-based methods.}
Since we use an IPW-type
estimator, this work is generally related to reweighting-based approaches for 
off-policy evaluation (OPE) in offline RL and contextual bandits such 
as marginal importance sampling~\citep{liu2018breaking,xie2019towards,uehara2020minimax,zhang2020gendice}, whose goal 
is to estimate the value of a policy from a pre-collected data set.  
Several works in this strand also consider policy optimization~\citep{nachum2019algaedice,lee2021optidice,zhan2022offline,rashidinejad2022optimal} 
with some pessimism (or conservativeness) features.  
These works operate under uniformly bounded weights, i.e., 
the uniform overlap condition. 
Our empirical Bernstein's inequality applies to unbounded weights, 
which may also be of interest for extending these works to broader settings.  
Finally, similar to our motivation,~\cite{kuzborskij2021confident} studies 
estimation of $Q(\pi)$ for a small number of policies $\pi$
with i.i.d.~data without uniform overlap between $\pi$ and the behavior policy. 
They use  self-normalized importance weighting 
to trade bias for reduced variance. Compared with them, 
our work uses a different estimator, applies to adaptively collected data, 
and studies policy optimization instead of policy evaluation. 
It would be interesting to see whether their ideas may lead to other 
construction of LCBs for similar ideas as PPL 
that works also for the adaptive case.

We summarize our contributions and their relation to existing literature 
in Figure~\ref{fig:diagram}.
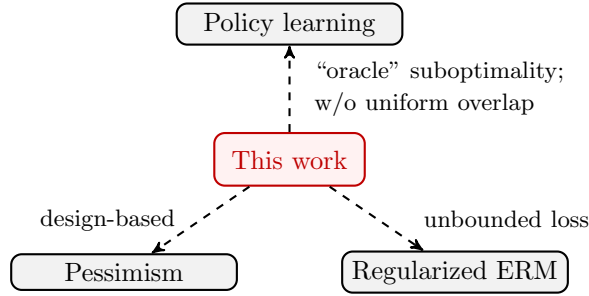
\begin{figure}[h!]
  \centering  \hspace{6em}
  \begin{tikzpicture}[->,>=stealth', thick, main node/.style={circle,draw}]
  
  \node[rounded corners, rectangle, text=black, draw=black, fill=black!5, minimum width = 3cm] (1) at  (0,1.8) {Policy learning} ;
  \node[rounded corners, rectangle, text=red!75!black, draw=red!75!black, fill=red!5,minimum width = 2cm, minimum height=0.7cm] (2) at  (0,0) {This work} ;
  \node[rounded corners, rectangle, text=black, draw=black, fill=black!5, minimum width = 3cm] (3) at  (-2.2,-1.5) {Pessimism} ;
  \node[rounded corners, rectangle, text=black, draw=black, fill=black!5, minimum width = 3cm] (4) at  (2.2, -1.5) {Regularized ERM} ;

  \draw[->, dashed] (2) -- (1) node [text width = 5cm, right = 0.2cm, midway] {\small ``oracle''  suboptimality; \\ w/o uniform overlap};
  \draw[->,dashed] (2) -- (3) node [text width = 2cm, left, midway] {\small design-based};
  \draw[->,dashed] (2) -- (4) node [text width = 4.5cm, right =0.5cm, midway] {\small unbounded loss};
  \end{tikzpicture}  
  \caption{Our contributions and their relation to existing literature.}
  \label{fig:diagram}
  \end{figure}

\section{Pessimistic policy learning} 


\subsection{What makes offline policy learning difficult?}
\label{subsec:challenge}

Before presenting our method, we first discuss 
the necessity of 
controlling the worst-case estimation error for greedy approaches, 
which explains the dependence on stringent uniform overlap assumptions. 

As introduced earlier, a greedy learner returns
$
\hat\pi_{\greedy} = \argmax_{\pi\in\Pi} ~ \hat{Q}(\pi).
$
Here, $\hat{Q}(\pi)$ is typically 
an IPW-type (unbiased) point estimator, 
which takes the form
$
\hat{Q}(\pi) = \frac{1}{T}\sum_{t=1}^T \frac{\ind\{A_t=\pi(X_t)\}}{e_t(X_t,\pi(X_t)\given \cH_t)} Y_t.
$
We can decompose the suboptimality of $\hat\pi =\hat\pi_{\greedy}$  into
\#\label{eq:subopt_decomp}
Q(\pi^*) - Q(\hat\pi ) = \underbrace{Q(\pi^*) - \hat{Q}(\pi^*)}_{\textrm{(i) intrinsic uncertainty}} + \underbrace{\hat{Q}(\pi^*) - \hat{Q}(\hat\pi)}_{\textrm{(ii) optimization error}}
+ \underbrace{\hat{Q}(\hat\pi) - Q(\hat\pi)}_{\textrm{(iii) greedy uncertainty}}.
\#
Above, (i) is the estimation uncertainty for the 
optimal policy $\pi^*$, (ii) is the nonpositive optimization error,
and (iii) is the estimation uncertainty of the 
greedily learned policy. Among these three sources
of uncertainty, it is most challenging to control (iii), since 
it might well be the case that
$\hat{\pi}_{\greedy}$ appears to be optimal
{\em as a result of} a large estimation error.

In Figure~\ref{fig:w01}, we consider a simple case 
with two policies $\Pi=\{\pi_1,\pi_2\}$ and a fixed behavior policy $e(x,a)$
to illustrate how poor estimation of
suboptimal policies confuses the greedy learner. 
In this example, $\pi^*=\pi_1$
is well-explored by the behavior policy --- i.e., its 
propensities $e(X_t,\pi^*(X_t))$
are all large --- and as a result, $\hat{Q}(\pi_1)-Q(\pi_1)$ is of a small scale; 
in contrast, the suboptimal policy $\pi_2$ is less explored by the
behavior policy and has small propensities, 
which translates to a large scale of $\hat{Q}(\pi_2)-Q(\pi_2)$. 
Thus, even though  $Q(\pi_2)$ is smaller than $Q(\pi_1)$, 
a greedy learner may still choose $\hat\pi_\greedy = \pi_2$ 
since there is a considerable chance that $\hat{Q}(\pi_2)$ 
appears to be larger than $\hat{Q}(\pi_1)$. 
In other words, to ensure a small suboptimality of $\hat\pi_\greedy$, 
we need both policies to be accurately estimated. 

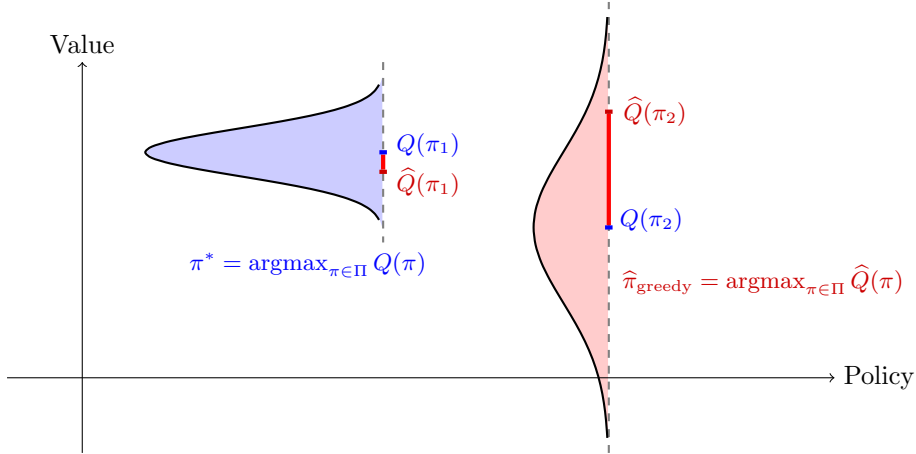
\begin{figure}[h!]
\centering
\begin{tikzpicture}
\fill [blue!20, domain=1:3, variable=\y]
        (2.99, 1.1)
        -- plot ({-exp(-(\y-2)*(\y-2)*5) * sqrt(10) + 3.01},{\y})
        -- (2.99, 1.9)
        -- cycle; 
\fill [red!20, domain=-1.8:3.8, variable=\y]
        (5.99, -1.8)
        -- plot ({-exp(-(\y-1)*(\y-1)/2) + 6}, {\y})
        -- (5.99, 3.8)
        -- cycle; 
\draw[->] (-2, -1) -- (9, -1) node[right] {Policy};
\draw[->] (-1, -2) -- (-1, 3.2) node[above] {Value};
\draw [gray, dashed, thick] (3,3.2) --  (3,0.8) ;
\draw[domain=1.1:2.9, smooth, variable=\y, black, thick]  plot ({-exp(-(\y-2)*(\y-2)*5) * sqrt(10) + 3}, {\y}); 
\draw [gray, dashed, thick] (6,-2) -- (6,4);
\draw[domain=-1.8:3.8, smooth, variable=\y, black, thick]  plot ({-exp(-(\y-1)*(\y-1)/2)  + 6}, {\y});

\draw [blue, ultra thick] (5.95,1) -- (6.05,1) ;
\draw [blue, ultra thick] (6,1.1) node[right]{\small $Q(\pi_2)$};
\draw [red!80!black, ultra thick] (5.95,2.54) -- (6.05,2.54) node[right]{\small $\hat Q(\pi_2)$};

\draw [red, ultra thick] (3,1.77) -- (3,1.97); 
\draw [red, ultra thick] (6,1.03) -- (6,2.51);

\draw [blue, ultra thick] (2.95,2) -- (3.05,2) ; 
\node[text=blue] at (3.6,2.1) {\small $Q(\pi_1)$}; 
\draw [red!80!black, ultra thick] (2.95,1.74) -- (3.05,1.74);  
\node[text=red!80!black] at (3.6,1.6) {\small $\hat Q(\pi_1)$}; 

\node[text=red!80!black] at (8.05,0.3) {\small $\hat\pi_{\textrm{greedy}} = \argmax_{\pi\in\Pi}\hat{Q}(\pi)$};
\node[text=blue] at (2,0.5) {\small $\pi^*  = \argmax_{\pi\in\Pi} Q(\pi)$};
\end{tikzpicture}
\caption{An illustration of how poor estimation of a suboptimal policy $\pi_2$ (red) leads to 
large suboptimality of the greedy approach. Here, $Q(\pi_i)$ is the actual 
policy value, and $\hat{Q}(\pi_i)$ the IPW estimator 
for $i=1,2$. The shaded area represents the distribution of two estimators. 
In this case, the greedy learner picks $\pi_2$ because $\hat{Q}(\pi_2)$ appears to 
be large due to its large estimation error. 
}
\label{fig:w01}
\end{figure}

The standard approach to controlling (iii) is by 
bounding the worst-case estimation error $V_\worst := \sup_{\pi\in\Pi} |\hat{Q}(\pi)-Q(\pi)|$,
which is indeed necessary for the greedy approach
as illustrated by the above example. 
Since the estimation error (or roughly, the standard deviation) of $\hat{Q}(\pi)$ depends on 
the inverse propensity for $\pi$, controlling $V_\worst$ 
often relies on assuming 
sufficiently large $e_t(X_t,\pi(X_t)\given \cH_t)$ for all $\pi\in\Pi$ and all $t\in[T]$~\citep{kitagawa2018should,zhan2021policy,bibaut2021risk}. 
As we discussed earlier, such assumptions often fail to hold in practice. 
In contrast, term (i) is much easier to control since 
it is the estimation error of one single policy. 
Indeed, we will show that for both batched and 
adaptively collected data,  (i) is the key quantity that 
captures the intrinsic difficulty of policy learning, 
as it arises in the information-theoretic lower bound.

\subsection{Pessimistic policy learning}
\label{subsec:pess_idea}

To eliminate the reliance on the uniform overlap assumption, 
we make an algorithmic change to the greedy learning paradigm
by introducing the pessimism principle~\citep{jin2021pessimism}.

For every policy $\pi\in\Pi$,
we complement the point estimator $\hat{Q}(\pi)$
with a policy-dependent regularization term
$R(\pi)$,
and optimize the {\em pessimistic policy value}
\#\label{eq:def_pess}
\hat\pi \defn \argmax_{\pi\in\Pi} ~ \big\{ \hat{Q}(\pi) - R(\pi) \big\}.
\#
Roughly speaking, we expect $R(\pi)$ to reflect 
the estimation error of $\hat{Q}(\pi)$, and therefore~\eqref{eq:def_pess} can be 
viewed as optimizing the lower confidence bounds (LCBs) for the policy values. 
The core merit of the pessimism principle is captured 
by the following proposition.

\begin{prop}\label{lem:pess}
Let $\hat\pi$ be obtained as in~\eqref{eq:def_pess}, 
and let $\pi^* =\pi^*(\cC,\Pi)$ be 
the optimal policy among $\Pi$.
Then  
\$
\cL(\hat\pi;\cC,\Pi) \leq 2 R(\pi^*)
\$ 
holds on the event 
\#\label{eq:event}
\cE \defn \big\{ |\hat{Q}(\pi) - Q(\pi ) |\leq R(\pi),~\forall \pi\in\Pi   \big\}.
\# 
\end{prop}

\begin{proof}[Proof of Proposition~\ref{lem:pess}]
The uniform concentration on $\cE$ ensures
$
Q(\pi^*) \leq \hat{Q}(\pi^*) + R(\pi^*)
$
and 
$
Q(\hat\pi) \geq \hat{Q}(\hat\pi) - R(\hat\pi)
$.
On the other hand,  the optimality of $\hat\pi$ implies
$
 \hat{Q}(\hat\pi) - R(\hat\pi) \geq \hat{Q}(\pi^*) - R(\pi^*).
$
Collectively,
\$
Q(\hat\pi) \geq \hat{Q}(\pi^*) - R(\pi^*) \geq  {Q}(\pi^*) - 2R(\pi^*).
\$
By the definition of suboptimality, we conclude the proof of Proposition~\ref{lem:pess}.
\end{proof}

A crucial consequence of Proposition~\ref{lem:pess} is that 
the suboptimality of $\hat\pi$ in~\eqref{eq:def_pess} only 
depends on the regularization term (roughly, the estimation uncertainty) 
of the \emph{optimal} policy. 
This result shows the particular role of pessimism 
in eliminating the most challenging term (iii) in~\eqref{eq:subopt_decomp}. 
We will see in Sections~\ref{sec:fix} and~\ref{sec:adaptive} 
that with a proper construction, $R(\pi^*)$ 
is small as long as $\pi^*$ has sufficiently large propensities, 
and thus we achieve efficient policy learning under a 
significantly weaker condition than the uniform overlap.  


Both the learning performance guarantee and the uniform concentration event $\cE$
in our approach significantly differ from the greedy algorithms in the literature. For the greedy algorithm, the theoretical guarantee 
often relies on the fact that $\cL(\hat\pi;\cC,\Pi) \leq 2 R$ on the event 
\$ 
\cE_\greedy:= \big\{  |\hat{Q}(\pi) - Q(\pi ) |\leq R ,~\forall \pi\in\Pi \big\},
\$
where $R\in \RR$ is a constant that upper bounds the worst-case estimation error 
but does not appear in the optimization problem.  
In contrast, we seek a data-dependent bound $R(\pi)$ for every policy $\pi\in\Pi$. 
Incorporating this piece of information into our optimization problem 
leads to an improved data-dependent suboptimality upper bound. 

Proper construction of   $R(\pi)$ is key to efficient policy learning bounds as in Proposition~\ref{lem:pess}: we need it to tightly reflect the estimation quality of $\hat{Q}(\pi)$ for each individual policy $\pi\in\Pi$. In the rest of this paper, our efforts are devoted to 
constructing the regularization term and proving the uniform concentration 
result in the form of~\eqref{eq:event}. 
  The key technical milestone is a high-probability upper bound on the self-normalized-type quantity 
  \$ 
  \sup_{\pi\in\Pi}\frac{|\frac{1}{T}\sum_{t=1}^T \hat\Gamma_t(\pi)-Q(\pi)|}{V(\pi)} \leq \beta
  \$ 
  for both i.i.d.~and adaptively collected data without overlap conditions, where $V(\pi)$ is an empirical proxy of the estimation error that does not rely on unknown quantities.


\subsection{Construction of the estimator and the regularizer}
\label{subsec:generic_algo}

The key contribution of PPL is to instantiate the generic idea in Section~\ref{subsec:pess_idea} 
into a concrete algorithm
by constructing the estimator $\hat{Q}(\pi)$ and 
the regularizer $R(\pi)$ for all $\pi\in\Pi$. 
We begin with a standard AIPW-type~\citep{robins1994estimation} estimator for the policy value: 
\#\label{eq:def_hatQ_generic}
\hat{Q}(\pi) = \frac{1}{T}\sum_{t=1}^T \hat\Gamma_t(\pi),\quad 
\hgamma_t(\pi) = \hat\mu_t(X_t,\pi(X_t)) + \frac{\ind\{A_t=\pi(X_t)\}}{e_t(X_t,\pi(X_t)\given \cH_t)} \cdot\big(Y_t-\hat\mu_t(X_t,\pi(X_t))\big),
\# 
and we set $\hat{Q}(\pi)=-\infty$ if 
$e_t(X_t,\pi(X_t)\given \cH_t)=0$ for any $t\in[T]$. 
Here $\hat\mu_t(\cdot,\cdot)\colon \cX\times\cA\to [0,1]$ is 
an estimator for $\mu(x,a)$.
With batched data,  we assume 
for simplicity that 
$\hat\mu_t \equiv \hat\mu (\cdot,\cdot)\colon \cX\times\cA\to [0,1]$ is an estimator
that is independent of the data. 
This is only for the convenience of exposition; in practice, this 
can be achieved by a cross-fitted
version of our algorithm~\citep{schick1986asymptotically,chernozhukov2018double} 
described in Section~A of the supplementary material with the same theoretical guarantee.
With adaptively collected data, 
we allow $\hat\mu_t(x,a)$ to be  constructed using 
$\{X_i,A_i,Y_i\}_{i=1}^{t-1}$. 
The estimator~\eqref{eq:def_hatQ_generic} is standard in the 
policy learning literature. 
Taking $\hat\mu(\cdot,\cdot)\equiv 0$ 
reduces~\eqref{eq:def_hatQ_generic} to the 
IPW estimator  
used in~\cite{zhao2012estimating,kitagawa2018should,swaminathan2015batch,bibaut2021risk}. 
For general $\hat{\mu}_t(\cdot,\cdot)$, 
it 
is similar to the AIPW-type estimator 
in~\cite{athey2021policy,zhou2022offline,zhan2021policy}. 


The second  ingredient 
is the regularizer $R(\pi)$. 
PPL takes inspirations from empirical Bernstein's inequality, in that $R(\pi)$ shall capture the standard deviation of $\hat{Q}(\pi)$. 
However, naively taking $R(\pi)$ to be the empirical standard deviation 
as in the classical empirical Bernstein's inequalities~\citep{bartlett2005local,bartlett2006convexity,koltchinskii2006local}
does not lead to the desired bound.
There remains significant challenges in establishing the inequalities in our context, 
including (i) unbounded $\hat\Gamma_t(\pi)$, 
(ii) dependent data, and (iii) uniform convergence.

We construct 
$R(\pi) = \beta\cdot V(\pi)$ for 
a scaling constant $\beta>0$ to be decided later (with additional log-terms for adaptive data), and  
\#\label{eq:def_V_generic}
V(\pi) \defn  
\max\big\{\vs(\pi), \vp(\pi), \vh(\pi)\big\} 
\# 
with the three data-dependent components
\#\label{eq:V_terms}
\mbox{(sample deviation)} \qquad & \vs(\pi) = \frac{1}{T}  
\bigg( \sum^T_{t=1} \frac{\ind \{A_t = \pi(X_t)\}}{e(X_t,\pi(X_t)\given \cH_t)^2}\bigg)^{1/2}, \notag  \\
\mbox{(population deviation)} \qquad & \vp(\pi) = \frac{1}{T}
\bigg( \sum^T_{t=1} \frac{1}{e(X_t,\pi(X_t)\given \cH_t)}\bigg)^{1/2},  \\
  \mbox{(higher-order deviation)} \qquad & \vh(\pi) = \frac{1}{T}
\bigg( \sum^T_{t=1} \frac{1}{e(X_t,\pi(X_t)\given \cH_t)^3}\bigg)^{1/4}. \notag
\#
As we will see later, the scaling factor $\beta>0$ depends on the 
complexity of the policy class $\Pi$, and 
the choice of $\beta$ differs slightly for batched and adaptively collected data. 

Let us discuss the intuitions in these three terms. 
The idea is to leverage knowledge of the behavior policy 
to estimate the 
fluctuation of $\hat{Q}(\pi)$ around $Q(\pi)$.
The uncertainty term $V(\pi)$ can be 
interpreted as a proxy  
for the standard deviation of $\hat{Q}(\pi)$. 
To be more specific, 
$\vs(\pi)$ roughly upper bounds the 
uncertainty of $\hat{Q}(\pi)$ conditional on all but $\{Y_t\}_{t=1}^T$:
\#\label{eq:sn_ineq}
\frac{1}{T^2}\sum^T_{t=1}  \frac{\ind\{A_t = \pi(X_t)\}}{e_t(X_t,\pi(X_t)\given \cH_t)^2} \cdot
\Var(Y_t\given X_t,A_t) \le \vs(\pi)^2.
\#
Similarly,  
$\vp(\pi)$ approximately upper bounds the 
uncertainty conditional on $\{X_t\}_{t=1}^T$:
\$
\frac{1}{T^2}\sum^T_{t=1}  \frac{\Var(Y_t\given X_t,\pi(X_t))}{e_t(X_t,\pi(X_t)\given \cH_t)}
  \le \vp(\pi)^2.
\$
Finally, $\vh(\pi)$ upper bounds the 
standard deviation of $\vs(\pi)$
conditional on $\{X_t\}_{t=1}^T$,
and is a higher-order term
dominated by $\vp(\pi)$ in general, in the sense that 
$\vh(\pi) \le \vp(\pi)$ if $e(x,a) \ge 1/T$ for 
all $(x,a) \in \cX \times \cA$.  


  In Sections~4 and~5, we will demonstrate that  $R(\pi)$ is valid uncertainty quantifier obeying~(4).  
  This technical route diverges significantly from existing works in the policy learning 
  literature~\citep{zhao2012estimating,zhou2022offline,zhan2021policy},
   which rely heavily on lower bounded propensities -- hence upper bounded $\hat\Gamma_t(\pi)$ -- to establish uniform concentration of $\hat{Q}(\pi)$. 
   Without imposing lower bounds on $e_t(x,a\given \cH_t)$, it is unclear whether, and at what rate, $\hat{Q}(\pi)$ would converge to $Q(\pi)$. 
   In contrast, our approach seeks for a data-dependent characterization of the estimation quality of $\hat{Q}(\pi)$. 
   This strategy not only facilitates analysis in  the absence of overlap 
   but also enables instantiating the pessimism principle for IPW-type policy value estimators.
  
                Our proof leverages self-normalization inequalities to 
                derive the deviation bound of the martingale $\hat\Gamma_t(\pi)-\mu(X_t,A_t)$.
                Importantly, the unbounded $\hat\Gamma_t(\pi)$ and dependent data structure 
                requires going beyond a straightforward application of existing self-normalized 
                concentration inequalities. This is because many exponential-tail self-normalized 
                inequalities (see, e.g.,~\cite[Section 9.3]{pena2009self}) require constant bounds 
                on $|\hat\Gamma_t(\pi)|$ or summation of the second moments of $|\hat\Gamma_t(\pi)|$,
                which is not infeasible in the absence of uniform overlap. 
                To address these challenges, 
                we employ a suite of techniques to obtain a specialized class of martingales, 
                which then enables the application of  a particular self-normalized concentration inequality 
                that does not require any bounded constraints~\citep{de2004self}. Further details on these techniques are provided in our discussion of the theoretical results.

Finally, we highlight PPL as the first-of-the-kind instantiation of the general pessimism principle in optimizing AIPW-type estimates over a policy class. While the pessimism principle has led to various algorithms in value-based offline RL~\citep{jin2021pessimism,rashidinejad2021bridging,xie2021bellman}, these results rely on well-specified function classes for the value function ($\mu(x,a)$ in our context). Instead, by leveraging knowledge of the sampling policies, we enable a design-based, model-free implementation. As such, the analysis techniques and theoretical guarantees are quite different.

\section{Learning from a fixed behavior policy}
\label{sec:fix}


In this section, we consider the i.i.d.~setting 
where the offline data are collected by a fixed behavior policy, i.e., 
$e_t(x,a\given H_t) \equiv e(x,a)$ 
for a \emph{known} policy $e(\cdot,\cdot)$ and all $(x,a)\in \cX\times\cA$, $t \in [T]$. 
We begin in Section~\ref{sec:iid_upper_bnd} with an upper
bound on the suboptimality of our pessimistic learning
algorithm, and provide a matching lower bound
in Section~\ref{sec:iid_lower_bnd}.


\subsection{Suboptimality upper bound}
\label{sec:iid_upper_bnd}
 
 We set the scaling constant $\beta$ to be
approximately $\sqrt{\ndim(\Pi)}$ in order to achieve efficient
learning. The suboptimality of the resulting algorithm
is characterized by Theorem~\ref{thm:fix_upper}, 
whose proof is in Appendix~\ref{subsec:sketch_iid}. 

\begin{theorem}\label{thm:fix_upper}
Fix  $\delta\in(0,1)$. 
For any policy class $\Pi$, 
suppose we set $\hat{Q}(\pi)$ as in~\eqref{eq:def_hatQ_generic}, $V(\pi)$ as in~\eqref{eq:def_V_generic}, and 
\$R(\pi)=\beta\cdot V(\pi), \quad \mbox{ for }
\beta \geq 10\sqrt{2(\ndim(\Pi) \log(TK^2) + \log(16/\delta))}.
\$ 
Then with probability at least $1-\delta$, it holds that 
\$
\big| \hat{Q}(\pi) - Q(\pi) \big| \leq \beta\cdot V(\pi) \quad \textrm{for all }\pi\in \Pi.
\$ 
Furthermore, with such a choice, we have 
\#\label{eq:fix_upper}
\cL(\hat\pi;\cC,\Pi) \leq  \min\big\{2 \beta \cdot V(\pi^*),1\big\} 
\# 
with probability at least $1-\delta$, 
where $\hat\pi$ is defined in~\eqref{eq:def_pess} and $\pi^* = \argmax_{\pi\in\Pi} Q(\pi;\cC)$. 
\end{theorem}

The constant term in the above upper bound is trivial since
$0 \le Q(\pi;\cC) \le 1$ for all $\pi \in \Pi$ by assumption;
in the term $2\beta \cdot V(\pi^*)$, the regularizer $V(\pi^*)$ is roughly 
the estimation error for $\hat{Q}(\pi^*)$, and 
the choice of $\beta$ arises from uniform 
concentration over the policy class $\Pi$.  
Theorem~\ref{thm:fix_upper} implies that
the suboptimality gap of pessimistic learning
is small when $V(\pi^*)$ is small --- that is, 
when the value of $Q(\pi^*)$ can be well estimated.
This is the case, for example, when there
is enough overlap for the optimal policy $\pi^*$. 

Theorem~\ref{thm:fix_upper} also shows an 
``oracle'' property~\citep{donoho1994ideal,fan2001variable,zou2006adaptive} 
of PPL, in the sense that it 
automatically ``adapts'' to the overlap condition of $\pi^*$, even though $\pi^*$ is 
not known a priori~\citep{jin2021pessimism}. 

\begin{remark}\normalfont 
Theorem~\ref{thm:fix_upper} holds without any assumptions on the behavior policy. 
Instead, the burden of assumptions is 
``shifted'' to the upper bound, which honestly reflects 
how well the offline data collecting process is for enabling learning the optimal policy. 
\end{remark}

To better understand the upper bounds in Theorem~\ref{thm:fix_upper}, 
we consider in the following corollary a specific setting where
there is sufficient overlap for the \emph{optimal} policy, 
but not necessarily for all $\pi\in\Pi$.  

\begin{corollary}\label{cor:fixed}
Fix any $\delta\in(0,\exp(-1))$. 
Assume there exists some $C_*>0$ (allowed to depend on $T$)
such that $e(x,\pi^*(x))\geq C_*$  for $\PP_X$-almost all $x\in \cX$. 
Suppose we choose $\beta = c \cdot \sqrt{  \ndim(\Pi) \log(TK^2) + \log(16/\delta) }$ 
for some $c\geq 10\sqrt{2}$, 
and suppose 
$T\geq 3$. 
Then with probability at least $1-2\delta$, 
\#\label{eq:cor_fix_upper}
\cL(\hat\pi;\cC,\Pi) \leq   
\min\bigg\{ 2c  \cdot \sqrt{ \frac{ \ndim(\Pi) \log(TK^2)}{C_* T}   }  \cdot  \log(2/\delta)^{3/2}
,1\bigg\},
\# 
where $\hat\pi$ is defined in~\eqref{eq:def_pess} and $\pi^* = \argmax_{\pi\in\Pi} Q(\pi;\cC)$. 
\end{corollary}

The proof is provided in Appendix~\ref{app:subsec_cor_fixed}. 
Corollary~\ref{cor:fixed} implies that efficient policy learning with 
a rate of $T^{-1/2}$ is feasible 
as long as the optimal policy is explored well. 
It is helpful to compare Corollary~\ref{cor:fixed} 
to state-of-the-art results in the literature to gain more intutitions.~\cite{kitagawa2018should}
show that the greedy policy obeys $\cL(\hat\pi_{\text{greedy}};\cC,\Pi) \lesssim \sqrt{\ndim(\Pi) /(\underline{C}T)}$ under the condition that $\bar{C}>0$, where $\lesssim$ hides logarithmic terms, and $\underline{C}=\inf_{x,a}e(x,a)$. 
Compared with this result, 
PPL broadens the scenarios where efficient policy learning is feasible, 
as it could happen that $\bar{C}=0$ but $C_*>0$.  
In addition, PPL achieves sharper statistical rates as it holds
deterministically that $C_* \geq \bar{C}$.

Finally, we extend Theorem~\ref{thm:fix_upper} and show 
that the performance of PPL is comparable to any reference policy 
under consideration. Its proof slightly modifies that of Theorem~\ref{thm:fix_upper}, 
and can be found in Appendix~\ref{app:cor_general}. 

\begin{corollary} 
\label{cor:general}
Under the same setup as in Theorem~\ref{thm:fix_upper}, 
with probability at least $1-\delta$, 
\$
\cL(\hat\pi;\cC,\Pi) \leq \inf_{\pi\in\Pi}\big\{ \cL(\pi;\cC,\Pi) + 2\beta\cdot V(\pi) \big\}.
\$
\end{corollary}

    Corollary~\ref{cor:general} can be viewed as a ``softened'' version 
    of Theorem~\ref{thm:fix_upper}, which further broadens the possibility of efficient policy learning: 
    even if the optimal policy does not have sufficient overlap so that $V(\pi^*)$ is large, 
    the suboptimality of PPL would still be small  
    as long as there is some near-optimal policy with good overlap. 
    This is an additional ``oracle'' property that PPL is comparable to any well-covered policy in the class. 

\begin{remark}[Extension to unknown propensities]\normalfont
     We have assumed knowledge of the true propensities $e_t(\cdot,\cdot)$. 
            For batched data, this is a standard setting considered in the literature~\citep{kitagawa2018should,zhao2012estimating}. 
            For future work, it would be interesting to extend PPL to  
            observational studies where the fixed behavior policy $e(\cdot,\cdot)$ has 
            to be estimated from data~\citep{athey2021policy}. 
            We expect this to be a  nontrivial task, especially 
            when it is estimated via nonparametric regression or machine learning algorithms 
            that do not rely on strong modeling assumptions~\citep{chernozhukov2018double}. We defer a detailed discussion on the technical challenges to Appendix~\ref{app:subsec_unknown_propensity}. 
\end{remark}

\subsection{Pessimistic policy learning is minimax optimal}
\label{sec:iid_lower_bnd}
We now show that 
PPL
is the ``best effort'' given the fixed behavior policy. 
To be specific, given an action set $\cA$,
a policy class $\Pi$,
a (possibly $T$-dependent) constant $C_*>0$,
and a sample size $T\in \NN_+$,
we let $\cR(C_*,T,\cA,\Pi)$ denote the set of instances we consider
--- each element of $\cR(C_*, T, \cA, \Pi)$ corresponds to 
a contextual bandit model $\cC$ defined on $\cA$, 
a behavior policy $e(\cdot,\cdot)$ obeying
\#\label{eq:lower_bd}
 \inf_{x\in\cX} e(x,\pi^* (x)) \geq C_*,
\#
which together generates the offline data $\cD = \{(X_t,Y_t,A_t)\}_{t=1}^T$.
Here $\pi^* =\pi^*(\cC,\Pi)$ is the optimal policy among $\Pi$ 
for the contextual bandit model $\cC$ defined in~\eqref{eq:def_pi*}.    

The next theorem establishes the minimax lower bound 
for policy learning in $\Pi$ from any behavior policy 
obeying~\eqref{eq:lower_bd}. 

\begin{theorem}\label{thm:lower}
For any action set $\cA$, sample size $T \in \NN_+$, any policy class 
$\Pi$ and any $C_*>0$ with 
$ \frac{\ndim(\Pi)}{ C_* T} \leq 1.5$, one has
\#
\inf_{\hat\pi} \sup_{(\cC,e)\in \cR(C_*,T,\cA,\Pi)} \EE_{\cC,e} \big[  \cL(\hat\pi\,;  \cC,\Pi ) \big]
\geq 0.12   \sqrt{ \frac{ \ndim(\Pi)}{ C_* T} },
\#
where $\ndim(\Pi)$ is the Natarajan dimension of $\Pi$, $\hat\pi$ is any 
data-dependent policy, and $\EE_{\cC,e}$ means taking expectation  with respect to
the randomness of the offline data generated under $\cC$ and $e(\cdot,\cdot)$. 
\end{theorem}
The lower bound in Theorem~\ref{thm:lower} 
matches the upper bound in Corollary~\ref{cor:fixed}
up to logarithm factors of $T$ and $K$: when $\ndim(\Pi) \ge C_*T$,
the upper and lower bound are both of the constant order;
when $\ndim(\Pi) < C_* T$, both bounds scale as $\sqrt{\ndim(\Pi)/C_* T}$
(up to logarithmic factors). Through matching upper and lower bounds, we have shown
that pessimistic policy learning is the best effort  given 
offline data from a fixed behavior policy. Theorem~\ref{thm:lower} can 
in fact be implied by~\cite[Theorem 1]{zhan2021policy},
and we provide its proof in Appendix~\ref{app:subsec_lower} 
for completeness.

\section{Learning from  adaptive behavior policies}
\label{sec:adaptive}
 

In this section, we turn our attention to pessimistic 
policy learning from adaptively collected data. Here, 
the behavior policy is allowed to change over time 
depending on previous observations. 
%
%



\subsection{Suboptimality upper bound}
The following theorem provides 
a data-dependent upper bound for the suboptimality 
of pessimistic policy learning,
and its proof can be found in Appendix~\ref{subsec:proof_upper_adapt}. 


\begin{theorem}\label{thm:adapt_upper}
  Fix any $\delta\in(0,1)$. 
  For any policy class $\Pi$, we set $\hat{Q}(\pi)$ as in~\eqref{eq:def_hatQ_generic}, 
  $V(\pi)$ as in~\eqref{eq:def_V_generic}, and let
  \$
  R(\pi) = \beta\cdot V(\pi) \sqrt{ 2\log (30T\cdot V(\pi))},
  \quad \mbox{for } \beta\geq  6 \sqrt{\ndim(\Pi)\log(TK^2) + \log(16/\delta)}.
  \$
  Then  with probability at least $1-\delta$, it holds that 
  \#\label{eq:ineq_adapt}
  \big|\hat{Q}(\pi) - Q(\pi;\cC) \big|\leq \beta\cdot V(\pi),\quad \textrm{for all }\pi\in\Pi.
  \#
  Furthermore, 
  \$
  \cL(\hat\pi;\cC,\Pi) \leq \min \big\{2\beta \cdot V(\pi^*)\sqrt{ 2\log (30T\cdot V(\pi^*))},1 \big\}
  \$
  holds with probability at least $1-\delta$, where $\hat\pi$ 
  is defined in~\eqref{eq:def_pess} and $\pi^* = \argmax_{\pi\in\Pi} Q(\pi;\cC)$. 
\end{theorem}

As before, the term $V(\pi^*)$ in the upper bound roughly 
reflects the estimation error for $\hat{Q}(\pi^*)$, and $\beta$ arises from uniform 
concentration over the policy class $\Pi$.
The implication of Theorem~\ref{thm:adapt_upper} is that the suboptimality is small as long as the value of $\pi^*$ can be well estimated, 
which generalizes the ``oracle'' property we stated for the batched case to the 
much more challenging adaptive setting. 
Again, similar to Corollary~\ref{cor:general}, 
we can derive a generalized oracle property which states that with probability at least $1-\delta$, 
$Q(\hat\pi) - Q(\pi) \leq \min\{1, 2\beta \cdot V(\pi)\sqrt{ 2\log (30T\cdot V(\pi))}\}$ simultaneously 
for all $\pi\in\Pi$. This means 
the policy learned by PPL is comparable to any policy that is 
explored well in the adaptively collected dataset.


  The key result in Theorem~\ref{thm:adapt_upper} is $|\hat{Q}(\pi)-Q(\pi)|\preceq V(\pi)\sqrt{\log(V(\pi))}$. We call it a ``generalized empirical Bernstein's inequality'' because $V(\pi)^2$ is an empirical estimator for the variance of $\hat{Q}(\pi)$, and it generalizes existing empirical Bernstein's inequalities~\citep{bartlett2005local,bartlett2006convexity,koltchinskii2006local} to adaptive data and unbounded loss. 
  Without requiring any lower bounds on the adaptive sampling probability $e_t(x,a\given \cH_t)$, this result is applicable to any adaptive data collection process. 

  Compared with Theorem~\ref{thm:fix_upper} for non-adaptive data, the additional $\log(V(\pi))$ term  arises from applying a special self-normalized inequality (c.f.~Lemma~\ref{lemma:self_normal} in the supplementary material) we use to address unbounded loss. As discussed in Section~\ref{subsec:generic_algo}, we utilize the martingale structure of $\hat\Gamma_t(\pi)=\ind\{A_t=\pi(X_t)\}/e_t(X_t,A_t\given \cH_t)\cdot Y_t$ to address the adaptive nature of data. However, existing self-normalized inequalities often need to rely on constant bounds on the moments of martingale differences~\cite{de2004self}, which is not available  as we allow $1/e_t(X_t,A-t\given \cH_t)$ to grow arbitrarily and adaptively as $t$ proceeds. 
  To circumvent this limitation, we transform the original problem to bounding another class of martingales. Derived through a series of techniques, including symmetrization, their randomness is solely from a tree Rademacher process while conditioned on the data. This structure enables the application of  Lemma~\ref{lemma:self_normal} to unbounded martingales; a detailed proof based on these ideas is available in Appendix~\ref{subsec:proof_upper_adapt}. 

\begin{remark}\normalfont
  \label{rm:bound_adapt}
  The $\log V(\pi)$ term in Theorem~\ref{thm:adapt_upper} is small compared with $V(\pi)$, and can be upper bounded by poly-log$(T)$ terms under mild conditions. 
  For instance, if there exists a fixed constant $\alpha \geq 1$ 
  such that  
  $\log e_t(x,a\given \cH_t) \geq - (\log T)^\alpha$ 
  deterministically for all $t\in[T]$ and all $(x,a)\in \cX\times\cA$, 
  then it holds that $\log V(\pi)\leq (\log T)^\alpha$ and the upper bound in Theorem~\ref{thm:adapt_upper} becomes $O(V(\pi)\cdot (\log T)^{\alpha/2})$. Taking $\alpha=2$, this is to require $e_t(x,a\given \cH_t)\geq 1/T^{\log T}$, which is minimal for obtaining any meaningful exploration. 
  \end{remark}

\begin{remark}\normalfont
Since Theorem~\ref{thm:adapt_upper} does not require any assumption  on the
adaptive data collecting process, we obtain concrete upper bounds on the suboptimality 
even in situations where the conditions in the literature do 
not hold~\citep{zhan2021policy,bibaut2021risk}. 
The data-dependent upper bound frees us from 
imposing stringent assumptions on the adaptive propensities; as a result,  
the learning performance directly depends on the ``quality'' of the offline data.
\end{remark}

To gain a more quantitative understanding of our upper bound, 
we study a special case where the overlap of
the \emph{optimal}  
policy admits a polynomial-decaying lower bound.

\begin{assumption}[Polynomial decay] \label{assump:poly}
  Assume $e_t(X,\pi^*(X_t)\given \cH_t)\geq g_t$ almost surely 
  for a deterministic sequence $g_t:=\bar{c}\cdot t^{-\gamma}$ 
  for some constant $\bar{c}>0$ and  all $t\in[T]$. 
\end{assumption}

We emphasize that Assumption~\ref{assump:poly} does not 
require any of the suboptimal policies to have lower
bounded sampling probability,
and the standard assumption 
that 
$e_t(x,a\given \cH_t)\geq  \bar{c}\cdot t^{-\lambda}$ 
for some constant $\lambda \in(0,1)$ 
for \emph{all} $(x,a)\in \cX\times\cA$ almost surely~\citep{zhan2021policy,bibaut2021risk} 
is strictly stronger  than ours. 
Under Assumption~\ref{assump:poly},
we show below that 
the upper bound on the suboptimality of 
pessimistic policy learning is polynomial in $T$. 
Its proof is in Appendix~\ref{app:cor_adapt}. 

\begin{corollary}\label{cor:adaptive}
Suppose Assumption~\ref{assump:poly} holds, and 
we set  
\$\beta\geq c\cdot (\log T)^{\alpha/2} \cdot \sqrt{\ndim(\Pi)\log(TK^2) + \log(8/\delta)}
\$ 
for some $c\geq 67$, 
and suppose $TK^2\geq 8/\delta$. Then with probability at least $1-2\delta$,  
\$
\cL(\hat\pi;\cC,\Pi) \leq 
\min\bigg\{
12c \cdot \sqrt{ \frac{\ndim (\Pi)}{T^{1-\gamma}} } \cdot \frac{(\log (TK^2))^{(1+\alpha)/2} \cdot \log(1/\delta)}{\max\{1,\bar c^{ 3/4} \}},
1\bigg\}.
\$
\end{corollary}

Our upper bound in Corollary~\ref{cor:adaptive}
improves upon the results developed 
in~\cite{zhan2021policy,bibaut2021risk}, where they show  
the greedy approach ensures a suboptimality upper bound 
of order $\sqrt{\ndim(\Pi)/T^{1-\lambda}}$ 
when $\inf_{x,a}e_t(x,a\given \cH_t)\geq \bar{c}\cdot t^{-\lambda}$  
for some $\lambda\in(0,1)$. 
We discuss two examples where 
the improvement is significant.
First, when $\inf_{x,a}e_t(x,a\given \cH_t)\geq \bar{c}\cdot t^{-\lambda}$ does hold, 
and $\pi^*$ is explored better than this lower bound, so that Assumption~\ref{assump:poly} 
holds for some $\gamma < \lambda$. 
By Corollary~\ref{cor:adaptive}, pessimistic policy learning achieves a strictly 
sharper rate than greedy learning.  
Second,  when the assumptions in greedy learning  
do not hold (so that they cannot provide any guarantee),
our method may still enable efficient learning. 
For instance, in Thompson sampling~\citep{thompson1933likelihood,russo2018tutorial},
an arm that is considered suboptimal may have a vanishing sampling propensity. 
However, if the adaptive experiment turns out to 
successfully preserve the arms for the optimal policy $\pi^*$, 
Assumption~\ref{assump:poly} may still be true, and potentially even true for $\gamma=0$, 
leading to an efficient upper bound of order $O(\sqrt{\ndim(\Pi)/T})$.  

\begin{remark}
  With Theorems~\ref{thm:fix_upper} and~\ref{thm:adapt_upper}, the use of PPL may extend beyond contextual bandits to dynamic treatment regimes~\citep{chakraborty2014dynamic}. In particular, suppose there are $T$ units, each subject to a sequence of treatments (actions) $A_{t,k}\in \cA_k$, $k=1,\dots,N$ and receiving an outcome $Y_t$ after all actions are taken. The sampling probability of $A_{t,k}$ depends on $\cH_{t,k}:=\{X_{s,\ell},A_{s,\ell},Y_s\}_{s\leq t-1}\cup\{A_{t,\ell}\}_{\ell\leq k-1}$ and is assumed to be known. There are similar IPW-type estimators for evaluating a policy $\pi$ over the action sequence $(A_1,\dots,A_N)$, yet due to the long sequence of actions the lack of overlap is even more exacerbated. PPL with our data-dependent concentration inequality naturally applies to such estimators to enable efficient time-varying policy learning. On the other hand, IPW-type estimators are also used in offline RL~\citep{xie2019towards}; however, we anticipate that our ideas are only applicable when the state transition probabilities can be accurately estimated. 
\end{remark}

\subsection{Minimax optimality}

Finally, we investigate the minimax optimality of pessimistic policy learning  
when data is adaptively collected. 
As the adaptive propensities are difficult to characterize in general, 
we work under the polynomial decay setting for 
clear interpretability.   
As in the fixed behavior policy case, 
our  minimax lower bound unveils that 
the fundamental difficulty of policy learning with adaptive behavior policy
is actually captured by the sampling propensities for the optimal policy.

Given an action set $\cA$, a policy class $\Pi$, a sample
size $T \in \NN_+$ and some constants $\bar{c} > 0$
and $\gamma \in (0,1)$, we denote $\cR(\bar{c},\gamma,T,\cA,\Pi)$ 
as the set of instances we consider. Each element of 
the instance set corresponds to 
a contextual bandit model $\cC$ defined on $\cA$
and an adaptive sampling process obeying
\#\label{eq:adapt_lower_bd}
e_t(X,\pi^* (X)\given \cH_t) \geq \bar{c}\cdot t^{-\gamma},\quad a.s.,
\#
which together generates the offline data $\cD = \{(X_t,Y_t,A_t)\}_{t=1}^T$ 
as outlined in Section~\ref{sec:prelim}.
Here, $\pi^* =\pi^*(\cC,\Pi)$ is the optimal policy among $\Pi$ 
for the contextual bandit $\cC$ defined in~\eqref{eq:def_pi*}.    
The next theorem provides the minimax lower bound 
for learning any policy in $\Pi$ from an adaptive data collecting process
obeying~\eqref{eq:adapt_lower_bd}, which matches 
our upper bound in Corollary~\ref{cor:adaptive} up to logarithm factors.

\begin{theorem}\label{thm:lower_adapt}
For any action set $\cA$, any policy class $\Pi$,
any sample size $T \in \NN_+$, any $\bar{c}>0$ and any $\gamma\in(0,1)$ with 
$ (1-\gamma)\cdot \ndim(\Pi) \leq 1.5\bar{c} \cdot T^{1-\gamma}$, one has
\#
\inf_{\hat\pi} \sup_{(\cC,e)\in \cR(\bar{c},\gamma,T,\cA,\Pi)} 
\EE_{\cC,e} \big[  \cL(\hat\pi\,;  \cC,\Pi ) \big]
\geq 0.12 \cdot \sqrt{\frac{ \ndim(\Pi) }{ T^{1-\gamma}}} \cdot \sqrt{\frac{1-\gamma}{\bar{c}}} ,
\#
where $\ndim(\Pi)$ is the Natarajan dimension of $\Pi$, $\hat\pi$ is any 
data-dependent policy, and $\EE_{\cC,e}$ means taking expectation over the 
randomness of the offline data under $\cC$ and $e_t(\cdot,\cdot\given \cH_t)$. 
\end{theorem}

Theorem~\ref{thm:lower_adapt} matches our 
upper bound in Corollary~\ref{cor:adaptive}, 
which shows that PPL is minimax optimal for adaptively collected data.
Our lower bound is implied by 
the lower bound in~\cite{zhan2021policy} established for a family of problems with the uniform lower bound assumption 
$\inf_{x,a}e_t(x,a\given \cH_t)\geq \bar{c}\cdot t^{-\gamma}$. 
Indeed, we take the worst-case suboptimality over a larger set of instances 
obeying~\eqref{eq:adapt_lower_bd}. 
We include the proof in Appendix~\ref{app:subsec_lower_adapt} for completeness.  

In summary, our theoretical results   (1) describe the fundamental difficulty of offline policy learning 
in terms of the overlap of the optimal policy  (Theorem~\ref{thm:lower} for the fixed case 
and Theorem~\ref{thm:lower_adapt} for the adaptive case), and (2) show  that  
PPL is minimax optimal (Theorem~\ref{thm:fix_upper} for the fixed case and Theorem~\ref{thm:adapt_upper}
for the adaptive case). 
More specifically, these results concern a family of contextual bandit problems 
whose propensity of the {\em optimal policy} is lower bounded. 
While this family is larger than existing ones in the literature for both cases,  our lower bounds coincide with them. 
At a high level, this reveals the key role 
of the overlap of the \emph{optimal} policy in the fundamental limit of 
learning from an offline dataset. 
By providing an optimal solution to a broader family of problems, 
we show that some problems are not as hard as they have been perceived in the literature, especially when the uniform overlap is weak yet the overlap for the optimal policy is strong.

\begin{remark}\normalfont
  The improvement of PPL relative to the greedy approach 
  depends on the unknown overlap 
  of the optimal policy $\pi^*$, and   
  our minimax lower bound shows that PPL is the best effort 
  for a given dataset. 
  When one has the power to collect data, 
  an important follow-up question is how to design the behavior policy 
  to exploit the benefit of pessimism
  in subsequent policy learning. 
  By Theorem~\ref{thm:fix_upper} and~\ref{thm:adapt_upper}, 
  one should aim to ensure a good overlap for $\pi^*$.  
  In the batched setting with a fixed behavior policy, one may utilize external data to 
  guess and assign higher probability to the optimal arms. 
  In the adaptive setting, one may use online algorithms   
  to learn the optimal arms; one may also leverage 
  the quantification of uncertainty in~\eqref{eq:ineq_adapt} to guide the online exploration 
  to avoid strong modeling assumptions on the reward regression function. 
  We leave these aspects for future investigation. 
\end{remark}


\section{Practical algorithms}
\label{sub:alg}

Our theoretical results so far 
have addressed the minimax optimal rates of 
policy learning under weak overlap achieved by PPL. 
In complement, we now devise practical PPL algorithms, whose efficacy 
will be demonstrated in the next section via numerical experiments. 
In Section~\ref{subsec:alg}, we propose concrete 
optimization algorithms to approximate our initial 
non-convex program~\eqref{eq:def_pess}. 
In Section~\ref{subsec:cv}, we discuss a cross-validation scheme 
for choosing the hyperparameter in practice. 
It is worth noting that while these methods are heuristic and non-exact by nature, similar optimization and computational challenges are present in existing methods~\citep{zhou2022offline,swaminathan2015batch,jin2021pessimism} (see Appendix~\ref{app:subsec_opt_discuss} of the supplementary material for a detailed discussion). Consequently, we defer exact optimization and the theoretical aspects of PPL algorithm to future work. Nonetheless, our extensive experiments indicate that the proposed algorithms exhibit robust, superior performance and remain a practical choice.

\subsection{Optimization for PPL}
\label{subsec:alg}

The first practical challenge we address 
is that for a given parameter $\beta>0$, 
the penalty term $R(\pi)$ is non-convex in $\pi$, 
adding to the well-recognized  non-convexity of 
greedy empirical welfare maximization (i.e., without our penalty term)~\citep{zhou2022offline}. 
To this end, 
we leverage  a Majorization-Maximization scheme inspired by SVM approximation algorithms~\citep{van2016gensvm,swaminathan2017off} that 
progressively approximate the non-convex, non-additive penalty term~\eqref{eq:V_terms} 
with an additive objective that is compatible with 
the widely adopted tree search algorithm for policy learning~\citep{sverdrup2020policytree}. 

Specifically, tree search~\citep{sverdrup2020policytree} works for optimizing any 
additive objective of the form $\sum_{t=1}^T \Gamma_t(X_t,\pi(X_t))$
such as that in~\cite{zhan2021policy} over a tree policy class.
However, it is not readily applicable to PPL 
since our penalty term $R(\pi)=\beta\cdot V(\pi)$ is non-additive. 
To begin with, we ignore the higher order deviation in~\eqref{eq:V_terms} 
since it is usually of a smaller scale, and further replace  
$V(\pi)$ by $\vs(\pi)+\vp(\pi)$, noting that $V(\pi)\leq \vs(\pi)+\vp(\pi) \leq 2V(\pi)$. 
The following lemma finally bounds it by an additive form 
based on any starting point $\pi_0$. The proof adapts that of~\cite{swaminathan2017off} 
and is omitted here. 

\begin{lemma}
Denote $\Gams_t(\pi ) = \frac{\ind\{A_t = \pi(X_t)\}}{e(X_t,\pi(X_t)\given \cH_t)^2}$
and $\Gamp_t(\pi ) = \frac{1}{e(X_t,\pi(X_t)\given \cH_t)}$, so that 
$\vs(\pi) = \frac{1}{T}\sqrt{ \sum_{t=1}^T \Gams_t(\pi)}$ 
and $\vp(\pi) = \frac{1}{T}\sqrt{\sum_{t=1}^T \Gamp_t(\pi)}$. 
Then for any $\pi_0\colon \cX\to \cA$ such that $\vs(\pi_0 )>0$,
\$
\vs(\pi) \leq G_{\rm s}(\pi;\pi_0):=C_{\rm s} + \frac{B_{\rm s}}{T} \sum_{t=1}^T \Gams_t(\pi)^2,
\quad 
\vp(\pi) \leq G_{\rm p}(\pi;\pi_0):=C_{\rm p} + \frac{B_{\rm p}}{T} \sum_{t=1}^T \Gamp_t(\pi)^2,
\$
where $C_{\rm s}=\vs(\pi_0)/2$, $C_{\rm p} = \vp(\pi_0)/2$, 
$B_{\rm s} = (2T\vs(\pi_0))^{-1}$, and $B_{\rm p} = (2T\vp(\pi_0))^{-1}$.
\end{lemma}

The Majorization-Minimization scheme is proved to converge to local training minimum 
for cost functions in SVM methods~\citep{van2016gensvm}. 
Inspired by this fact, we use the following 
Majorization-Maximization scheme: starting with some $\pi_0$, $i=0$, 
we iteratively replace 
our original objective $\frac{1}{T}\sum_{t=1}^T \hat\Gamma_t(\pi) - \beta\cdot V(\pi)$ 
by its lower bound 
$\frac{1}{T}\sum_{t=1}^T \{\hat\Gamma_t(\pi) - \beta\cdot G_{\rm s}(\pi;\pi_i) - \beta\cdot G_{\rm p}(\pi;\pi_i)\}$, and use tree search to find the optimal solution $\pi_t$, and continue 
this process until a convergence criterion is met. 
We summarize the approximation algorithm in Algorithm~\ref{alg:ppl}, 
where we use $\texttt{Tree}(L,\Pi_D)$ 
to denote the tree search algorithm that maximizes the 
empirical welfare $L$ over the policy class $\Pi_D$, 
which is the collection of depth-$D$ trees~\citep{sverdrup2020policytree}. 
In our numerical experiments, 
we set $\pi_0$ as the greedy policy from tree search, 
and set the convergence criterion to be 
either $i>50$ or $\pi_i(X_t) = \pi_{i+1}(X_t)$ for all $t\in [T]$.

\begin{algorithm}[h] 
  \caption{Approximate Iterative Majorization-Maximization for PPL}\label{alg:ppl}
  \begin{algorithmic}[1]
  \REQUIRE Training data $\{X_t,A_t,Y_t\}_{t=1}^T$, convergence criterion, tree search algorithm $\texttt{Tree}(\cdot;\Pi)$ and depth of tree class $D$, parameter $\beta>0$.
  \vspace{0.05in} 
  \STATE Set $i=0$, and compute $\pi_0 = \texttt{Tree}(\sum_{t=1}^T \hat\Gamma_t(\pi);\Pi_D)$ according to~\eqref{eq:def_hatQ_generic}.  \hfill \texttt{// Initialization}

  \vspace{0.3em}

  \WHILE{Convergence criterion not met} 
  \STATE Compute $B_{\rm s} = (2T \vs(\pi_i))^{-1}$ and $B_{\rm p} = (2T\vp(\pi_i))^{-1}$.  \hfill \texttt{// Iterative update}
  \STATE Set $\Gams_t(\pi ) = \ind\{A_t = \pi(X_t)\}/e(X_t,\pi(X_t)\given \cH_t)^2$.
  \STATE Set welfare function $\hat\Gamma_{t,i}(\pi) = \hat\Gamma_t(\pi) - \beta \frac{B_{\rm s}}{T}\Gams_t(\pi) - \beta \frac{B_{\rm p}}{T}\Gamp_t(\pi)$. 
  \STATE Compute $\pi_{i+1} = \texttt{Tree}(\sum_{t=1}^T \hat\Gamma_{t,i}(\pi);\Pi_D)$.
  \STATE $i \leftarrow i+1$.
  \ENDWHILE
  \vspace{0.3em}
  \ENSURE Learned policy $\pi_i$.
  \end{algorithmic}
\end{algorithm}

\subsection{Cross-validation for hyper-parameter}
\label{subsec:cv}

Having addressed the optimization with a fixed $\beta>0$, 
the second challenge we tackle is the choice of the parameter 
in practice. Note that Theorems~\ref{thm:fix_upper}
and~\ref{thm:adapt_upper} provide recipes for choosing $\beta$ 
that achieves the minimax optimal statistical rates, but 
the specific value of $\beta$ that leads to the actual optimal policy 
value may depend on the problem context. 
Also, the necessity of computing the Natarajan dimension 
can be undesirable in spite of existing bounds~\citep{daniely2011multiclass,jin2023upper}. 

A natural idea is to choose the parameter by cross-validation. 
For batched data, one can directly use classical cross-validation methods~\citep{stone1974cross,hastie2009elements}. 
In addition, here we develop a  heuristic cross-validation strategy 
for adaptively collected data, summarized in Algorithm~\ref{alg:cv}. 
We find it effective in our numerical experiments, whose theoretical properties are left for future work.

\begin{algorithm}[htbp] 
  \caption{Cross validation for PPL with adaptive data}\label{alg:cv}
  \begin{algorithmic}[1]
  \REQUIRE Training data $\{X_t,A_t,Y_t\}_{t=1}^T$, candidate parameter list $\cB$, 
  number of folds $N$, convergence criterion, tree search algorithm $\texttt{Tree}(\cdot;\Pi)$ and depth of tree class $D$.
  \vspace{0.05in} 
  \STATE Split $\{1,\dots,T\}$ into $N$ folds  
  $\cI_j = \{\lfloor (j-1) T/N \rfloor+1,\dots, \lfloor j T/N \rfloor\}$, $j=1,\dots,N$. \hfill \texttt{// Data split}

  \vspace{0.3em}

  \FOR{$\beta$ in $\cB$} 
    \FOR{$j=1,\dots,\lfloor 3N/4\rfloor$}
    \STATE Set training fold $\cI_j^{\rm train} = \cup_{\ell\leq j} \cI_\ell$.
    \STATE Set evaluation fold $\cI_j^{\rm eval} = \cup_{\ell> j} \cI_\ell$.
  \STATE Learn policy $\hat\pi_{j,\beta}$ with Algorithm~\ref{alg:ppl} applied to training data $\{X_t,A_t,Y_t\}_{t\in \cI_j^{\rm train} }$.  \hfill \texttt{// Training}
  \STATE Evaluate $\hat{Q}_{j,\beta} = \frac{1}{|\cI_j^{\rm eval}|}\sum_{t\in \cI_j^{\rm eval}} \hat\Gamma_t(\hat\pi_{j,\beta})$. \hfill \texttt{// Evaluation}
    \ENDFOR
    \STATE Compute average evaluated reward $\hat{Q}_\beta = \frac{1}{\lfloor 3N/4\rfloor}\sum_{j=1}^{\lfloor 3N/4\rfloor} \hat{Q}_{j,\beta}$.
  \ENDFOR
  \vspace{0.3em}
  \STATE Compute $\hat\beta^{\rm cv} = \argmax\{ \hat{Q}_\beta\colon \beta\in \cB\}$. 
  \ENSURE Cross-validated hyper-parameter $\hat\beta^{\rm cv}$. 
  \end{algorithmic}
\end{algorithm}

The adaptive setting requires additional care due to its sequential nature: 
with dependent data, one cannot randomly split the data points. 
Instead, we are to split the data into $N$ folds without permutation, 
so that the folds containing consecutive data points preserve the sequential structure. 
Then, we repeatedly use the first $j$ folds 
to learn the policy with a specific parameter, 
use the remaining data to evaluate the learned policy, 
and aggregate the evaluation to select the parameter. 
In this way, the martingale structure ensures 
unbiased evaluation of the policy learned with earlier folds.
Furthermore, for the stability in weak overlap settings, 
we only take $j\leq 3N/4$ to ensure sufficiently large evaluation folds. 

\section{Numerical experiments}
\label{sec:exp}

In this part, we evaluate the proposed PPL algorithms 
via extensive numerical studies.  
In 
Section~\ref{subsec:simu_mab}, we use several stylized scenarios 
with non-adaptive propensities
to build intuitions for the behavior of PLL.   
Section~\ref{subsec:simu_contextual} evaluate 
the finite-sample performance of PPL and 
shows its efficacy in a wide range of complex scenarios 
where data are adaptively collected by Thompson Sampling.  
Finally, in Section~\ref{subsec:real}, we apply PPL to 33 
real-world datasets on OpenML~\citep{OpenML2013} which shows its 
practical relevance.\footnote{Reproduction code is available in the GitHub repository~\url{https://github.com/ying531/pess-policy-learning}.}

\subsection{Simulations with non-adaptive propensities}
\label{subsec:simu_mab}

We first consider offline datasets collected by non-adaptive propensities, 
i.e., $e_t(x,a\given \cH_t)$ may change with $t$ but does not depend on historical data. 
We use such well-controlled settings to build intuitions on 
what should be expected from the algorithms
and verify our theoretical results. 

\subsubsection{Multi-armed bandits}
\label{subsubsec:mab}
We first consider 5-armed bandits without covariates. 
For a fixed vector $\mu = (\mu_1,\dots,\mu_5)\in \RR^5$, 
we sample $T$ i.i.d.~data from a fixed behavior policy with 
probability $e_k$ for each arm $k\in[5]$.  
To ensure that we are in an interesting asymptotic 
regime, in each experiment with sample size $T$, 
the mean vector for the arms is set to be $\mu/\sqrt{T}$.  
We suppose the data are sampled from a fixed behavior policy with 
probability $e_k$ for arm $k$, $k\in[5]$. 
Given the data, a greedy learner (called GPL hereafter for simplicity) 
takes $\argmax_{k}\{ \hat\mu_k\}$, where 
$\hat\mu_k$ is the   mean reward for data with pulled arm $k$. 
In contrast, PPL takes $\argmax_k \{ \hat\mu_k - \hat\sigma_k/\sqrt{T}\}$, 
where $\hat\sigma_k$ is the standard deviation of 
the observed rewards for arm $k$. 
This is a simplified implementation of PPL, which nevertheless delivers 
essential messages.  
We design three  data-generating processes:

\begin{itemize}
  \item \emph{Setting 1: Optimal}. Set $\mu = (0, 0.05, 0.01, 0, -0.01)$ and $e = (0.07, 0.9, 0.01, 0.01, 0.01)$. 
  The optimal arm is sampled with high probability, while the second-optimal is sampled with a small probability. According to Section~\ref{subsec:challenge}, GPL would be confused by the poor estimation of the second-optimal arm. 
  \item \emph{Setting 2: Suboptimal}. Set $\mu = (0, 0.05, 0.04, 0.01, -0.01)$ and $e =(0.07, 0.1, 0.8, 0.02, 0.01)$. 
  With weak overlap, GPL might confuse the two, but our theory implies that PPL would learn an arm that is very close to arm 3 (as the estimation error of arm 3 is small). 
  \item \emph{Setting 3: Uniform}. Set $\mu=(0, 0.05, 0.03, 0.01, -0.01)$ and $e_k\equiv 0.2$. 
  This means all arms are sampled equally well, and we expect both GPL and PPL to perform well. 
\end{itemize}

Figure~\ref{fig:mab} plots the rescaled suboptimality $\sqrt{T}\cdot(\mu^* - \hat\mu)$ 
for GPL and PPL with various choices of $\beta$ and sample sizes,  
where $\mu^*=\max_k \{\mu_k\}$ and $\hat\mu$ is the true reward for the learned arm. 
Setting 3 confirms the comparable performance of greedy and PPL 
given high-quality data. 
However, once the overlap is not uniformly good (settings 1 and 2), PPL show superior performance especially with reasonably large $\beta$. 
The behavior of PPL and GPL is demonstrated in Appendix~\ref{subsec:map_freq}. 
In setting 1, GPL misses the optimal arm 
even if it is pulled many times due to large estimation uncertainty of 
the second-optimal one (see our discussion in Section~\ref{subsec:challenge}), 
while PPL overcomes this issue. 
In setting 2, although both methods miss the optimal arm due to poor overlap, 
PPL can still find the second-optimal arm whose overlap is sufficiently good 
(see Corollary~\ref{cor:general}). 
Finally,   
 GPL and PPL perform comparably in setting 3 where all arms are sampled uniformly well.

In Appendix~\ref{eq:subsubsec_MAB_clip}, 
we compare PPL with~\cite{swaminathan2015batch} where the inverse propensities are first clipped to a maximum of $M=5$ (which can be viewed as another hyper-parameter to tune in their framework) and then penalized welfare maximization similar to PPL is applied. In setting 1-2, we observe inferior performance of their method due to the bias such clipping introduces, which highlight the importance of properly dealing with weak overlap even in the simplest i.i.d.~settings. In general, tuning the parameter $M$ may introduce another layer of complexity in practice, whereas PPL achieves automatic tradeoff between bias and variance.

\begin{figure}[h]
  \centering 
  \includegraphics[width=0.8\linewidth]{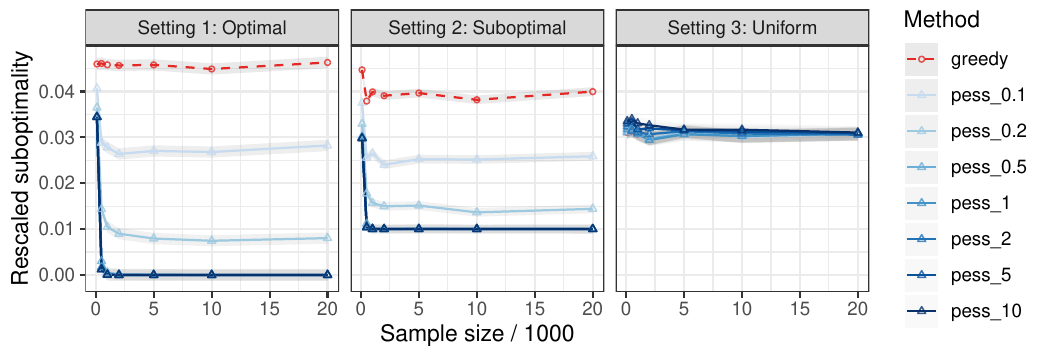}
  \caption{Average rescaled suboptimality of the learned policy  from GPL 
  and PPL with $\beta\in \{0.1, 0.2, 1, 10\}$ at various sample sizes. The confidence bands shows mean $\pm$ two times the empirical standard deviation over $N=1000$ independent runs.}
  \label{fig:mab}
\end{figure}

\subsubsection{Contextual bandits with tree search}
\label{subsec:simu_dt}
We then study  $10$-arm
contextual bandits, with policy learned by 
tree search~\citep{sverdrup2020policytree}. 
In this part, 
we still control the propensity sequence $e_t(x,a)$, 
but the policy search is more complex. 

We fix the number of arms at $K=10$ and data dimension at $p=2$, 
extending the setup of~\cite{zhan2021policy} to a larger problem scale. 
Let $\{a_k\}_{k=1}^{10}$ be equally spaced within $[-1,1]$. 
At each time $t$, one randomly pulls an arm $A_t$ 
with probability $e_t(x,a)$ for $a\in \cA$, 
and the reward is generated via $Y_{t} = \mu(X_t,A_t)+\epsilon_t$, 
with $X_{t,j}\iid \textrm{Unif}[-2,2]$ and $\epsilon_t\iid N(0,0.1^2)$. 
We set nonlinear rewards $\mu(x,a_k) = 1-\alpha_k/2 + \alpha_k x_1^2/2$, 
and vary the time-dependent propensity sequence in three ways:
\begin{itemize}
  \item Setting 1 (Optimal overlap): We fix $e_t(x,a_k) = 0.1 \cdot t^{-\alpha}$ 
  if $k\neq \argmax_{\ell} \{\mu(x,a_\ell)\}$ 
  and $e_t(x,a_k) = 1-0.9\cdot t^{-\alpha}$ 
  for the optimal arm, while varying $\alpha\in \{0.2,0.4,0.6,0.8\}$. 
  Thus, the optimal arms 
  has good overlap, but the suboptimal ones are pulled with decaying rates. 
  \item Setting 2 (Half-half overlap): We split the entire trajectory into 5 batches, 
  and alternately in each batch, we let $e_t(x,a_k) = 0.2(1-   t^{-\alpha})$
  if $\mu(x,a_k)\geq \textrm{Median}( \{\mu(x,a_\ell)\}_{\ell})$ 
  and  $e_t(x,a_k) = 0.2 \cdot t^{-\alpha}$ otherwise, or vice versa, varying $\alpha\in \{0.2,0.4, 0.6,0.8\}$. This mimics a case where the agent alternates between 
  sampling   high-valued  and low-valued arms. 
  \item Setting 3 (Worst overlap): We fix $e_t(x,a_k) = 0.2 \cdot t^{-\alpha}$ 
  if $\mu(x,a_k)\geq \textrm{Median}( \{\mu(x,a_\ell)\}_{\ell})$ 
  and $e_t(x,a_k) = 0.2(1-   t^{-\alpha})$ otherwise, varying $\alpha\in \{0.2,0.4, 0.6,0.8\}$. 
  In this case, the overlap for arms with large rewards decays quickly: 
  for $\alpha=0.2$ and $T=1000$, the optimal arms are only pulled about $25$ times. 
\end{itemize}


Besides four scenarios with various decay rates $\alpha$, we additionally study a ``pure exploration'' scenario 
where  $e_t(x,a_k)=0.001$ for low-propensity arms 
and $e_t(x,a)=0.99$ for high-propensity arms in each setting, in order 
to show the most drastic violation of uniform overlap. 
Given the data, we apply GPL (the method of~\cite{zhan2021policy}) and PPL 
with $\Pi$ being the hybrid policy tree with depth $5$, 
optimized using PolicyTree~\citep{sverdrup2020policytree}. 
As the tree search algorithm is not exact, 
there might be mis-specification in the policy class. 

In addition to GPL and PPL, we evaluate the linear PEVI method in \cite{jin2021pessimism} that addresses the overlap issue via the pessimism principle with linear function approximation  $\mu(x,a) = \phi(x,a)^\top \theta$, where $\theta$ is an unknown parameter (the method of \cite{swaminathan2015batch} evaluated in the preceding part is not applicable to changing behavior policies). Following~\cite{jin2021pessimism}, we set the feature vector as $\phi(x,a)\in \RR^{20}$ where the $(2i-1)$-th and $(2i)$-th elements in $\phi(x,a)$ are $(1,x)$ for the $i$-th action $a\in \cA$, so the linear model is mis-specified. Then, a pessimistic point-wise estimation $\hat\mu(x,a) = \phi(x,a)^\top\hat\theta - \beta\cdot |\phi(x,a)^\top \Lambda^{-1}\phi(x,a)|^{1/2}$ is obtained, where $\hat\theta$ is the ridge regression estimator, $\Lambda$ is the augmented covariance matrix, and  $\beta>0$ is a tuning parameter. The learned policy is given by $\hat\pi(x)=\argmax_{a\in \cA} \hat\mu(x,a)$. 
This method differs from PPL in both the policy class searched over and  the theoretical assumptions needed for efficient learning. 
The performance of the three methods across all settings 
is shown in Figure~\ref{fig:simu_dt}.

\begin{figure}[htbp]
  \centering 
  \includegraphics[width=\linewidth]{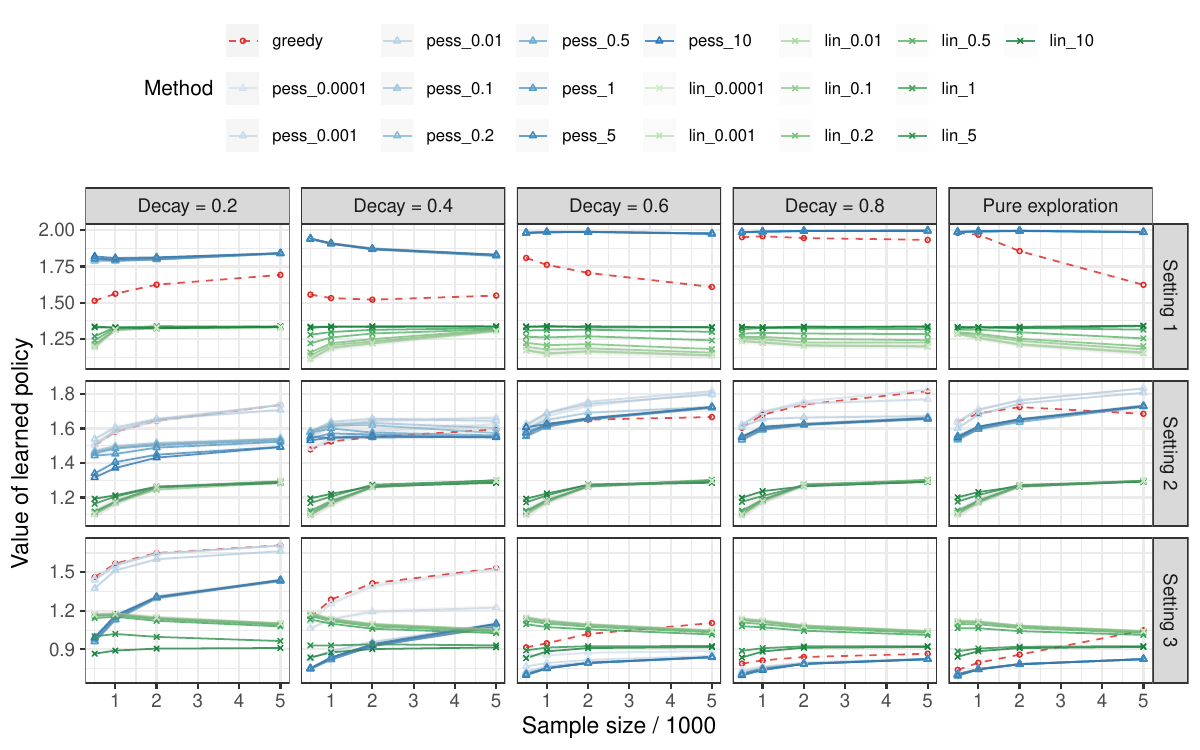}
  \caption{Evaluated value of policy learned from GPL (\texttt{greedy}), PPL (\texttt{pess}), and the linear method of~\cite{jin2021pessimism} (\texttt{lin}) with  $\beta\in \{0.1,0.2,0.5,1,5,10\}$, averaged over $N=200$ runs. Each subplot shows one sampling scheme.}
  \label{fig:simu_dt}
\end{figure}

In setting 1 with optimal overlap,  
PPL stably outperforms both GPL and PEVI.
A moderate decay rate like $0.4$ and $0.6$  
already imposes challenges to a greedy learner. 
GPL performs better under a higher decay rate of $0.8$: 
in this case, suboptimal arms are rarely sampled for any $x$, thus 
the AIPW estimator favors optimal ones. 
However, in the pure exploration setting, with tiny propensities 
and large sample sizes, suboptimal arms that occur in the offline data 
by chance could incur a huge 
estimation error in the AIPW estimator. 
Due to mis-specification in the linear model, PEVI 
suffers from a large bias.
In setting 2 where  each arm 
alternates between good overlap and poor overlap, 
both methods start to struggle, yielding lower values than setting 1. 
In this case, PPL with a small value of $\beta$ outperforms GPL, 
but a large penalty can be harmful. 
PEVI still suffers from mis-specification bias.
Finally, in setting 3 where optimal arms are very rarely taken, 
both GPL and PPL lead to poor learning performance  
(PPL  appears comparable to GPL especially with small penalty), while 
PEVI slightly outperforms the two tree-based methods despite reduced performance compared with settings 1 and 2. While PEVI is only slightly better than random guess (a uniform policy over $\cA$ would yield a reward of $1$), these results demonstrate the stability of linear methods in the most challenging--nearly hopeless--scenarios.

\subsection{Simulations with Thompson Sampling}
\label{subsec:simu_contextual}

We then move on to contextual bandits 
with data collected by an adaptive linear Thompson Sampling (TS) agent~\citep{agrawal2013thompson}, 
following the simulations in~\cite{zhan2021policy}. 
As we now have less control over the sampling process, 
our goal is to demonstrate the \emph{finite-sample} performance 
of PPL in relatively complex settings 
with weaker overlap and various levels of mis-specification in exploration. 


We consider $10$-arm contextual bandits similar to Section~\ref{subsec:simu_dt}. 
Let $\{a_k\}_{k=1}^{10}$ be equally spaced within $[-1,1]$. 
At each time $t$, the linear TS agent randomly pulls arm $A_t$, 
and the reward is generated via $Y_{t} = \mu(X_t,A_t)+\epsilon_t$, 
with $X_{t,j}\iid \textrm{Unif}[-2,2]$ for $j=1,2$, and $\epsilon_t\iid N(0,0.1^2)$.  
We design three data generating processes 
with varied  adaptivity of the agent: 
 
\begin{itemize}
  \item \emph{Setting 1: well-specified exploration.} We set  $\mu(x,a_k) = 1-\alpha_k/2 + x_1/2 - x_2$, so that the linear TS is well-specified and the sampling process settles down to the optimal arms, but the overlap of sub-optimal arms can be poor unless a lower bound is enforced. 
  \item \emph{Setting 2: misspecified with optimal overlap.} We set nonlinear rewards $\mu(x,a_k) = 1-\alpha_k/2 + \alpha_k x_1^2/2$, so that the linear TS is mis-specified. After calculating the adaptive propensities $\bar{e}(x,a\given \cH_t)$ by TS, we linearly rescale it to make sure the optimal arm is pulled with probability at least $0.1$. 
  \item \emph{Setting 3: misspecified exploration.} We set nonlinear rewards $\mu(x,a_k) = 1-\alpha_k/2 + \alpha_k x_1^2/2$, so that the linear TS is mis-specified, and the sampling process might not concentrate on optimal arms. 
\end{itemize}

In each setting, we enforce deterministic lower bounds on 
$e(x,a\given \cH_t)$ in three ways: 
(i) Pure exploration: no lower bound, so $e(x,a\given \cH_t)$ can decay very fast. 
(ii)-(iv): Polynomial decay: linearly rescaling the raw propensity by linear TS 
so that $e(x,a\given \cH_t)\geq   T^{-\alpha}/K$, where 
$\alpha\in \{0.2, 0.5, 0.8\}$. 
The decay rate of $0.5$ is studied in the simulations of~\cite{zhan2021policy}, 
and we investigate a broader class of settings. 
The larger $\alpha$ is, the more imbalanced the overlap is. 
In addition, we set the agent to update its TS sampler (i.e., linear estimators 
for arm means) with batch sizes in $\{10,100\}$, where 
a smaller batch size allows more frequent updates hence stronger adaptivity.  
The tree-based policy search is the same as Section~\ref{subsec:simu_dt}.
The experiments are repeated for $N=200$ times. 


\subsubsection{Well-specified exploration}

Figure~\ref{fig:simu_lin} shows the average reward 
for GPL and PPL when the linear TS sampler is well-specified,
across the  eight configurations 
of propensity lower bound (column) and batch size (row). 
In this setting, the true optimal arms have a good overlap, while uniform overlap 
might not be as good. 

The results confirm the \emph{superior finite-sample performance} of 
PPL across all settings. 
Consistent with Setting 1 (optimal actions have good overlap) 
in Section~\ref{subsec:simu_dt}, 
a polynomial decay rate of $0.5$ 
seems most challenging for GPL with TS sampler, 
which is also the setting where PPL yields the largest improvement. 
In specific, even though the linear TS sampler is well-specified, 
decaying overlap for suboptimal arms still 
confuses the greedy learner, and its performance even decreases with the sample size. 
The pure exploration setting is also challenging for GPL,  
but the performance of PPL improves with the sample size. 
In addition, PPL slightly improves as the exploration becomes more aggressive, 
where we observe that the overlap for optimal arms improves. 
Both algorithms perform stably across 
batch sizes.  

\begin{figure}[h]
  \centering 
  \includegraphics[width=0.9\linewidth]{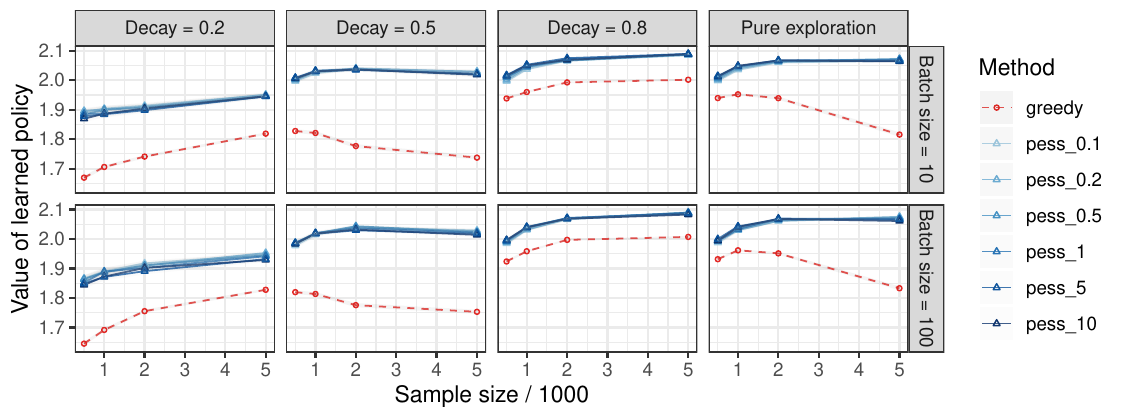}
  \caption{Evaluated value of policy learned from GPL and PPL with $\beta\in \{0.1,0.2,0.5,1,5,10\}$ for well-specified exploration, averaged over $N=200$ runs. Shaded bands are $\pm 2$ times standard deviation. Each column corresponds to one exploration scheme, and each row corresponds to one batch size.}
  \label{fig:simu_lin}
\end{figure}

\subsubsection{Mis-specified with optimal overlap}
\label{subsec:simu_cb_opt}

Figure~\ref{fig:simu_opt} shows the results when 
the linear TS sampler is mis-specified but 
the optimal action is always well-explored, meaning that 
$V(\pi^*)$ is small. 
The pattern we observe is still consistent with 
Setting 1 in Section~\ref{subsec:simu_dt}. 
This setting is strictly more challenging than the preceding one, 
but PPL maintains its superior finite-sample performance for 
appropriate choice of $\beta$. 
We also observe that PPL is now more sensitive to the choice of $\beta$. 
The decay rate of $0.5$ still seems most challenging for GPL, 
under which PPL also yields the most significant improvement.

\begin{figure}[h]
  \centering 
  \includegraphics[width=0.9\linewidth]{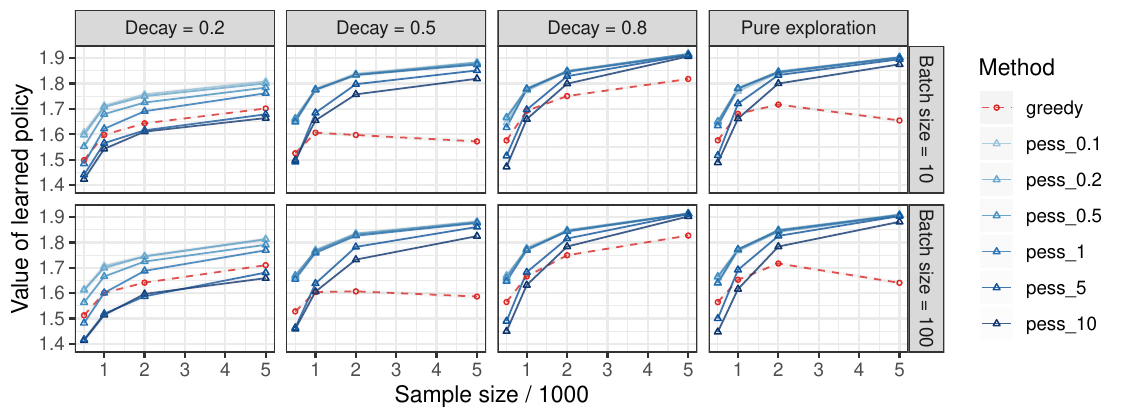}
  \caption{Evaluated value of policy learned from GPL and PPL with mis-specified exploration yet optimal overlap. Details are otherwise the same as Figure~\ref{fig:simu_lin}.}
  \label{fig:simu_opt}
\end{figure}

\subsubsection{Entirely mis-specified exploration}

Figure~\ref{fig:miss} shows the reward for GPL and PPL 
when the TS sampler is entirely mis-specified. 
This is intuitively the most challenging setting for both, since 
the exploration can be highly imbalanced. 
However, we empirically find that 
the optimal arms are still sampled sufficiently many times, 
hence the overlap  for the optimal policy is still 
good, yet not for all arms. 
As such, we still observe a similar pattern as 
Setting 1 in Section~\ref{subsec:simu_dt}: 
GPL performs relatively well with a decay rate of $0.8$ compared with other cases; 
we conjecture 
implicit regularization in the tree search algorithm. 
GPL can again fail with pure exploration 
when the uniform overlap further worsens. 
In contrast, PPL yields reliable performance in all settings.  
In addition, 
GPL already performs well when the decay rate is $0.2$, yet PPL still yields slight improvement 
in finite-sample, which requires a smaller penalty parameter $\beta$. 
We also note that the benefit of pessimism is more significant 
for  larger sample sizes, under which 
the uniform overlap is worse.

\begin{figure}[h]
  \centering 
  \includegraphics[width=0.9\linewidth]{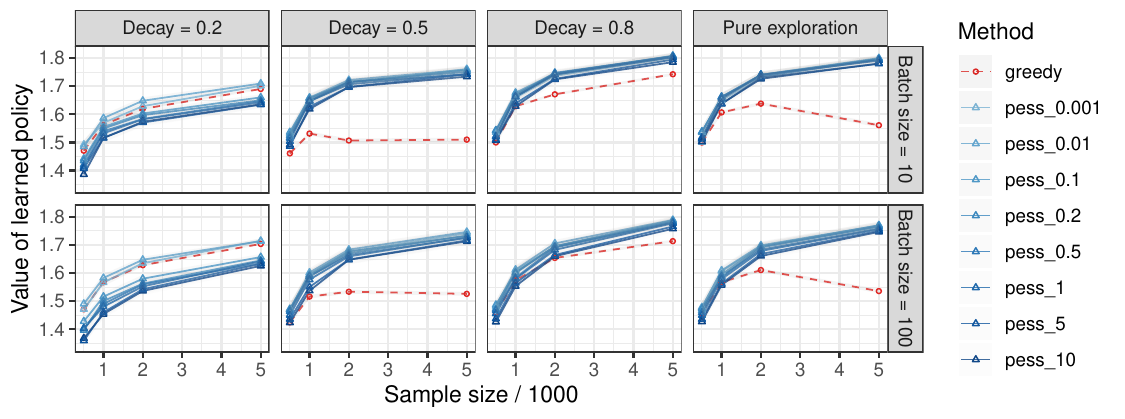}
  \caption{Evaluated value of policy learned from GPL and PPL with entire mis-specified exploration. Details are otherwise the same as Figure~\ref{fig:simu_lin}.}
  \label{fig:miss}
\end{figure}

\subsubsection{Results for cross-validation}

Finally, we evaluate whether cross-validation performs reasonably well 
in our simulations. We focus on Setting 2 (misspecified with optimal overlap)  where the choice of $\beta$ makes a significant difference, 
and only show the results 
with batch size equal to $10$ for brevity. 
We 
apply the 5-fold cross-validation in Algorithm~\ref{alg:cv} to determine the 
parameter $\hat\beta^{\rm cv}$, and run PPL with $\hat\beta^{\rm cv}$ 
on the same training data. 
The policy value is again computed on the evaluation fold.

Figure~\ref{fig:cv} shows the average evaluated reward 
when the hyper-parameter $\beta$ 
are chosen by 5-fold cross-validation. We observe that 
our cross-validation strategy is effective in selecting the 
hyper-paramter with near-optimal choice, leading to significant 
improvement upon GPL. 

\begin{figure}[h]
  \centering 
  \includegraphics[width=0.8\linewidth]{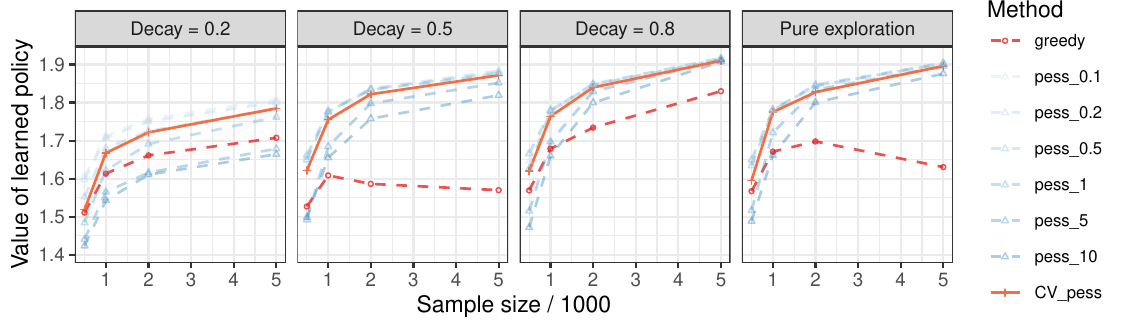}
  \caption{Evaluated value of policy learned from GPL (red, dashed line) and PPL with 5-fold cross-validation in Algorithm~\ref{alg:cv} (orange, solid line) under the setting of Section~\ref{subsec:simu_cb_opt}. The light blue curves are PPL with various fixed hyper-parameters shown for reference. Details are otherwise the same as Figure~\ref{fig:simu_lin}.}
  \label{fig:cv}
\end{figure}

\subsection{Real data experiments}
\label{subsec:real}

Finally, we test the performance of PPL 
on 33 real-world classification datasets on OpenML~\citep{OpenML2013} 
(a subset of datasets used in \cite{zhan2021policy} 
subject to availability at the time of this paper), following 
the classficiation-to-bandit transformation in \cite{zhan2021policy}. 

In specific, we treat each class as an arm, 
and generate $\mu(X_t,a)=1$ if the original label is $a$
and $0$ otherwise, as well as the semi-synthetic rewards 
$Y_t=\mu(X_t,A_t)+\epsilon_t$, where $\epsilon_t\iid N(0,0.1^2)$. 
We randomly split the data into  training   and evaluation folds.
We focus on the adaptive setting where 
a linear Thompson Sampling agent adaptively collects data on 
the training fold 
similar to Section~\ref{subsec:simu_contextual}. 
Finally, we apply GPL and PPL  
to the training fold ($\beta\in\{0.1, 0.2,0.5, 1,2,5,10,15\}$ chosen by 5-fold cross validation in Algorithm~\ref{alg:cv}), 
and compute the average reward of the learned policy on the evaluation fold. 
We plot the absolute improvement in the evaluated policy value 
of PPL upon that of GPL in Figure~\ref{fig:real}.

\begin{figure}[h]
  \centering 
  \includegraphics[width=0.9\linewidth]{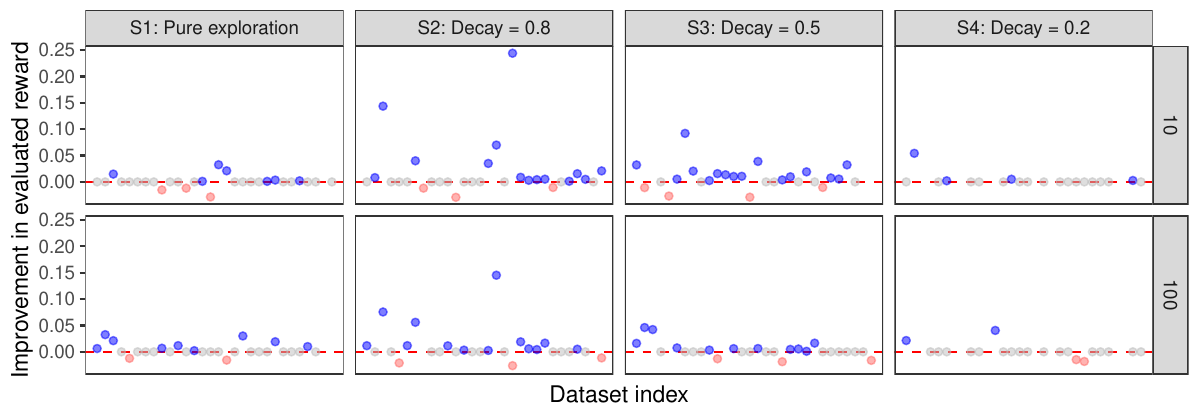}
  \caption{Improvement in the evaluated policy value learned from PPL (with $\beta$ chosen by 5-fold cross validation) compared with that from the greedy algorithm for all real datasets. Each column corresponds to one exploration scheme, and each row corresponds to one batch size ($10$ or $100$). The $x$-axis is the index for a dataset.}
  \label{fig:real}
\end{figure}

The message from the real datasets is largely consistent with 
our simulation studies. 
PPL outperforms GPL in many of the real datasets (marked in blue). 
In particular, the margin is significant 
when the decay rate is moderately large, such as $0.5$ and $0.8$. 
GPL and PPL achieve comparable performance at a decay rate of $0.2$, 
where there is no huge gap between the uniform overlap and optimal overlap. 
The results also show that our cross-validation strategy 
is able to find hyper-parameters that lead to favorable performance.

\section*{Discussion}
In this paper, we introduce a new algorithm based on the pessimism principle 
for policy learning 
from offline data (collected by either fixed or adaptive behavior policies). 
We show that the proposed  
procedure promises efficient learning {\em as long as}
the optimal policy is sufficiently covered by the offline
data set, in stark contrast to existing methods in the literature
that require uniform coverage for all policies. 
Our theoretical guarantee builds upon self-normalized
processes; as the key technical tool, we develop a generalized empirical 
Bernstein inequality that 
provides uncertainty quantification for estimators with unbounded empirical losses. 
We also devise practical algorithms and demonstrate their relevance 
via extensive experiments. 
We  close the paper by pointing out several potential future directions. 

\paragraph{Extension to ERM.}  
In this work, we have primarily focused on 
policy learning, which is a special case
within the general ERM framework. 
An interesting future research avenue is 
to extend our method to more 
general ERM problems such as regression~\citep{swaminathan2015batch}. 
The challenge there is to establish uniform concentration inequalities
for more complex function classes, 
for which additional techniques might be 
needed.

\paragraph{Efficient and scalable algorithms.} 
We present both theoretical aspects of pessimistic policy learning 
and practical implementation with policy tree search. 
A natural follow-up step is to improve the optimization component of our algorithm 
and develop variants that are  applicable to more general 
policy classes while preserving 
similar theoretical properties.  

\paragraph{Extension to sequential decision making.}
This work considers policy optimization in contextual bandits 
using IPW-type estimators, where we 
design a ``model-free'' pessimistic penalty without referring to 
parametric forms or function approximation. 
It will be interesting to see if this approach 
can be applied to reweighting-based methods in sequential decision making such as 
offline reinforcement learning.


\subsection*{Acknowledgement}
Z.~R.~was supported by supported by the Office of Naval Research via grant N00014-20-1-2337.
The authors thank Kevin Guo, Lihua Lei, Aaditya Ramdas, Dominik Rothenh\"ausler, Csaba Szepesv\'ari, 
Stefan Wager, Davide Viviano, 
and Ruohan Zhan
for helpful discussions and feedbacks.

\newpage
\bibliographystyle{apalike}
\bibliography{reference}

\begin{thebibliography}{}

\bibitem[Agrawal and Goyal, 2013]{agrawal2013thompson}
Agrawal, S. and Goyal, N. (2013).
\newblock Thompson sampling for contextual bandits with linear payoffs.
\newblock In {\em International conference on machine learning}, pages 127--135. PMLR.

\bibitem[Anscombe, 1963]{anscombe1963sequential}
Anscombe, F. (1963).
\newblock Sequential medical trials.
\newblock {\em Journal of the American Statistical Association}, 58(302):365--383.

\bibitem[Athey and Imbens, 2017]{athey2017econometrics}
Athey, S. and Imbens, G.~W. (2017).
\newblock The econometrics of randomized experiments.
\newblock In {\em Handbook of economic field experiments}, volume~1, pages 73--140. Elsevier.

\bibitem[Athey and Wager, 2021]{athey2021policy}
Athey, S. and Wager, S. (2021).
\newblock Policy learning with observational data.
\newblock {\em Econometrica}, 89(1):133--161.

\bibitem[Bartlett et~al., 2005]{bartlett2005local}
Bartlett, P.~L., Bousquet, O., and Mendelson, S. (2005).
\newblock Local rademacher complexities.
\newblock {\em The Annals of Statistics}, 33(4):1497--1537.

\bibitem[Bartlett et~al., 2006]{bartlett2006convexity}
Bartlett, P.~L., Jordan, M.~I., and McAuliffe, J.~D. (2006).
\newblock Convexity, classification, and risk bounds.
\newblock {\em Journal of the American Statistical Association}, 101(473):138--156.

\bibitem[Bembom and van~der Laan, 2008]{bembom2008data}
Bembom, O. and van~der Laan, M.~J. (2008).
\newblock Data-adaptive selection of the truncation level for inverse-probability-of-treatment-weighted estimators.

\bibitem[Bertsimas et~al., 2017]{bertsimas2017personalized}
Bertsimas, D., Kallus, N., Weinstein, A.~M., and Zhuo, Y.~D. (2017).
\newblock Personalized diabetes management using electronic medical records.
\newblock {\em Diabetes care}, 40(2):210--217.

\bibitem[Bibaut et~al., 2021]{bibaut2021risk}
Bibaut, A., Kallus, N., Dimakopoulou, M., Chambaz, A., and van~der Laan, M. (2021).
\newblock Risk minimization from adaptively collected data: Guarantees for supervised and policy learning.
\newblock {\em Advances in Neural Information Processing Systems}, 34:19261--19273.

\bibitem[Bottou et~al., 2013]{bottou2013counterfactual}
Bottou, L., Peters, J., Qui{\~n}onero-Candela, J., Charles, D.~X., Chickering, D.~M., Portugaly, E., Ray, D., Simard, P., and Snelson, E. (2013).
\newblock Counterfactual reasoning and learning systems: The example of computational advertising.
\newblock {\em Journal of Machine Learning Research}, 14(11).

\bibitem[Branson et~al., 2023]{branson2023causal}
Branson, Z., Kennedy, E.~H., Balakrishnan, S., and Wasserman, L. (2023).
\newblock Causal effect estimation after propensity score trimming with continuous treatments.
\newblock {\em arXiv preprint arXiv:2309.00706}.

\bibitem[Buckman et~al., 2020]{buckman2020importance}
Buckman, J., Gelada, C., and Bellemare, M.~G. (2020).
\newblock The importance of pessimism in fixed-dataset policy optimization.
\newblock {\em arXiv preprint arXiv:2009.06799}.

\bibitem[Chakraborty and Murphy, 2014]{chakraborty2014dynamic}
Chakraborty, B. and Murphy, S.~A. (2014).
\newblock Dynamic treatment regimes.
\newblock {\em Annual review of statistics and its application}, 1(1):447--464.

\bibitem[Chen and Jiang, 2022]{chen2022offline}
Chen, J. and Jiang, N. (2022).
\newblock Offline reinforcement learning under value and density-ratio realizability: the power of gaps.
\newblock {\em arXiv preprint arXiv:2203.13935}.

\bibitem[Chen et~al., 2023]{chen2023steel}
Chen, X., Qi, Z., and Wan, R. (2023).
\newblock Steel: Singularity-aware reinforcement learning.
\newblock {\em arXiv preprint arXiv:2301.13152}.

\bibitem[Chernozhukov et~al., 2018]{chernozhukov2018double}
Chernozhukov, V., Chetverikov, D., Demirer, M., Duflo, E., Hansen, C., Newey, W., and Robins, J. (2018).
\newblock Double/debiased machine learning for treatment and structural parameters.
\newblock {\em The Econometrics Journal}, 21(1):C1--C68.

\bibitem[Collins et~al., 2007]{collins2007multiphase}
Collins, L.~M., Murphy, S.~A., and Strecher, V. (2007).
\newblock The multiphase optimization strategy (most) and the sequential multiple assignment randomized trial (smart): new methods for more potent ehealth interventions.
\newblock {\em American journal of preventive medicine}, 32(5):S112--S118.

\bibitem[Daniely et~al., 2011]{daniely2011multiclass}
Daniely, A., Sabato, S., Ben-David, S., and Shalev-Shwartz, S. (2011).
\newblock Multiclass learnability and the {ERM} principle.
\newblock In {\em Proceedings of the 24th Annual Conference on Learning Theory}, pages 207--232. JMLR Workshop and Conference Proceedings.

\bibitem[De~la Pena and Gin{\'e}, 2012]{de2012decoupling}
De~la Pena, V. and Gin{\'e}, E. (2012).
\newblock {\em Decoupling: from dependence to independence}.
\newblock Springer Science \& Business Media.

\bibitem[De~la Pena et~al., 2004]{de2004self}
De~la Pena, V.~H., Klass, M.~J., and Lai, T.~L. (2004).
\newblock Self-normalized processes: exponential inequalities, moment bounds and iterated logarithm laws.
\newblock {\em Annals of probability}, pages 1902--1933.

\bibitem[Dimakopoulou et~al., 2017]{dimakopoulou2017estimation}
Dimakopoulou, M., Zhou, Z., Athey, S., and Imbens, G. (2017).
\newblock Estimation considerations in contextual bandits.
\newblock {\em arXiv preprint arXiv:1711.07077}.

\bibitem[Donoho and Johnstone, 1994]{donoho1994ideal}
Donoho, D.~L. and Johnstone, J.~M. (1994).
\newblock Ideal spatial adaptation by wavelet shrinkage.
\newblock {\em {B}iometrika}, 81(3):425--455.

\bibitem[Duchi and Namkoong, 2016]{duchi2016variance}
Duchi, J. and Namkoong, H. (2016).
\newblock Variance-based regularization with convex objectives.
\newblock {\em arXiv preprint arXiv:1610.02581}.

\bibitem[Fan and Li, 2001]{fan2001variable}
Fan, J. and Li, R. (2001).
\newblock Variable selection via nonconcave penalized likelihood and its oracle properties.
\newblock {\em Journal of the American statistical Association}, 96(456):1348--1360.

\bibitem[Farias and Li, 2019]{farias2019learning}
Farias, V.~F. and Li, A.~A. (2019).
\newblock Learning preferences with side information.
\newblock {\em Management Science}, 65(7):3131--3149.

\bibitem[Freedman, 1975]{freedman1975tail}
Freedman, D.~A. (1975).
\newblock On tail probabilities for martingales.
\newblock {\em the Annals of Probability}, pages 100--118.

\bibitem[Gui et~al., 2022]{gui2022conformalized}
Gui, Y., Hore, R., Ren, Z., and Barber, R.~F. (2022).
\newblock Conformalized survival analysis with adaptive cutoffs.
\newblock {\em arXiv preprint arXiv:2211.01227}.

\bibitem[Hadad et~al., 2021]{hadad2021confidence}
Hadad, V., Hirshberg, D.~A., Zhan, R., Wager, S., and Athey, S. (2021).
\newblock Confidence intervals for policy evaluation in adaptive experiments.
\newblock {\em Proceedings of the National Academy of Sciences}, 118(15):e2014602118.

\bibitem[Hastie et~al., 2009]{hastie2009elements}
Hastie, T., Tibshirani, R., Friedman, J.~H., and Friedman, J.~H. (2009).
\newblock {\em The elements of statistical learning: data mining, inference, and prediction}, volume~2.
\newblock Springer.

\bibitem[Hirano and Porter, 2009]{hirano2009asymptotics}
Hirano, K. and Porter, J.~R. (2009).
\newblock Asymptotics for statistical treatment rules.
\newblock {\em Econometrica}, 77(5):1683--1701.

\bibitem[Hoeffding, 1994]{hoeffding1994probability}
Hoeffding, W. (1994).
\newblock Probability inequalities for sums of bounded random variables.
\newblock In {\em The collected works of Wassily Hoeffding}, pages 409--426. Springer.

\bibitem[Imbens and Rubin, 2015]{imbens2015causal}
Imbens, G.~W. and Rubin, D.~B. (2015).
\newblock {\em Causal inference in statistics, social, and biomedical sciences}.
\newblock Cambridge University Press.

\bibitem[Jin, 2023]{jin2023upper}
Jin, Y. (2023).
\newblock Upper bounds on the natarajan dimensions of some function classes.
\newblock In {\em 2023 IEEE International Symposium on Information Theory (ISIT)}, pages 1020--1025. IEEE.

\bibitem[Jin et~al., 2023]{jin2023sensitivity}
Jin, Y., Ren, Z., and Cand{\`e}s, E.~J. (2023).
\newblock Sensitivity analysis of individual treatment effects: A robust conformal inference approach.
\newblock {\em Proceedings of the National Academy of Sciences}, 120(6):e2214889120.

\bibitem[Jin et~al., 2021]{jin2021pessimism}
Jin, Y., Yang, Z., and Wang, Z. (2021).
\newblock Is pessimism provably efficient for offline rl?
\newblock In {\em International Conference on Machine Learning}, pages 5084--5096. PMLR.

\bibitem[Kallus, 2018]{kallus2018balanced}
Kallus, N. (2018).
\newblock Balanced policy evaluation and learning.
\newblock {\em Advances in neural information processing systems}, 31.

\bibitem[Kennedy, 2019]{kennedy2019nonparametric}
Kennedy, E.~H. (2019).
\newblock Nonparametric causal effects based on incremental propensity score interventions.
\newblock {\em Journal of the American Statistical Association}, 114(526):645--656.

\bibitem[Khan et~al., 2023]{khan2023off}
Khan, S., Saveski, M., and Ugander, J. (2023).
\newblock Off-policy evaluation beyond overlap: partial identification through smoothness.
\newblock {\em arXiv preprint arXiv:2305.11812}.

\bibitem[Kim et~al., 2011]{kim2011battle}
Kim, E.~S., Herbst, R.~S., Wistuba, I.~I., Lee, J.~J., Blumenschein, G.~R., Tsao, A., Stewart, D.~J., Hicks, M.~E., Erasmus, J., Gupta, S., et~al. (2011).
\newblock The battle trial: Personalizing therapy for lung cancerthe battle trial: Personalizing therapy for lung cancer.
\newblock {\em Cancer discovery}, 1(1):44--53.

\bibitem[Kitagawa and Tetenov, 2018]{kitagawa2018should}
Kitagawa, T. and Tetenov, A. (2018).
\newblock Who should be treated? empirical welfare maximization methods for treatment choice.
\newblock {\em Econometrica}, 86(2):591--616.

\bibitem[Kohavi et~al., 2020]{kohavi2020trustworthy}
Kohavi, R., Tang, D., and Xu, Y. (2020).
\newblock {\em Trustworthy online controlled experiments: A practical guide to a/b testing}.
\newblock Cambridge University Press.

\bibitem[Koltchinskii, 2006]{koltchinskii2006local}
Koltchinskii, V. (2006).
\newblock Local rademacher complexities and oracle inequalities in risk minimization.
\newblock {\em The Annals of Statistics}, 34(6):2593--2656.

\bibitem[Kuzborskij et~al., 2021]{kuzborskij2021confident}
Kuzborskij, I., Vernade, C., Gyorgy, A., and Szepesv{\'a}ri, C. (2021).
\newblock Confident off-policy evaluation and selection through self-normalized importance weighting.
\newblock In {\em International Conference on Artificial Intelligence and Statistics}, pages 640--648. PMLR.

\bibitem[Lai et~al., 1985]{lai1985asymptotically}
Lai, T.~L., Robbins, H., et~al. (1985).
\newblock Asymptotically efficient adaptive allocation rules.
\newblock {\em Advances in applied mathematics}, 6(1):4--22.

\bibitem[Lee et~al., 2021]{lee2021optidice}
Lee, J., Jeon, W., Lee, B., Pineau, J., and Kim, K.-E. (2021).
\newblock Optidice: Offline policy optimization via stationary distribution correction estimation.
\newblock In {\em International Conference on Machine Learning}, pages 6120--6130. PMLR.

\bibitem[Li et~al., 2010]{li2010contextual}
Li, L., Chu, W., Langford, J., and Schapire, R.~E. (2010).
\newblock A contextual-bandit approach to personalized news article recommendation.
\newblock In {\em Proceedings of the 19th international conference on World wide web}, pages 661--670.

\bibitem[Li et~al., 2011]{li2011unbiased}
Li, L., Chu, W., Langford, J., and Wang, X. (2011).
\newblock Unbiased offline evaluation of contextual-bandit-based news article recommendation algorithms.
\newblock In {\em Proceedings of the fourth ACM international conference on Web search and data mining}, pages 297--306.

\bibitem[Liu et~al., 2018]{liu2018breaking}
Liu, Q., Li, L., Tang, Z., and Zhou, D. (2018).
\newblock Breaking the curse of horizon: Infinite-horizon off-policy estimation.
\newblock {\em Advances in Neural Information Processing Systems}, 31.

\bibitem[Liu et~al., 2023]{liu2023average}
Liu, Y., Li, H., Zhou, Y., and Matsouaka, R. (2023).
\newblock Average treatment effect on the treated, under lack of positivity.
\newblock {\em arXiv preprint arXiv:2309.01334}.

\bibitem[Manski, 2004]{manski2004statistical}
Manski, C.~F. (2004).
\newblock Statistical treatment rules for heterogeneous populations.
\newblock {\em Econometrica}, 72(4):1221--1246.

\bibitem[Maurer and Pontil, 2009]{maurer2009empirical}
Maurer, A. and Pontil, M. (2009).
\newblock Empirical bernstein bounds and sample variance penalization.
\newblock {\em arXiv preprint arXiv:0907.3740}.

\bibitem[Mou et~al., 2023]{mou2023kernel}
Mou, W., Ding, P., Wainwright, M.~J., and Bartlett, P.~L. (2023).
\newblock Kernel-based off-policy estimation without overlap: Instance optimality beyond semiparametric efficiency.
\newblock {\em arXiv preprint arXiv:2301.06240}.

\bibitem[Murphy, 2003]{murphy2003optimal}
Murphy, S.~A. (2003).
\newblock Optimal dynamic treatment regimes.
\newblock {\em Journal of the Royal Statistical Society: Series B (Statistical Methodology)}, 65(2):331--355.

\bibitem[Murphy, 2005]{murphy2005experimental}
Murphy, S.~A. (2005).
\newblock An experimental design for the development of adaptive treatment strategies.
\newblock {\em Statistics in medicine}, 24(10):1455--1481.

\bibitem[Nachum et~al., 2019]{nachum2019algaedice}
Nachum, O., Dai, B., Kostrikov, I., Chow, Y., Li, L., and Schuurmans, D. (2019).
\newblock Algaedice: Policy gradient from arbitrary experience.
\newblock {\em arXiv preprint arXiv:1912.02074}.

\bibitem[Natarajan, 1989]{natarajan1989learning}
Natarajan, B.~K. (1989).
\newblock On learning sets and functions.
\newblock {\em Machine Learning}, 4(1):67--97.

\bibitem[Nguyen-Tang and Arora, 2023]{nguyen2023viper}
Nguyen-Tang, T. and Arora, R. (2023).
\newblock Viper: Provably efficient algorithm for offline rl with neural function approximation.
\newblock {\em arXiv preprint arXiv:2302.12780}.

\bibitem[Offer-Westort et~al., 2021]{offer2021adaptive}
Offer-Westort, M., Coppock, A., and Green, D.~P. (2021).
\newblock Adaptive experimental design: Prospects and applications in political science.
\newblock {\em American Journal of Political Science}, 65(4):826--844.

\bibitem[Pe{\~n}a et~al., 2009]{pena2009self}
Pe{\~n}a, V.~H., Lai, T.~L., and Shao, Q.-M. (2009).
\newblock {\em Self-normalized processes: Limit theory and Statistical Applications}.
\newblock Springer.

\bibitem[Rakhlin et~al., 2015]{rakhlin2015sequential}
Rakhlin, A., Sridharan, K., and Tewari, A. (2015).
\newblock Sequential complexities and uniform martingale laws of large numbers.
\newblock {\em Probability theory and related fields}, 161(1):111--153.

\bibitem[Rashidinejad et~al., 2021]{rashidinejad2021bridging}
Rashidinejad, P., Zhu, B., Ma, C., Jiao, J., and Russell, S. (2021).
\newblock Bridging offline reinforcement learning and imitation learning: A tale of pessimism.
\newblock {\em Advances in Neural Information Processing Systems}, 34:11702--11716.

\bibitem[Rashidinejad et~al., 2022]{rashidinejad2022optimal}
Rashidinejad, P., Zhu, H., Yang, K., Russell, S., and Jiao, J. (2022).
\newblock Optimal conservative offline rl with general function approximation via augmented lagrangian.
\newblock {\em arXiv preprint arXiv:2211.00716}.

\bibitem[Robins et~al., 1994]{robins1994estimation}
Robins, J.~M., Rotnitzky, A., and Zhao, L.~P. (1994).
\newblock Estimation of regression coefficients when some regressors are not always observed.
\newblock {\em Journal of the American statistical Association}, 89(427):846--866.

\bibitem[Russo et~al., 2018]{russo2018tutorial}
Russo, D.~J., Van~Roy, B., Kazerouni, A., Osband, I., Wen, Z., et~al. (2018).
\newblock A tutorial on thompson sampling.
\newblock {\em Foundations and Trends{\textregistered} in Machine Learning}, 11(1):1--96.

\bibitem[Schick, 1986]{schick1986asymptotically}
Schick, A. (1986).
\newblock On asymptotically efficient estimation in semiparametric models.
\newblock {\em The Annals of Statistics}, pages 1139--1151.

\bibitem[Schnabel et~al., 2016]{schnabel2016recommendations}
Schnabel, T., Swaminathan, A., Singh, A., Chandak, N., and Joachims, T. (2016).
\newblock Recommendations as treatments: Debiasing learning and evaluation.
\newblock In {\em international conference on machine learning}, pages 1670--1679. PMLR.

\bibitem[Shi et~al., 2022]{shi2022pessimistic}
Shi, L., Li, G., Wei, Y., Chen, Y., and Chi, Y. (2022).
\newblock Pessimistic q-learning for offline reinforcement learning: Towards optimal sample complexity.
\newblock {\em arXiv preprint arXiv:2202.13890}.

\bibitem[Simon, 1977]{simon1977adaptive}
Simon, R. (1977).
\newblock Adaptive treatment assignment methods and clinical trials.
\newblock {\em Biometrics}, pages 743--749.

\bibitem[Stone, 1974]{stone1974cross}
Stone, M. (1974).
\newblock Cross-validatory choice and assessment of statistical predictions.
\newblock {\em Journal of the royal statistical society: Series B (Methodological)}, 36(2):111--133.

\bibitem[Stoye, 2009]{stoye2009minimax}
Stoye, J. (2009).
\newblock Minimax regret treatment choice with finite samples.
\newblock {\em Journal of Econometrics}, 151(1):70--81.

\bibitem[Stoye, 2012]{stoye2012minimax}
Stoye, J. (2012).
\newblock Minimax regret treatment choice with covariates or with limited validity of experiments.
\newblock {\em Journal of Econometrics}, 166(1):138--156.

\bibitem[Sverdrup et~al., 2020]{sverdrup2020policytree}
Sverdrup, E., Kanodia, A., Zhou, Z., Athey, S., and Wager, S. (2020).
\newblock policytree: Policy learning via doubly robust empirical welfare maximization over trees.
\newblock {\em Journal of Open Source Software}, 5(50):2232.

\bibitem[Swaminathan and Joachims, 2015]{swaminathan2015batch}
Swaminathan, A. and Joachims, T. (2015).
\newblock Batch learning from logged bandit feedback through counterfactual risk minimization.
\newblock {\em The Journal of Machine Learning Research}, 16(1):1731--1755.

\bibitem[Swaminathan et~al., 2017]{swaminathan2017off}
Swaminathan, A., Krishnamurthy, A., Agarwal, A., Dudik, M., Langford, J., Jose, D., and Zitouni, I. (2017).
\newblock Off-policy evaluation for slate recommendation.
\newblock {\em Advances in Neural Information Processing Systems}, 30.

\bibitem[Thompson, 1933]{thompson1933likelihood}
Thompson, W.~R. (1933).
\newblock On the likelihood that one unknown probability exceeds another in view of the evidence of two samples.
\newblock {\em Biometrika}, 25(3-4):285--294.

\bibitem[Tsybakov, 2008]{tsybakov2004introduction}
Tsybakov, A.~B. (2008).
\newblock {\em Introduction to nonparametric estimation}.
\newblock Springer Science \& Business Media.

\bibitem[Uehara et~al., 2020]{uehara2020minimax}
Uehara, M., Huang, J., and Jiang, N. (2020).
\newblock Minimax weight and q-function learning for off-policy evaluation.
\newblock In {\em International Conference on Machine Learning}, pages 9659--9668. PMLR.

\bibitem[Uehara and Sun, 2021]{uehara2021pessimistic}
Uehara, M. and Sun, W. (2021).
\newblock Pessimistic model-based offline reinforcement learning under partial coverage.
\newblock {\em arXiv preprint arXiv:2107.06226}.

\bibitem[Van Den~Burg and Groenen, 2016]{van2016gensvm}
Van Den~Burg, G.~J. and Groenen, P.~J. (2016).
\newblock Gensvm: A generalized multiclass support vector machine.
\newblock {\em Journal of Machine Learning Research}, 17(224):1--42.

\bibitem[Vanschoren et~al., 2013]{OpenML2013}
Vanschoren, J., van Rijn, J.~N., Bischl, B., and Torgo, L. (2013).
\newblock Openml: networked science in machine learning.
\newblock {\em SIGKDD Explorations}, 15(2):49--60.

\bibitem[Vapnik and Chervonenkis, 2015]{vapnik2015uniform}
Vapnik, V.~N. and Chervonenkis, A.~Y. (2015).
\newblock On the uniform convergence of relative frequencies of events to their probabilities.
\newblock In {\em Measures of complexity}, pages 11--30. Springer.

\bibitem[Waudby-Smith and Ramdas, 2020]{waudby2020estimating}
Waudby-Smith, I. and Ramdas, A. (2020).
\newblock Estimating means of bounded random variables by betting.
\newblock {\em arXiv preprint arXiv:2010.09686}.

\bibitem[Waudby-Smith et~al., 2022]{waudby2022anytime}
Waudby-Smith, I., Wu, L., Ramdas, A., Karampatziakis, N., and Mineiro, P. (2022).
\newblock Anytime-valid off-policy inference for contextual bandits.
\newblock {\em arXiv preprint arXiv:2210.10768}.

\bibitem[Xie et~al., 2021a]{xie2021bellman}
Xie, T., Cheng, C.-A., Jiang, N., Mineiro, P., and Agarwal, A. (2021a).
\newblock Bellman-consistent pessimism for offline reinforcement learning.
\newblock {\em Advances in neural information processing systems}, 34:6683--6694.

\bibitem[Xie et~al., 2021b]{xie2021policy}
Xie, T., Jiang, N., Wang, H., Xiong, C., and Bai, Y. (2021b).
\newblock Policy finetuning: Bridging sample-efficient offline and online reinforcement learning.
\newblock {\em Advances in neural information processing systems}, 34:27395--27407.

\bibitem[Xie et~al., 2019]{xie2019towards}
Xie, T., Ma, Y., and Wang, Y.-X. (2019).
\newblock Towards optimal off-policy evaluation for reinforcement learning with marginalized importance sampling.
\newblock {\em Advances in Neural Information Processing Systems}, 32.

\bibitem[Xu and Zeevi, 2020]{xu2020towards}
Xu, Y. and Zeevi, A. (2020).
\newblock Towards optimal problem dependent generalization error bounds in statistical learning theory.
\newblock {\em arXiv preprint arXiv:2011.06186}.

\bibitem[Yan et~al., 2022]{yan2022efficacy}
Yan, Y., Li, G., Chen, Y., and Fan, J. (2022).
\newblock The efficacy of pessimism in asynchronous q-learning.
\newblock {\em arXiv preprint arXiv:2203.07368}.

\bibitem[Yang and Ding, 2018]{yang2018asymptotic}
Yang, S. and Ding, P. (2018).
\newblock Asymptotic inference of causal effects with observational studies trimmed by the estimated propensity scores.
\newblock {\em Biometrika}, 105(2):487--493.

\bibitem[Yin and Wang, 2021]{yin2021towards}
Yin, M. and Wang, Y.-X. (2021).
\newblock Towards instance-optimal offline reinforcement learning with pessimism.
\newblock {\em Advances in neural information processing systems}, 34:4065--4078.

\bibitem[Zanette et~al., 2021]{zanette2021provable}
Zanette, A., Wainwright, M.~J., and Brunskill, E. (2021).
\newblock Provable benefits of actor-critic methods for offline reinforcement learning.
\newblock {\em Advances in neural information processing systems}, 34:13626--13640.

\bibitem[Zhan et~al., 2021]{zhan2021policy}
Zhan, R., Ren, Z., Athey, S., and Zhou, Z. (2021).
\newblock Policy learning with adaptively collected data.
\newblock {\em arXiv preprint arXiv:2105.02344}.

\bibitem[Zhan et~al., 2022]{zhan2022offline}
Zhan, W., Huang, B., Huang, A., Jiang, N., and Lee, J. (2022).
\newblock Offline reinforcement learning with realizability and single-policy concentrability.
\newblock In {\em Conference on Learning Theory}, pages 2730--2775. PMLR.

\bibitem[Zhang et~al., 2012]{zhang2012estimating}
Zhang, B., Tsiatis, A.~A., Davidian, M., Zhang, M., and Laber, E. (2012).
\newblock Estimating optimal treatment regimes from a classification perspective.
\newblock {\em Stat}, 1(1):103--114.

\bibitem[Zhang et~al., 2020]{zhang2020gendice}
Zhang, R., Dai, B., Li, L., and Schuurmans, D. (2020).
\newblock Gendice: Generalized offline estimation of stationary values.
\newblock {\em arXiv preprint arXiv:2002.09072}.

\bibitem[Zhao et~al., 2023]{zhao2023positivity}
Zhao, P., Chambaz, A., Josse, J., and Yang, S. (2023).
\newblock Positivity-free policy learning with observational data.
\newblock {\em arXiv preprint arXiv:2310.06969}.

\bibitem[Zhao et~al., 2012]{zhao2012estimating}
Zhao, Y., Zeng, D., Rush, A.~J., and Kosorok, M.~R. (2012).
\newblock Estimating individualized treatment rules using outcome weighted learning.
\newblock {\em Journal of the American Statistical Association}, 107(499):1106--1118.

\bibitem[Zhao et~al., 2015]{zhao2015doubly}
Zhao, Y.-Q., Zeng, D., Laber, E.~B., Song, R., Yuan, M., and Kosorok, M.~R. (2015).
\newblock Doubly robust learning for estimating individualized treatment with censored data.
\newblock {\em Biometrika}, 102(1):151--168.

\bibitem[Zhou et~al., 2017]{zhou2017residual}
Zhou, X., Mayer-Hamblett, N., Khan, U., and Kosorok, M.~R. (2017).
\newblock Residual weighted learning for estimating individualized treatment rules.
\newblock {\em Journal of the American Statistical Association}, 112(517):169--187.

\bibitem[Zhou et~al., 2022]{zhou2022offline}
Zhou, Z., Athey, S., and Wager, S. (2022).
\newblock Offline multi-action policy learning: Generalization and optimization.
\newblock {\em Operations Research}.

\bibitem[Zou, 2006]{zou2006adaptive}
Zou, H. (2006).
\newblock The adaptive lasso and its oracle properties.
\newblock {\em Journal of the American statistical association}, 101(476):1418--1429.

\end{thebibliography}

\newpage 
\appendix 

\section{Deferred details}


\subsection{Real datasets}
The names of real datasets on the OpenML platform 
used in Section 7.3 are:

\begin{quote}
\texttt{`waveform-5000', `Long', `cmc', `artificial-characters', \\
`Click\_prediction\_small', `skin-segmentation', `allrep', \\
`mfeat-morphological', `satellite\_image',\\ 
`jungle\_chess\_2pcs\_endgame\_elephant\_elephant', `wilt', `Satellite',  `ringnorm', 
`mammography', `delta\_ailerons', `PhishingWebsites', `splice', `pendigits', `texture', `cardiotocography', `volcanoes-d4', `volcanoes-b3', `dis', `optdigits', `electricity', `kr-vs-kp', `bank-marketing', `satimage', `MagicTelescope', `houses', `eeg-eye-state', `car', `segment'}
\end{quote}

\subsection{Cross-fitted pessimistic policy learning}
\label{app:subsec_crossfit}

In this section, we provide a cross-fitted version of 
our algorithm for the fixed behavior policy case. 
We first randomly split the index set $[T]$ into two disjoint folds $\cT_1$ and $\cT_2$
of size $T/2$, where we assume $T/2\in\NN$ without loss of generality. 
For $k=1,2$ we use $\{X_t,A_t,Y_t\}_{t\notin \cT_k}$ 
to obtain an estimator $\hat\mu^{(k)}$ for $\mu(x,a)=\EE[Y(a)\given X=x]$, 
and construct the cross-fitted estimator: for $t\in \cT_k$,  
\$
\hat{Q}(\pi) = \frac{1}{T}\sum_{t=1}^T \hat\Gamma_t(\pi),\quad 
\hat\Gamma_t(\pi) = \hat\mu^{(k)}(X_t,\pi(X_t)) + \frac{\ind\{A_t=\pi(X_t)\}}{e (X_t,\pi(X_t) )} \cdot\big(Y_t-\hat\mu^{(k)}(X_t,\pi(X_t))\big).
\$
We then have $\hat\mu^{(k)}$ is independent of $\{X_t,A_t,Y_t\}_{t\in \cT_k}$. 
The estimation error of $\hat{Q}(\pi)$ can be bounded as 
\$
\big| \hat{Q}(\pi) - Q(\pi)\big| \leq \frac{|\hat{Q}^{(1)}(\pi) -Q(\pi)| }{2} + \frac{|\hat{Q}^{(2)}(\pi) -Q(\pi) | }{2},
\$
where each $|\hat{Q}^{(k)}(\pi)-Q(\pi)$ reduces to the case 
in Section~\ref{sec:fix} with sample size $|\cT_k|=T/2$. 
Using similar ideas as those in Section~\ref{sec:fix}, we construct 
the regularization term 
$R(\pi) = (R^{(1)}(\pi)+R^{(2)}(\pi))/2$, where $R^{(k)}(\pi) = \beta \cdot V^{(k)}(\pi)$, and 
\$
V^{(k)}(\pi) \defn  
\max\big\{\vs^{(k)}(\pi), \vp^{(k)}(\pi), \vh^{(k)}(\pi)\big\} 
\$
for  
\$
 & \vs^{(k)}(\pi) = \frac{2}{T}  
\bigg( \sum_{t\in \cT_k} \frac{\ind \{A_t = \pi(X_t)\}}{e(X_t,\pi(X_t) )^2}\bigg)^{1/2}, \notag  \\
 & \vp^{(k)}(\pi) = \frac{2}{T}
\bigg( \sum_{t\in \cT_k} \frac{1}{e(X_t,\pi(X_t) )}\bigg)^{1/2},  \quad \vh^{(k)}(\pi) = \frac{2}{T}
\bigg( \sum_{t\in \cT_k} \frac{1}{e(X_t,\pi(X_t) )^3}\bigg)^{1/4}. \notag
\$
The uniform concentration results can be similarly obtained 
as in Theorem~\ref{thm:fix_upper} by a union bound over $k=1,2$ and taking 
\$
\beta\geq  10\sqrt{2(\ndim(\Pi) \log(TK^2/2) + \log(4/\delta))}.
\$
We omit the details for brevity. 

\subsection{Deferred numerical experiments}

\subsubsection{Comparison with clipped method in MAB settings}
\label{eq:subsubsec_MAB_clip}

Figure~\ref{fig:mab_clip} compares PPL with the method of \cite{swaminathan2015batch} for i.i.d.~data in MAB, where the inverse propensities are clipped before applying variance-based penalized welfare maximization. We see that clipping introduces a bias which deteriorates the performance (suboptimality). 

\begin{figure}[H]
  \centering 
  \includegraphics[width=0.8\linewidth]{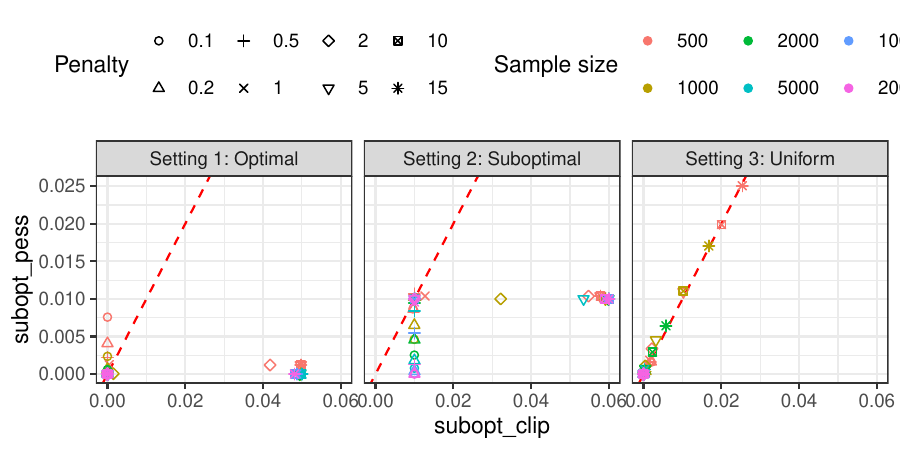}
  \caption{Average rescaled suboptimality of the learned policy  from PPL ($y$-axis) and the method of \cite{swaminathan2015batch} ($x$-axis) with various penalty parameter $\beta$ at various sample sizes in the three settings of Section~\ref{subsubsec:mab} in the main text.}
  \label{fig:mab_clip}
\end{figure}

\subsubsection{Selection frequency in MAB experiments}
\label{subsec:map_freq}

Figure~\ref{fig:mab_freq} plots the frequency of selecting each arm using the learned policy from both GPL and PPL in the simulations in Section~\ref{subsubsec:mab} of the main text.

\begin{figure}[h]
  \centering 
  \includegraphics[width=\linewidth]{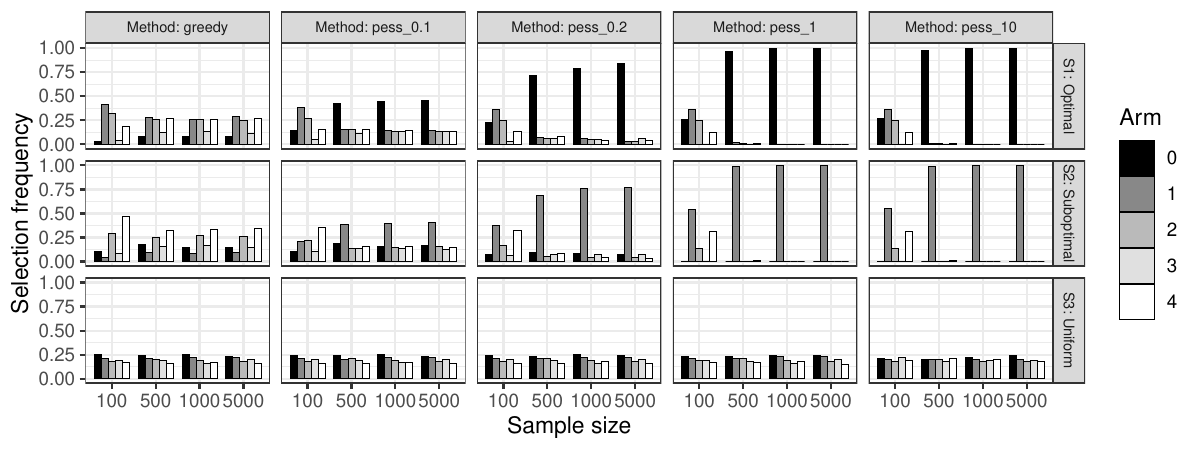}
  \caption{Frequency of selecting each arm using the learned policy from GPL and PPL with $\beta\in \{0.1, 0.2, 1, 10\}$, averaged over $N=1000$ independent runs. Arms are colored from darkest to lightest according to highest to lowest rewards. Each row corresponds to one setting, 
  and each column is a method.}
  \label{fig:mab_freq}
\end{figure}

\subsection{Deferred discussion on optimization}
\label{app:subsec_opt_discuss}

In the following, we compare the optimization aspects of PPL with several genres of related methods:
\begin{itemize}
    \item \textbf{Learning over a policy class.} When using IPW estimator to learn over a policy class, state-of-the-art methods that come with optimization schemes~\citep{athey2021policy,zhou2022offline,zhan2021policy} all rely on tree-based search algorithms to learn policies based on flexible decision trees or random forests. Notably, they need an exhaustive search over all possible trees of a fixed depth, with is the same as ours. To facilitate efficient computation, they also need to rely on approximations without clear guarantees (though these methods tend to work well in practice). As such, our strategy is comparable to them. 
    
    \item \textbf{Penalized ERM-based policy learning.} \cite{swaminathan2015batch} uses a similar variance-penalized ERM (with clipped propensities) for policy learning. Our optimization algorithm is indeed inspired by their strategies, where they also need to use approximate optimization that can only coverge to a local minima. 
    
    \item \textbf{Other structure-based pessimism approach.} Other pessimism-based algorithms that rely on structured reward functions are also either purely theoretical exploration without a practical implementation~\citep{xie2021bellman,yin2021towards}, 
    or admit a closed-form solution that is similarly nonconvex (such as with linear function classes~\citep{jin2021pessimism}). 
    For the latter, the optimization aspects of the nonconvex problem have also inspired a line of research that aims to approximately reduce (but not fully address) the computational issue via e.g., random perturbations~\citep{nguyen2023viper}. 
    As such, we view the optimization issue as pertinent to pessimism-based methods whose exact solution goes way beyond the scope of the current work. 
\end{itemize}

\subsection{Deferred discussion on unknown propensity scores}
\label{app:subsec_unknown_propensity}

\begin{remark}[Extension to unknown propensities, continued]\normalfont
  We have assumed knowledge of the true propensities $e_t(\cdot,\cdot)$. 
         For batched data, this is a standard setting considered in the literature~\citep{kitagawa2018should,zhao2012estimating}. 
         For future work, it would be interesting to extend PPL to  
         observational studies where the fixed behavior policy $e(\cdot,\cdot)$ has 
         to be estimated from data~\citep{athey2021policy}. 
         We expect this to be a highly nontrivial task, especially 
         when it is estimated via nonparametric regression or machine learning algorithms 
         that do not rely on strong modeling assumptions~\citep{chernozhukov2018double}. 
         In specific, standard analysis of AIPW estimators with estimated 
         propensity scores $\hat{e}(\cdot,\cdot)$ heavily relies on the uniform overlap condition, 
         under which the impact of estimation error (which is often characterized 
         by the $L_2$-norm or sup-norm of $\hat{e}(\cdot,\cdot)-e(\cdot,\cdot)$)
         becomes negligible in  inverse-propensity weighting. 
         However, without uniform overlap, it is unclear how to  (i) characterize 
         the estimation error (since actions with very small propensity rarely appear), 
         (ii) characterize the impact of estimation error on the AIPW estimator, 
         and (iii) translate them into a data-dependent, policy-dependent, 
         and finite-sample valid upper bound $R(\pi)$. 
         Even if (i) and (ii) are plausible to be  addressed 
         with assumptions on the model of $\hat{e}(\cdot,\cdot)$, 
         the last step (iii) remains challenging. 
         We thus leave this extension for future work. 
         Finally, we  note that knowledge of $e_t(\cdot,\cdot)$ is necessary for learning from 
         adaptively collected data, which we study in Section~\ref{sec:adaptive}. Estimating $e_t(\cdot,\cdot)$ 
         is infeasible for adaptive data  
 with one trajectory without strong assumptions~\citep{zhan2021policy}, 
 and the issues for uncertainty quantification with estimated propensities 
 we discuss above still persist.
\end{remark}

\section{Proof sketch of main results}


This section provides proof sketches for the main results.
Section~\ref{subsec:sketch_iid} focuses on the upper bound
in the fixed policy case and Section~\ref{subsec:proof_upper_adapt}
that in the adaptive policy case. Throughout, we omit 
the dependence on $\cC$ to simplify the notation
when no confusion can arise given the context.

\subsection{Proof of  upper bound for the fixed policy case}
\label{subsec:sketch_iid}
\begin{proof}[Proof of Theorem 4.1]
We start by showing that the (rescaled) estimation error  
$\frac{\hat{Q}(\pi) - Q(\pi)}{V(\pi)}$ is uniformly 
bounded over $\pi\in \Pi$. Recall that $\mu(x,a) 
\defn \EE\big[Y_t(a) \given X_t =x ]$
for any $(x,a) \in \cX \times \cA$, 
and we decompose the error as follows.
\$
\bigg|\frac{\hat{Q}(\pi) - Q(\pi)}{V(\pi)}\bigg|
& = \bigg|\frac{1}{T} \sum^T_{t=1} \frac{\hat \Gamma_t(\pi) - Q(\pi) }{V(\pi)} \bigg|\\
& \le \underbrace{\bigg|\frac{1}{T}\sum^T_{t=1}\frac{ \hat{\Gamma}_t(\pi) - \mu(X_t,\pi(X_t))}{V(\pi)} \bigg|}_{\displaystyle \rm term (i)}
+ \underbrace{\bigg|\frac{1}{T}\sum^T_{t=1} \frac{ \mu(X_t,\pi(X_t)) - Q(\pi)}{V(\pi) }\bigg|}_{\displaystyle\rm term (ii)}.
\$ 
We show in steps I and II how to bound term (i),
and establish the upper bound for term (ii) in step III. 

\noindent\textit{Step I: Symmetrization via a Rademacher process.}
For each $t\in [T]$, we let $(A_t',Y_t')$ be an independent copy
of $(A_t,Y_t)$ conditional on $X_t$, i.e.,
\$
(A_t',Y_t') \indep (A_t,Y_t) \given X_t 
\quad \mbox{and}\quad
(A_t',Y_t')\given X_t \stackrel{\rm d}{=} (A_t,Y_t) \given X_t. 
\$
Based on these independent copies, we define
\$
& \hat \Gamma'_t(\pi) = \hat{\mu}(X_t,\pi(X_t))
+ \frac{\ind\{A_t' = \pi(X_t)\}}{e(X_t,\pi(X_t))}
\cdot \big(Y_t' - \hat{\mu}(X_t,\pi(X_t))\big),\\ 
\mbox{ and }~  
& V_{\rm s}'(\pi) = \frac{1}{T} \bigg(\sum^T_{t=1} \frac{\ind\{A_t' = \pi(X_t)\}}
{e(X_t,\pi(X_t))}\bigg)^{1/2}.
\$
We are to use the exchangeability between $\{A_t,Y_t\}_{t=1}^T$ 
and $\{A_t',Y_t'\}_{t=1}^T$ conditional on $\{X_t\}_{t=1}^T$ 
to turn the uniform concentration to tail bounds 
on a Rademacher process. 

To this end, we let $x = (x_1,\ldots,x_T)$,
$a = (a_1,\ldots,a_T)$,
$a' = (a'_1,\ldots,a'_T)$,
$y = (y_1,\ldots,y_T)$ 
and $y' = (y_1',\ldots,y_T')$ 
denote any realized values
of $\{X_t\}_{t=1}^T$,
$\{A_t\}_{t=1}^T$,
$\{A_t'\}_{t=1}^T$,
$\{Y_t\}_{t=1}^T$,
and $\{Y_t'\}_{t=1}^T$,
respectively. We then define
\$
              & \gamma_t(\pi) = \hmu(x_t,\pi(x_t)) + 
\frac{\ind\{a_t = \pi(x_t)\}}{e(x_t,\pi(x_t))}
\cdot\{y_t - \hmu(x_t,\pi(x_t))\}\\
\mbox{ and}\quad &
\gamma_t'(\pi) = \hmu(x_t,\pi(x_t)) + 
\frac{\ind\{a_t' = \pi(x_t)\}}{e(x_t,\pi(x_t))}
\cdot\{y_t' - \hmu(x_t,\pi(x_t))\}
\$
as the corresponding realization of $\hat\Gamma_t(\pi)$ and $\hat\Gamma_t'(\pi)$. 
The following lemma bounds the probability of term (i)
exceeding a constant $\xi \ge 4$ by the tail probability 
of a Rademacher process, and its proof is in Appendix~\ref{app:proof_symmetrization}.
\begin{lemma}
\label{lemma:symmetrization}
For any constant $\xi \ge 4$, it holds that 
\$
& \PP\bigg(\sup_{\pi \in \Pi} \Big| \frac{1}{T} 
\sum^T_{t=1} \frac{\hgamma_t(\pi) - \mu(X_t,\pi(X_t))}
{V(\pi)}\Big| \ge \xi 
\bigg)\\
&  \le \sup_{x,a,a,',y,y'}
2 \PP_{\eps}\Bigg(    \bigg|\sum^T_{t=1} \eps_t \cdot (\hat{\gamma}_t(\pi) - \hat{\gamma}'_t(\pi))\bigg|  \\
& \qquad \qquad \qquad \qquad \ge \frac{\xi}{8}\cdot \bigg\{\sum^T_{t=1} \frac{(\ind\{a_t = \pi(x_t)\} + \ind\{a_t' = \pi(x_t)\})^2}
{e(x_t,\pi(x_t))^2} \bigg\}^{1/2}  ~\textnormal{for all }\pi\in\Pi\Bigg),
\$
where $\PP_{\eps}$ is the probability over 
i.i.d.~Rademacher random variables $\eps_1,\ldots,\eps_T
\stackrel{\textnormal{i.i.d.}}{\sim} \textnormal{Unif}\{\pm 1\}$ while holding other quantities fixed at their realized values.
\end{lemma}

\noindent\textit{Step II: Uniform concentration.}  
The next lemma establishes a uniform upper bound for 
the tail probability of the Rademacher process in terms of the sample size $T$ and 
the complexity of $\Pi$, whose proof is provided in Appendix~\ref{app:subsec_rademacher}.

\begin{lemma}
\label{lemma:tail_rademacher}
For any fixed $\delta \in (0,1)$ 
and any fixed realizations $x,a,a',y,y'$ 
of $\{X_t\}_{t=1}^T$,
$\{A_t\}_{t=1}^T$,
$\{A_t'\}_{t=1}^T$,
$\{Y_t\}_{t=1}^T$,
and $\{Y_t'\}_{t=1}^T$, the event
\$
\Big|\sum^T_{t=1} \eps_t \cdot (\hat{\gamma}_t(\pi) - \hat{\gamma}'_t(\pi))\Big|
& \le 2\sqrt{2\big(\ndim(\Pi) \log (TK^2) + \log(2/\delta)\big)}\\
&   \times \Big\{\sum^T_{t=1} \frac{(\ind\{a_t = \pi(x_t)\} + \ind\{a_t' = \pi(x_t)\})^2}
{e(x_t,\pi(x_t))^2} \Big\}^{1/2},
\quad \mbox{for all }\pi \in \Pi
\$ 
holds 
with probability at 
least $1-\delta$ over the randomness of $\varepsilon_1,\dots,\varepsilon_T$.
\end{lemma}

Taking $\xi=16\sqrt{2\big(\ndim(\Pi) \log (TK^2) + \log(8/\delta)\big)}$, 
combining Lemmas~\ref{lemma:symmetrization} and~\ref{lemma:tail_rademacher}, 
we know that 
\#\label{eq:fix_bd_i}
 \sup_{\pi \in \Pi} \frac{1}{T} \cdot 
\frac{|\sum^T_{t=1} \hgamma_t(\pi) - \mu(X_t,\pi(X_t))|}
{V(\pi)} \le  8\sqrt{2 \cdot \ndim(\Pi) \cdot \log (TK^2) + 2\log(8/\delta) }
\#
with probability at least $1-\delta/2$. This controls term (i). 

\noindent\textit{Step III: Controlling term (ii).}  
Next, we use similar techniques as in Step I and Step II 
to develop a uniform bound
for term (ii). 
The proof of the next lemma is deferred to Appendix~\ref{app:subsec_bounded_part}.
\begin{lemma}
\label{lemma:bounded_part}
For any $\delta \in (0,1)$, with probability at least $1-\delta$, there is
\$
\sup_{\pi \in \Pi} \, \frac{|\sum^T_{t=1} Q(X_t,\pi) - Q(\pi)|}{V(\pi)}
\le 2\sqrt{2 \cdot \ndim(\Pi) \cdot \log (TK^2) + 2\log(8/\delta) }.
\$
\end{lemma}
Finally, applying a union bound to~\eqref{eq:fix_bd_i} and Lemma~\ref{lemma:bounded_part}, 
we have with probability at least $1-\delta$ that 
\$
\bigg|\frac{\hat{Q}(\pi) - Q(\pi)}{V(\pi)}\bigg| \leq 10\sqrt{2 \cdot \ndim(\Pi) \cdot \log (TK^2) + 2\log(16/\delta) } \leq \beta,
\$
which concludes the proof of the first claim. 
Combining this result with Lemma 3.1 by taking $R(\pi)=\beta\cdot V(\pi)$, 
we conclude 
the proof of Theorem 4.1. 
\end{proof}

\subsection{Proof of upper bound for the adaptive policy case}
\label{subsec:proof_upper_adapt}

\begin{proof}[Proof of Theorem~\ref{thm:adapt_upper}]
As before, we begin with the triangle inequality 
\$
&\sup_{\pi\in\Pi}\bigg|\frac{\hat{Q}(\pi) - Q(\pi)}{V(\pi)}\bigg| \\
&\le \underbrace{\sup_{\pi\in\Pi} \bigg|\frac{1}{T}\sum^T_{t=1}\frac{ \hat{\Gamma}_t(\pi) - \mu(X_t,\pi(X_t))}{V(\pi)} \bigg|}_{\displaystyle \rm term (i)}
+ \underbrace{\sup_{\pi\in\Pi} \bigg|\frac{1}{T}\sum^T_{t=1} \frac{ \mu(X_t,\pi(X_t)) - Q(\pi)}{V(\pi) }\bigg|}_{\displaystyle \rm term (ii)},
\$ 
for which we will control terms (i) and (ii) separately. 
With adaptive behavior policy, bounding term (ii) is similar to the i.i.d.~case
while controlling term (i) is much more challenging. 
The general tool for controlling term (i) is still symmetrization. 
Due to the adaptivity of the data-collecting processes, however, 
we are to use a decoupled tangent sequence~\citep{de2012decoupling,rakhlin2015sequential} 
which possesses a ``weaker'' symmetry property than the 
symmetric copies generated in the i.i.d.~case. 
The decoupled tangent sequence allows us to 
turn the tail bound on term (i) to that of 
a tree-Rademacher process~\citep{rakhlin2015sequential}, whose definition will be provided shortly; 
from there, we shall utilize a conditional symmetry property of the 
tree-Rademacher process to invoke a self-normalized concentration inequality 
to control term (i). 

\noindent\textit{Step I: Symmetrization.} 
For every $t\in[T]$, 
we let $(A_t',Y_t')$ be an independent copy of $(A_t,Y_t)$ 
conditional on $\{X_t\}\cup\{X_i,A_i,Y_i\}_{i=1}^{t-1}$,
and $\{A_t',Y_t'\}_{t=1}^T$  
are conditionally independent of each other 
given $\{X_t,A_t,Y_t\}_{t=1}^T$. 
This can be achieved as follows. 
For $t=1,2,\dots,T$, we sequentially sample 
$A_t'\sim e_t(X_t,\pi(X_t)\given \cH_{t})$ 
and $Y_t'\sim \PP_{Y(a)\given X=x}$ for $a=A_t'$ and $x=X_t$, 
all independently of everything else. 
We then define 
\$
\hat\Gamma_t'(\pi) =  \hat\mu_t(X_t,\pi(X_t)) + \frac{\ind\{A_t'=\pi(X_t)\}}{e_t(X_t,\pi(X_t)\given \cH_t)} \cdot\big(Y_t'-\hat\mu_t(X_t,\pi(X_t))\big).
\$
For notational simplicity,  
for  any $t \in [T]$, we also define 
\$
\hat{U}_t(\pi) = 
\frac{(\ind \{A_t = \pi(X_t)\}+\ind\{A_t'=\pi(X_t)\})^2}
{e_t(X_t,\pi(X_t)\given \cH_t)^2}.
\$
The following lemma 
moves the supremum over $\pi\in\Pi$ out of 
the probability, and turns the tail bound of term (i) into that of 
a process symmetric in $(A_t,Y_t)$ and $(A_t',Y_t')$.
The proof is in Appendix~\ref{app:lem_adapt_1}. 

\begin{lemma}\label{lem:adapt_1}
Let
$
\Pi(X) = \big\{ (\pi(X_1),\dots,\pi(X_T))\in \cA^{T}\colon \pi\in \Pi   \big\}
$
be the set of all possible actions on $\{X_t\}_{t=1}^T$ from $\pi\in\Pi$.
For any $\xi \ge 4$ that may depend on $\{X_t\}_{t=1}^T$, it holds that 
\#\label{eq:lem_adapt_1}
&\PP\bigg(  \sup_{\pi \in \Pi} \frac{1}{T} \cdot
\frac{|\sum^T_{t=1}\hgamma_t(\pi) - \mu(X_t,\pi(X_t))|}
{V(\pi) \sqrt{2\log(30T\cdot V(\pi))}}\ge \xi \bigggiven \{X_t\}_{t=1}^T \bigg) \notag \\ 
&\leq 2\big|\Pi(X)\big| \cdot \sup_{\pi\in\Pi}~ \PP\Bigg(  \frac{\big|\sum^T_{t=1} \{ \hgamma'_t(\pi  ) - \hgamma_t(\pi ) \} \big|}{ 
  \sqrt{ \sum_{t=1}^T \hat{U}_t(\pi ) } \cdot\sqrt{\log\big(56\sum_{t=1}^T \hat{U}_t(\pi )\big)} }
\geq \frac{\xi}{8} \Bigggiven \{X_t\}_{t=1}^T  \Bigg).
\#
\end{lemma}

It then suffices to control the probability in~\eqref{eq:lem_adapt_1}  
for any fixed $\pi\in\Pi$. 
In the following, we are to 
turn 
the upper bound in Lemma~\ref{lem:adapt_1} 
into the tail probability of a tree Rademacher process~\citep{rakhlin2015sequential}. 

\noindent\textit{Step II: Tree Rademacher process.}
We define  a superset for all possible realizations $z_t = (a,a',y,y',f,h)$ for any $t\in[T]$ as
\#\label{eq:def_cZ}
\cZ = \cA\times\cA\times[0,1]\times[0,1]\times \cF_e \times\cF_\mu,
\#
where $\cF_e$ is the set of all 
possible sampling rule, i.e., all mappings of the form $f(\cdot,\cdot)\colon \cX\times\cA\to [0,1]$ 
such that $\sum_{a\in \cA}f(x,a)=1$ for all $x\in\cX$ and $\log f(x,a) \ge -(\log T)^\alpha$
for all $(x,a) \in \cX \times \cA$; 
$\cF_\mu$ 
is the set of all fitted value functions, i.e., 
all mappings of the form $h(\cdot,\cdot)\colon \cX\times\cA\to [0,1]$. 
Note that $\cZ$ is a superset of 
all possible realization of $(A_t,A_t',Y_t,Y_t',e_t(\cdot,\cdot ), 
\hat\mu_t(\cdot,\cdot))$ because in $\cZ$ we allow the 
sampling rule and fitted mean functions to be arbitrary. 
We further let $x_t$ denote any realization of $X_t$, and
then for any $x_t$ and $z_t= (a_t,a_t',y_t,y_t',f_t,h_t)\in \cZ_t$, we define the corresponding 
realized value of $\hgamma_t(\pi)$ and $\hgamma_t'(\pi)$ as 
\$
\hat\gamma_t(\pi;x_t,z_t) &= h_t(x_t,\pi(x_t) ) + \frac{\ind\{a_t=\pi(x_t)\}}{f_t(x_t,\pi(x_t) )} \cdot\big(y_t-h_t(x_t,\pi(x_t) )\big), \\
\hat\gamma_t'(\pi;x_t,z_t) &=h_t(x_t,\pi(x_t) ) + \frac{\ind\{a_t'=\pi(x_t)\}}{f_t(x_t,\pi(x_t) )} \cdot\big(y_t'-h_t(x_t,\pi(x_t) )\big),
\$ 
and the realized value of $\hat{U}_t(\pi)$ as
\$
\hat{u}_t(\pi;x_t,z_t) = \frac{(\ind \{a_t = \pi(x_t)\}+\ind\{a_t'=\pi(x_t)\})^2}{f_t(x_t,\pi(x_t) )^2}.
\$

\begin{definition}[$\cZ$-valued tree~\citep{rakhlin2015sequential}]
\label{defn:tree}
A $\cZ$-valued tree $\bz$ of depth $T$ 
is a rooted complete binary tree with 
nodes labeled by elements of $\cZ$ defined in~\eqref{eq:def_cZ}. 
We identify the tree $\bz$ with the sequence $(\bz_1,\dots,\bz_T)$ 
of labeling functions $\bz_t\colon \{\pm 1\}^{t-1}\to \cZ$ which 
provide the label for each node. Here, $\bz_1\in \cZ$ is the label of the root of
the tree, 
while $\bz_{t}(\eps_1,\dots,\eps_{t-1})$ for $t>1$ 
and $\eps_i\in \{\pm 1\}$ is the label of the node 
obtained by following the path of length $t-1$ from the root, 
with $+1$ indicating `right' and $-1$ indicating `left'. 
\end{definition}

More intuitively, each $\cZ$-valued tree of depth $T$ 
can be equivalently represented by $2^T-1$ elements in $\cZ$: 
\$
\underbrace{z_1,}_{\displaystyle 1\textnormal{ element}} 
\underbrace{z_2(-1),z_2(1),}_{\displaystyle 2\textnormal{ elements}}
~\ldots,~ \underbrace{z_T(-1,\ldots,-1),\ldots,z_T(1,\ldots,1)}_{\displaystyle 2^{T-1}\textnormal{ elements}}.
\$
See Figure~\ref{fig:tree} below for an illustration of a tree with
depth $T=3$.

\begin{figure}[H]
  \centering  
  \begin{tikzpicture}[->,>=stealth', thick, main node/.style={circle,draw}]
  
  \node[circle, text=black, circle, draw=black, fill=black!5, scale=1.7, label=right:{$z_1$}] (1) at  (0,0) { };
  \node[circle, text=black, circle, draw=black, fill=black!5, scale=1.7, label=left:{$z_2(-1)$}] (2) at  (-1.5,-1.5)  { }; 
  \node[circle, text=black, circle, draw=black, fill=black!5, scale=1.7, label=right:{ $z_2(1)$}] (3) at  (1.5,-1.5)  { }; 

  \node[circle, text=black, draw=black, fill=black!5, scale=1.7, label=below:{$z_3(-1,-1)$}] (4) at  (-3,-3) { }; 
  \node[circle, text=black, draw=black, fill=black!5, scale=1.7, label=below:{$z_3(-1,1)$}] (5) at  (-1,-3) { }; 
  \node[circle, text=black, draw=black, fill=black!5, scale=1.7, label=below:{$z_3(1,-1)$}] (6) at  (1,-3) { }; 
  \node[circle, text=black, draw=black, fill=black!5, scale=1.7, label=below:{$z_3(1,1)$}] (7) at  (3,-3) { }; 
  
  \draw[-] (1) edge [draw=black] (2);
  \draw[-] (1) edge [draw=black] (3); 
  \draw[-] (2) edge [draw=black] (4); 
  \draw[-] (2) edge [draw=black] (5); 
  \draw[-] (3) edge [draw=black] (6); 
  \draw[-] (3) edge [draw=black] (7); 
  
  \coordinate (a) at (0.2,-0.4);
  \coordinate (b) at (1.1,-1.35);
  \coordinate (c) at (1.2,-1.7);
  \coordinate (d) at (0.9,-2.65);

  \draw[->,magenta,dashed,very thick] (a) -- (b) node [] {};
  \draw[->,magenta,dashed,very thick] (c) -- (d) node [] {};
  \end{tikzpicture}  
  \caption{A $\cZ$-valued tree of depth $T=3$. 
  The pink line shows a path of $\eps_1=1$ and $\eps_2=-1$.}
  \label{fig:tree}
  \end{figure}

We denote the collection of such trees of depth $T$
as $\mathbb{T}_T(\cZ)$. 
For simplicity, we write $\eps_{1:t}=(\eps_1,\dots,\eps_t)$, 
and write $\bz=(z_1,\dots,z_T)$. 
The following lemma turns the tail bound~\eqref{eq:lem_adapt_1} in Step I 
into the tail probability of a tree Rademacher process, 
whose proof is in Appendix~\ref{app:lem_adapt_2}. 

\begin{lemma}\label{lem:adapt_2}
Let $\eps_1,\dots,\eps_T$ be independent Rademacher random variables. Then for any $\pi\in\Pi$,
\$
&\PP\Bigg(  \frac{\big|\sum^T_{t=1} \{ \hgamma'_t(\pi  ) - \hgamma_t(\pi ) \} \big|}{ 
  \sqrt{ \sum_{t=1}^T \hat{U}_t(\pi ) } \cdot\sqrt{\log\big(56\sum_{t=1}^T \hat{U}_t(\pi )\big)} }
\geq \frac{\xi}{8} \Bigggiven \{X_t\}_{t=1}^T  \Bigg) \notag \\
&\leq \sup_{\bz\in \mathbb{T}_T(\cZ) }\PP_{\eps}\Bigg( 
  \bigg|\frac{\sum^T_{t=1} \eps_t \cdot
  \big\{\hat{\gamma}_t(\pi; X_t,z_t(\eps_{1:(t-1)})) - \hat{\gamma}_t'(\pi;X_t,z_t(\eps_{1:(t-1)}))\big\}}
  { \sqrt{ \sum^T_{t=1} \hat{u}_t(\pi;X_t,z_t(\eps_{1:(t-1)})) } \sqrt{\log\big(56\sum^T_{t=1} \hat{u}_t(\pi;X_t,z_t(\eps_{1:(t-1)}))\big)} }\bigg|
  \ge \frac{\xi}{8}   \Bigg),
\$
where $\PP_\eps$ denotes the probability over $\eps_1,\dots,\eps_T$ 
conditional on everything else. 
\end{lemma}

\noindent\textit{Step III: Self-normalized concentration.}
Next, we invoke 
the self-normalized concentration inequality 
in Lemma~\ref{lemma:self_normal} to control the tail bound in Step II. 
A crucial property that allows us to 
apply this inequality is the conditional symmetry of the tree Rademacher process 
$ \sum^T_{t=1} \eps_t \cdot
\big\{\hat{\gamma}_t(\pi; X_t,z_t(\eps_{1:(t-1)})) - \hat{\gamma}_t'(\pi;X_t,z_t(\eps_{1:(t-1)}))\}$, coupled with the self-normalized-type 
regularization term $\sum_{t=1}^T \hat{u}_t(\pi;X_t,z_t(\eps_{1:(t-1)}))$ 
we design. 
The following lemma is proved in Appendix~\ref{app:lem_adapt_3}. 

  \begin{lemma}\label{lem:adapt_3}
    For any constant $\xi \geq \sqrt{2}/2$, we have 
    \$
    &\PP_\eps\bigg( \frac{|\sum^T_{t=1} \eps_t\cdot 
    \{\hat{\gamma}_t (\pi;X_t,z_t^*(\eps_{1:t-1})) - \hat{\gamma}'_t(\pi;X_t,z_t^*(\eps_{1:t-1}))\}|  }
    {    \sqrt{ \sum^T_{t=1}  \hat{u}_t (\pi;X_t,z_t^*(\eps_{1:(t-1)}))  } \sqrt{\log [56\sum^T_{t=1}  \hat{u}_t (\pi;X_t,z_t^*(\eps_{1:(t-1)}))]}  } 
    \ge \xi \Big)  \le 2e^{-\xi^2/8} .
    \$
    \end{lemma}

The above lemma, together with Lemma~\ref{lem:adapt_1}
and Lemma~\ref{lem:adapt_2} subsequently leads to
\$
\PP\bigg(  \sup_{\pi \in \Pi} \frac{1}{T} \cdot
\frac{|\sum^T_{t=1}\hgamma_t(\pi) - \mu(X_t,\pi(X_t))|}
{V(\pi) \sqrt{2\log(30T\cdot V(\pi))} }\ge \xi \bigggiven \{X_t\}_{t=1}^T \bigg) 
  \leq 2\big|\Pi(X)\big|\cdot \exp\Big(-\frac{\xi^2}{8 }\bigg),
\$

Invoking the Natarajan lemma (c.f.~Lemma~\ref{lemma:nat_lemma}), we further have 
\$
&\PP\bigg(  \sup_{\pi \in \Pi} \frac{1}{T} \cdot
\frac{|\sum^T_{t=1}\hgamma_t(\pi) - \mu(X_t,\pi(X_t))|}
{V(\pi)\sqrt{2\log(30T\cdot V(\pi))}}\ge \xi \bigggiven \{X_t\}_{t=1}^T \bigg) \\
&\leq 2T^{\ndim(\Pi)} K^{2\ndim(\Pi)}\cdot 
\exp\Big(-\frac{\xi^2}{ 8 }\Big) \leq \delta/2
\$
by taking $\xi =  3\sqrt{\ndim(\Pi)\log(TK^2) + \log(4/\delta)}$. 
We have thus controlled term (i).

\noindent\textit{Step IV: Controlling term (ii).} The following lemma is similar to 
Step III of the fixed policy results, whose proof is in 
Appendix~\ref{app:lem_adapt_4}. 

\begin{lemma}\label{lem:adapt_4}
For any $\delta \in (0,1)$, with probability at least $1-\delta/2$, there is
\$
\sup_{\pi \in \Pi} \, \Big|\sum^T_{t=1} 
\frac{\mu(X_t,\pi(X_t)) - Q(\pi)}{V(\pi)}\Big|
\le 2\sqrt{2 \cdot \ndim(\Pi) \cdot \log (TK^2) + 2\log(16/\delta) }.
\$
\end{lemma}

Finally, taking a union bound over terms (i) and (ii), we know that 
\$
\bigg|\frac{\hat{Q}(\pi) - Q(\pi)}{V(\pi)\sqrt{2\log(30T\cdot V(\pi))}}\bigg|  
\leq 6 \sqrt{\ndim(\Pi)\log(TK^2) + \log(16/\delta)} \leq \beta
\$
holds with probability at least $1-\delta$. 
Combining this result with Lemma 3.1 by taking $R(\pi)=\beta\cdot V(\pi)$, 
we conclude 
the proof of Theorem~\ref{thm:adapt_upper}. 
\end{proof}

\section{Proof of fixed-policy results}

This section collects the proofs of all intermediate lemmas 
in the proof of Theorem~\ref{thm:fix_upper}, 
and the proof of the lower bound in Theorem~\ref{thm:lower}. 

\subsection{Proof of Lemma~\ref{lemma:symmetrization}}
\label{app:proof_symmetrization}

\begin{proof}[Proof of Lemma~\ref{lemma:symmetrization}]
To begin with, for any $\xi \ge 4$, we define the event we wish to control as
\$
\cE \defn \bigg\{\sup_{\pi \in \Pi} \Big| \frac{1}{T} 
\sum^T_{t=1} \frac{\hgamma_t(\pi) - \mu(X_t,\pi(X_t))}
{V(\pi)} \Big| \ge \xi\bigg\} 
= \bigg\{ \Big| \frac{1}{T} 
\sum^T_{t=1} \frac{\hgamma_t(\pi^\dagger) - \mu(X_t,\pi^\dagger(X_t))}
{V(\pi^\dagger)} \Big| \ge \xi\bigg\}, 
\$
where we define
\$
\pi^\dagger = \argmax_{\pi \in \Pi} \, \Big| \frac{1}{T}
\sum^T_{t=1} \frac{\hgamma_t(\pi) - \mu(X_t,\pi(X_t))|}{V(\pi)}\Big|.
\$
We assume this supremum is attained; 
otherwise, a limiting argument yields exactly the same conclusion. 
Here, both $\cE$ and $\pi^\dagger$ are measurable with respect to
$\{X_t,A_t,Y_t\}_{t=1}^T$.  
Next, we define two additional events 
\$
\cE_1 &= \bigg\{\Big|\frac{1}{T} \sum^T_{t=1} \hgamma'_t(\pi^\dagger) - \mu(X_t,\pi^\dagger(X_t))\Big| 
\ge 2 \vp(\pi^\dagger) \bigg\},\\ 
\cE_2 &= \Big\{ V_{\rm s}'(\pi^\dagger)  \ge 2\cdot \max \big\{\vh(\pi^\dagger) ,\vp(\pi^\dagger)  \big\} \Big\}.
\$
The general idea of our proof is to control the probability of $\cE$ by proving 
\#\label{eq:step_general}
\PP\big(\cE \biggiven \{X_t\}_{t=1}^T\big) \leq 2\cdot \PP\big( \cE\cap \cE_1^c \cap \cE_2^c \biggiven \{X_t\}_{t=1}^T\big).
\#
From there, we will employ a symmetrization trick 
using Rademacher random variables to further bound 
$\PP\big( \cE\cap \cE_1^c \cap \cE_2^c \biggiven \{X_t\}_{t=1}^T\big)$ 
by the tail probability of a Rademacher process. 

A key step in the proof of~\eqref{eq:step_general}
is to show that 
\#\label{eq:step_bayes}
\PP\big( \cE_1^c\cap \cE_2^c\biggiven \{X_t,A_t,Y_t\}_{t=1}^T\big) \geq \frac{1}{2}. 
\#
The  next lemma proves~\eqref{eq:step_bayes}, 
which controls the deviation of 
$\frac{1}{T}\sum_{t=1}^T \hat\Gamma_t'(\pi)$ from 
its (conditional) expectation 
$\frac{1}{T}\sum_{t=1}^T \mu(X_t,\pi(X_t))$, and 
the sample uncertainty term $V_s'(\pi)$ for the copies 
from that of the original data.  
Its proof is deferred to Appendix~\ref{app:lem_copy_moment}. 
\begin{lemma}
\label{lemma:copy_moments}
    For any fixed $\pi \in \Pi$, we have 
    \$
    \PP\bigg(\Big|\frac{1}{T}  \sum^T_{t=1} \hgamma'_t(\pi) - \mu(X_t,\pi(X_t))\Big| 
    \ge 2 \vp(\pi) \Biggiven \{X_t,A_t,Y_t\}_{t=1}^T\bigg)
    \le \frac{1}{4},
    \$
    and
    \$
    \PP\Big(V_{\rm s}'(\pi)^2 \ge 4\cdot \max \big\{\vh(\pi)^2,\vp(\pi)^2 \big\}
    \Biggiven \{X_t,A_t,Y_t\}_{t=1}^T\Big) \le \frac{1}{4}.
    \$
\end{lemma}

By Lemma~\ref{lemma:copy_moments} and a union bound we get~\eqref{eq:step_bayes}. 
The tower property subsequently leads to
\$
& \PP\big( \cE_1^c\cap \cE_2^c
\biggiven \{X_t\}_{t=1}^T, \cE \big)
\ge \frac{1}{2}.
\$
Multiplying both sides of the above by $\PP(\cE \given \{X_t\}_{t=1}^T)$ 
and applying Bayes' rule,
we have
\#\label{eq:tail_event_1}
\frac{1}{2}\cdot \PP\big(\cE \biggiven \{X_t\}_{t=1}^T\big) 
&\leq \PP\big( \cE_1^c\cap \cE_2^c
\biggiven \{X_t\}_{t=1}^T, \cE \big)
\times \PP\big(\cE \biggiven \{X_t\}_{t=1}^T\big)
= \PP\big(\cE\cap\cE_1^c\cap \cE_2^c \biggiven \{X_t\}_{t=1}^T\big).
\#  
Note that on the event $\cE_1^c\cap \cE $, it holds for any $\xi\geq 4$ that 
\$
\frac{1}{T}\bigg|\sum^T_{t=1} \hgamma_t(\pi^\dagger) - \hgamma_t'(\pi^\dagger)\bigg|
& \geq \frac{1}{T}\bigg|\sum^T_{t=1} \hgamma_t(\pi^\dagger) - \mu(X_t,\pi^\dagger(X_t))\bigg| 
- \frac{1}{T}\bigg|\sum^T_{t=1}\mu(X_t,\pi^\dagger(X_t)) - \hgamma'_t(\pi^\dagger) \bigg|\\
& \geq \frac{1}{T}\bigg|\sum^T_{t=1} \hgamma_t(\pi^\dagger) - \mu(X_t,\pi^\dagger(X_t))\bigg| 
- 2 \vp(\pi^\dagger)  \\ 
&\geq \frac{1}{T}\bigg|\sum^T_{t=1} \hgamma_t(\pi^\dagger) - \mu(X_t,\pi^\dagger(X_t))\bigg| 
- \frac{\xi}{2} \cdot \max\big\{\vs(\pi^\dagger), \vp(\pi^\dagger), \vh(\pi^\dagger)\big\} \\ 
&\geq \frac{\xi}{2} \cdot \max\big\{\vs(\pi^\dagger), \vp(\pi^\dagger), \vh(\pi^\dagger)\big\} ,
\$
where the first step uses the 
triangle inequality, the third step 
uses the fact that $\xi\geq 4$, 
and the final step follows from
the fact that $\frac{1}{T}\big|\sum^T_{t=1} \hgamma_t(\pi^\dagger) - \mu(X_t,\pi^\dagger(X_t))\big| \geq \xi \cdot \max\big\{\vs(\pi^\dagger), \vp(\pi^\dagger), \vh(\pi^\dagger)\big\} $ on the event $\cE$. 
Then, on the event $\cE_1^c\cap\cE_2^c\cap \cE $, 
the above inequality be further lower bounded as
\$
\frac{1}{T}\bigg|\sum^T_{t=1} \hgamma_t(\pi^\dagger) - \hgamma_t'(\pi^\dagger))\bigg|  
&\geq  \frac{\xi}{4} \cdot \max\big\{2\vs(\pi^\dagger), \vs'(\pi^\dagger)\big\}
\geq \frac{\xi}{8} \cdot  \big\{ \vs(\pi^\dagger) + \vs'(\pi^\dagger)\big\} \\
&\geq \frac{\xi}{8T}\cdot \bigg(\sum^T_{t=1} \frac{\big(\ind\{A_t = \pi^\dagger(X_t) + \ind\{A_t' = \pi^\dagger(X_t)\}\}\big)^2}
{e(X_t,\pi^\dagger(X_t))^2} \bigg)^{1/2},
\$
where the last step is due the triangle inequality 
$\|x\|+ \|y\|\geq \|x+y\|$ for $x,y\in \RR^{T}$ 
and noting that $V_s(\pi^\dagger)=\|x\|$ for 
$x_t=\ind\{A_t=\pi^\dagger(X_t)\}/e(X_t,\pi^\dagger(X_t))$ 
and $V_s'(\pi^\dagger) = \|y\|$ for 
$y_t=\ind\{A_t'=\pi^\dagger(X_t)\}/e(X_t,\pi^\dagger(X_t))$ .
The above relation then leads to the following inequality: 
\$
\eqref{eq:tail_event_1}
 & \le   \PP\Bigg( \bigg|\sum^T_{t=1} \big\{\, \hgamma_t(\pi^\dagger) - \hgamma_t'(\pi^\dagger)\big\} \bigg| \\
   &\qquad\qquad\qquad\qquad    \ge \frac{\xi}{8} \cdot \bigg\{\sum^T_{t=1} \frac{(\ind\{A_t = \pi^\dagger(X_t)\} + \ind\{A_t' = \pi^\dagger(X_t)\})^2}
       {e(X_t,\pi^\dagger(X_t))^2}  
       \bigg\}^{1/2}
\Bigggiven \{X_t\}_{t=1}^T\Bigg)\\
& \le   \PP\Bigg( \sup_{\pi\in\Pi} \bigg\{ \bigg|\sum^T_{t=1}\big\{\, \hgamma_t(\pi ) - \hgamma_t'(\pi )\big\}\bigg| 
\\
&\qquad\qquad\qquad \times  \bigg[ \sum^T_{t=1} \frac{(\ind\{A_t = \pi (X_t)\} + \ind\{A_t' = \pi (X_t)\})^2}
{e(X_t,\pi (X_t))^2}  
\bigg]^{-1/2} \bigg\} \ge \frac{\xi}{8}
\Bigggiven \{X_t\}_{t=1}^T\Bigg) \\  
&= \PP\Bigg( \sup_{\pi\in\Pi} \bigg\{\bigg|\sum^T_{t=1}\varepsilon_t \big\{\, \hgamma_t(\pi ) - \hgamma_t'(\pi )\big\}\bigg| \\
&\qquad\qquad\qquad \times \bigg[ \sum^T_{t=1} \frac{(\ind\{A_t = \pi (X_t)\} + \ind\{A_t' = \pi (X_t)\})^2}
{e(X_t,\pi (X_t))^2}  
\bigg]^{-1/2} \bigg\} \ge \frac{\xi}{8}
\Bigggiven \{X_t\}_{t=1}^T\Bigg),  
\$
where the probability is taken over $\{A_t,Y_t ,A_t',Y_t' \}_{t=1}^T$ 
and the Rademacher variables $\eps_1,\dots,\eps_T$ that are i.i.d.~from 
Unif$(\{\pm 1\})$.  
Above, the last step uses the exchangeability between $\{A_t,Y_t\}_{t=1}^T$ and
$\{A_t',Y_t'\}_{t=1}^T$ conditional on $\{X_t\}_{t=1}^T$. 
By the tower property, we know that 
\$
&\PP\Bigg( \sup_{\pi\in\Pi} \bigg\{\bigg|\sum^T_{t=1}\varepsilon_t \big\{\, \hgamma_t(\pi ) - \hgamma_t'(\pi )\big\}\bigg| \\
&\qquad\qquad\qquad \times \bigg[ \sum^T_{t=1} \frac{(\ind\{A_t = \pi (X_t)\} + \ind\{A_t' = \pi (X_t)\})^2}
{e(X_t,\pi (X_t))^2}  
\bigg]^{-1/2} \bigg\} \ge \frac{\xi}{8}
\Bigggiven \{X_t\}_{t=1}^T\Bigg)\\ 
&= \EE\Bigg[\PP\Bigg( \sup_{\pi\in\Pi} \bigg\{\bigg|\sum^T_{t=1}\varepsilon_t \big\{\, \hgamma_t(\pi ) - \hgamma_t'(\pi )\big\}\bigg| 
\cdot \bigg[ \sum^T_{t=1} \frac{(\ind\{A_t = \pi (X_t)\} + \ind\{A_t' = \pi (X_t)\})^2}
{e(X_t,\pi (X_t))^2}  
\bigg]^{-1/2} \bigg\} \\ 
&\qquad \qquad \qquad \qquad  \qquad \qquad \qquad \qquad \qquad \qquad \ge \frac{\xi}{8} \bigggiven \{X_t,A_t,A_t',Y_t,Y_t'\}_{t=1}^T \bigg)
\Bigggiven \{X_t\}_{t=1}^T\Bigg],
\$
where the inner conditional probability is over the 
randomness of the Rademacher random variables while conditioning on 
$\{X_t,A_t,A_t',Y_t,Y_t'\}_{t=1}^T$. Taking 
a supremum over all possible realizations of $\{X_t,A_t,Y_t,A_t',Y_t'\}_{t=1}^T$,  we have 
\$
\eqref{eq:tail_event_1}
 & \le   \sup_{x,a,a',y,y'}
 \PP\Bigg( \bigg|\sum^T_{t=1}\varepsilon_t \big\{\, \hat\gamma_t(\pi ) - \hat\gamma_t'(\pi )\big\}\bigg| \\
 &\qquad\qquad\qquad \qquad\ge \frac{\xi}{8}\cdot \bigg\{\sum^T_{t=1} \frac{(\ind\{a_t = \pi (x_t)\} + \ind\{a_t' = \pi (x_t)\})^2}
{e(x_t,\pi (x_t))^2}  
\bigg\}^{ 1/2} \textnormal{ for all }\pi\in\Pi \Bigg),  
\$
where for any $x,a,a',y,y'$, the probability is only over the Rademacher variables 
$\eps_1,\dots,\eps_T$ while holding other quantities at 
their realized values. 
This concludes the proof of Lemma~\ref{lemma:symmetrization}. 
\end{proof}

\subsection{Proof of Lemma~\ref{lemma:tail_rademacher}}
\label{app:subsec_rademacher}

\begin{proof}[Proof of Lemma~\ref{lemma:tail_rademacher}]

Fix any realizations $x,a,a',y,y'$ and any $\pi \in \Pi$, 
\$
& \big(\hat{\gamma}_t(\pi) - \hat{\gamma}'_t(\pi)\big)^2  \\
&= \bigg(\frac{\ind\{a_t = \pi(x_t)\}}{e(x_t,\pi(x_t))}
\cdot \big\{ y_t - \hat{\mu}(x_t,\pi(x_t))\big\} 
- \frac{\ind\{a_t' = \pi(x_t)\}}{e(x_t,\pi(x_t))}
\cdot \big\{ y_t' - \hat{\mu}(x_t,\pi(x_t))\big\} \bigg)^2\\
&\leq 2  \cdot \hat{\mu}(x_t,\pi(x_t))^2\cdot
\bigg(\frac{\ind\{a_t = \pi(x_t)\} - \ind\{a_t' = \pi(x_t)\}}{e(x_t,\pi(x_t))} \bigg) ^2 \\ 
&\qquad\qquad\qquad+ 2 \cdot \bigg(y_t \cdot \frac{\ind\{a_t = \pi(x_t)\}}{e(x_t,\pi(x_t))}
- y_t' \cdot \frac{\ind\{a_t' = \pi(x_t)\}}{e(x_t,\pi(x_t))}\bigg)^2\\
&\leq 4 \cdot 
\frac{(\ind\{a_t = \pi(x_t)\} + \ind\{a_t' = \pi(x_t)\})^2}{e(x_t,\pi(x_t))^2},
\$
where the second line uses 
the triangle inequality, 
and the third line uses the fact that $\hat\mu\in[0,1]$ and $y_t,y_t'\in[0,1]$. 
Then by Hoeffding's inequality (c.f.~Lemma~\ref{lem:hoeffding}), over the randomness in $\varepsilon$ one has 
\$
&\PP_\eps\Bigg( \bigg|\sum^T_{t=1} \eps_t \cdot (\hat{\gamma}_t(\pi) - \hat{\gamma}'_t(\pi))\bigg| 
\ge \frac{\xi}{8} \cdot \bigg\{\sum^T_{t=1} \frac{(\ind\{a_t = \pi(x_t)\} + \ind\{a_t' = \pi(x_t)\})^2}
{e(x_t,\pi(x_t))^2} \bigg\}^{1/2} \Bigg)  \\ 
&\le 2\cdot \exp\Big(-\frac{\xi^2}{512}\Big).
\$
We note that fixing $x,a,a',y,y'$, the only 
difference the choice of $\pi\in\Pi$ impacts the above event 
is through $\pi(x_1),\dots,\pi(x_T)$. 
Letting $\Pi(x) = \{(\pi(x_1),\ldots,\pi(x_T)): \pi \in \Pi\}$, 
from Lemma~\ref{lemma:nat_lemma} we know $|\Pi(x)|\leq T^{\ndim(\Pi)}K^{\ndim(\Pi)}$, 
where $K=|\cA|$ is the number of actions. 
Then, taking a union bound over $\Pi(x)$, we have
\$
&\PP_\eps\Bigg( \bigg|\sum^T_{t=1} \eps_t \cdot (\hat{\gamma}_t(\pi) - \hat{\gamma}'_t(\pi))\bigg| 
\ge \frac{\xi}{8} \cdot \bigg\{\sum^T_{t=1} \frac{(\ind\{a_t = \pi(x_t)\} + \ind\{a_t' = \pi(x_t)\})^2}
{e(x_t,\pi(x_t))^2} \bigg\}^{1/2}\textnormal{ for all }\pi\in\Pi \Bigg) \\
&\le 
2|\Pi(x)| \cdot \exp\Big(-\frac{\xi^2}{512}\Big)
\le 2 T^{\ndim(\Pi)}K^{2\ndim(\Pi)} \cdot \exp\Big(-\frac{\xi^2}{512}\Big).
\$
Finally, taking $\xi = 16\sqrt{2\ndim(\Pi) \log(TK^2) + 2\log(2/\delta)}$
completes the proof.
\end{proof}

\subsection{Proof of Lemma~\ref{lemma:bounded_part}}
\label{app:subsec_bounded_part}
\begin{proof}[Proof of Lemma~\ref{lemma:bounded_part}]
For each $t\in [T]$, we let $X_t'$ be an i.i.d.~copy of $X_t$. 
Recall that $Q(\pi) = \EE[Q(X_t,\pi)] = \EE[Y(\pi(X))]$ 
is the policy value. 
For any fixed $\eta > 1/\sqrt{2}$, by Markov's inequality we have  
\#\label{eq:mom_bound}
& \PP\bigg(\Big|\sum^T_{t=1}  \mu\big(X_t',\pi(X_t')\big) -Q(\pi) \Big| \ge \eta \sqrt{T}\bigg)\notag\\
& \le \frac{1}{\eta^2 T}\cdot
\EE\Bigg[\bigg(\sum^T_{t=1} \mu(X_t',\pi(X_t')) - Q(\pi) \bigg)^2\Bigg]
=  \frac{1}{\eta^2 T}\cdot
\sum^T_{t=1} \EE\Big[\big(\mu(X_t',\pi(X_t')) -Q(\pi)\big)^2\Big]
\le \frac{1}{\eta^2} \leq \frac{1}{2}, 
\#
where the second line uses the fact that $Q(X_t',\pi)$ are i.i.d., 
and that $\mu(X_t',\pi(X_t')),~Q(\pi)\in[0,1]$. 

Let
$\pi^* = \argmax_{\pi \in \Pi} |\sum^T_{t=1} Q(X_t,\pi) - Q(\pi)|$, 
which is measurable with respect to $\{X_t\}_{t=1}^T$. 
Since $\{X_t\}_{t=1}^T$ are independent of $\{X_t'\}_{t=1}^T$, by~\eqref{eq:mom_bound}
we have
\$
\PP\bigg(\Big|\frac{1}{T}\sum^T_{t=1} \mu(X_t',\pi^*(X_t')) - Q(\pi)\Big| < \eta \sqrt{T} 
\bigggiven \{X_t\}_{t=1}^T\bigg)
\ge \frac{1}{2}.
\$
Then for any $\eta\geq 1/\sqrt{2}$, 
we have  
\#
\label{eq:x_symmetry}
& \frac{1}{2} \cdot \PP\bigg(\sup_{\pi \in \Pi}\Big|\sum^T_{t=1}  \mu(X_t,\pi(X_t)) -Q(\pi)\Big| \ge 2\eta \sqrt{T}\bigg)\notag  \\
& = \frac{1}{2} \cdot \PP\bigg(\Big|\sum^T_{t=1}  \mu(X_t,\pi^*(X_t)) -Q(\pi^*)\Big| \ge 2\eta \sqrt{T}\bigg) \notag \\
&\le  \PP\bigg(\Big|\sum^T_{t=1} \mu(X_t',\pi^*(X_t')) -Q(\pi^*)\Big| \le \eta \sqrt{T} \bigggiven \{X_t\}_{t=1}^T \bigg) \notag \\
&\qquad \times \PP\bigg(\Big|\sum^T_{t=1} \mu(X_t,\pi^*(X_t)) -Q(\pi^*) \Big| \ge 2\eta \sqrt{T}\bigg) \notag \\
& =  \PP\bigg(\Big|\sum^T_{t=1} \mu(X_t',\pi^*(X_t')) -Q(\pi^*) \Big| \le \eta \sqrt{T},
\, \Big|\sum^T_{t=1} \mu(X_t,\pi^*(X_t)) -Q(\pi^*)\Big| \ge 2\eta \sqrt{T}\bigg) \notag \\
&\leq \PP\bigg(\Big|\sum^T_{t=1} \mu(X_t',\pi^*(X_t')) - \mu(X_t, \pi^*(X_t)) \Big| \ge \eta \sqrt{T} 
\bigg) \notag \\
&\le \PP\bigg(\sup_{\pi \in \Pi} 
\Big|\sum^T_{t=1} \mu(X_t',\pi(X_t')) -\mu(X_t, \pi(X_t))\Big| \ge \eta \sqrt{T}\bigg),
\#
where the next-to-last step is due to 
the triangle inequality.
Let $\eps_1,\ldots,\eps_T \sim \textrm{Unif}\{\pm 1\}$ be i.i.d.~Rademacher random variables.
By the exchangeability of $X_t$ and $X_t'$,
\$
\eqref{eq:x_symmetry} = 
& \PP\bigg(\sup_{\pi \in \Pi} 
\Big|\sum^T_{t=1} \eps_t\cdot \big(\mu(X_t',\pi(X_t')) -\mu(X_t, \pi(X_t))\big)\Big| \ge \eta \sqrt{T}\bigg)\\
&\leq  2 \PP\bigg(\sup_{\pi \in \Pi} 
\Big|\sum^T_{t=1} \eps_t\cdot \mu(X_t, \pi(X_t))\Big| \ge \eta \sqrt{T}/2\bigg)\\
& = 2 \EE\Bigg[\PP_\eps\bigg(\sup_{\pi \in \Pi} 
\Big|\sum^T_{t=1} \eps_t\cdot \mu(X_t, \pi(X_t))\Big| \ge \eta \sqrt{T}/2 \bigggiven \{X_t\}_{t=1}^T\bigg)\Bigg]\\
& \leq  2\sup_{x}\, \PP_{\eps}\bigg(\sup_{\pi \in \Pi} 
\Big|\sum^T_{t=1} \eps_t \cdot  \mu(x_t,\pi(x_t)) \Big| \ge \eta \sqrt{T}/2 \bigg),
\$
where $x = (x_1,\ldots,x_T)$ denotes any realization of $\{X_t\}_{t=1}^T$. 
Above, $\PP_\varepsilon$ denotes the probability only 
over the Rademacher random variables while 
holding the data as fixed. 
For any $x$, similar to our proof of Lemma~\ref{lemma:tail_rademacher}, 
\$
&\PP_{\eps}\bigg(\sup_{\pi \in \Pi} \Big|\sum^T_{t=1} \eps_t \cdot  \mu(x_t,\pi(x_t)) \Big| \ge \eta \sqrt{T}/2 \bigg) \\
&\leq \big|\Pi(x)\big| \cdot 
\PP_{\eps}\bigg(\Big|\sum^T_{t=1} \eps_t \cdot  \mu(x_t,\pi(x_t)) \Big| \ge \eta \sqrt{T}/2 \bigg)\\
&\leq 2T^{\ndim(\Pi)} K^{2\ndim(\Pi)} \cdot  
\PP_{\eps}\bigg(\Big|\sum^T_{t=1} \eps_t \cdot  \mu(x_t,\pi(x_t)) \Big| \ge \eta \sqrt{T}/2 \Big)\\
&\leq 2T^{\ndim(\Pi)} K^{2\ndim(\Pi)} \cdot  
      \exp(-\eta^2/8).
\$
Above, the third line applies Natarajan's lemma (c.f.~Lemma~\ref{lemma:nat_lemma}) 
and the fourth line applies Hoeffding's inequality. 
Finally, taking $\eta = \sqrt{8 (\ndim(\Pi)\log(TK^2)+ \log(8/\delta))}$ and 
noting that $V(\pi) \ge \sqrt{T}$ 
concludes the proof of Lemma~\ref{lemma:bounded_part}. 
\end{proof}

\subsection{Proof of Corollary~\ref{cor:fixed}}
\label{app:subsec_cor_fixed}

\begin{proof}[Proof of Corollary~\ref{cor:fixed}]
When $C_* \le 1/T$, the first term is dominated by the constant term,
and the bound follows from the boundedness of $Q(\pi)$. We now focused on 
the case where $C_* > 1/T$. From Theorem~\ref{thm:fix_upper}, 
for any $\beta \geq 10\sqrt{2(\ndim(\Pi) \log(TK^2) + \log(16/\delta))}$, we know that 
$\cL(\hat\pi;\cC,\Pi) \leq  2 \beta \cdot V(\pi^*)$
with probability at least $1-\delta$. In the following, 
we are to bound  $V(\pi^*)$ by 
controlling the three terms in Equation (7) separately. 

\noindent\textit{Bound $\vs(\pi^*)$.}
Note that $D_t:=\frac{\ind\{A_t=\pi^*(X_t)\}}{e(X_t,\pi^*(X_t))^2}$, 
$t=1,\dots,T$, 
are i.i.d.~random variables with $\EE[D_t]=\EE[ e(X ,\pi^*(X))^{-1}]$.
Recall that $C_*\leq \inf_{x}\mu(x,\pi^*(x))$ 
is the overlap for the optimal policy, so $|D_t|\leq 1/C_*^2$. 
We also know that
\$
\Var(D_t) \leq \EE[D_t^2] = \EE\big[ e(X_t,\pi^*(X_t))^{-3}  \big] \leq C_*^{-3},
\$
Next, invoking Bernstein's inequality (c.f.~Lemma~\ref{lem:bernstein}), we have
for any $x > 0$ that
\$
\PP\bigg(  \sum_{t=1}^T \frac{\ind\{A_t=\pi^*(X_t)\}}{e(X_t,\pi^*(X_t))^2} - 
T\cdot \EE\big[ e(X ,\pi^*(X))^{-1}\big] \geq x \bigg)
\leq \exp\bigg( \frac{-x^2/2}{T /C_*^3 + x/(3C_*^2)}  \bigg)
\$
Take $x = 2\log(1/\delta) \sqrt{T C_* ^{-3}}$. Since $C_* T > 1$,  we know 
\$
\frac{x^2/2}{T/C_*^3 + x/(3C_*^2)} = 
\frac{2\log^2(1/\delta)}{1 + \frac{2}{3}\log (1/\delta) /\sqrt{TC_*}} \geq \log(1/\delta) 
\$
as long as $\log(1/\delta) \geq 1$.
Therefore, with probability at least $1-\delta$, we have 
\$
\vs(\pi^*)&\leq 
\frac{1}{T}  
\sqrt{ T\cdot \EE\big[ e(X ,\pi^*(X))^{-1}\big] + 2\log(1/\delta) \sqrt{T / C_* ^{3}} } \\ 
&\leq \frac{1}{\sqrt{T}} \cdot\sqrt{C_*^{-1}  + 2 \log(1/\delta)\cdot  C_*^{-3/2} \cdot T^{-1/2} } 
\leq \frac{\log(1/\delta)}{\sqrt{C_*T}}.
\$
where 
the last inequality uses the fact that $\sqrt{a+b}=\sqrt{a}\cdot\sqrt{1+b/a} \leq \sqrt{a} \cdot b/(2a) $ for any $a,b>0$. 

\noindent\textit{Bounding $\vp(\pi^*)$.} 
Since $e(X_t,\pi^*(X_t))\geq C^*$ almost surely, we directly have the upper bound 
\$
\vp(\pi^*) \leq \frac{\sqrt{T/C_*}}{T} = \frac{1}{\sqrt{C_* T}}.
\$

\noindent\textit{Bounding $\vh(\pi^*)$.} 
Finally, the boundedness of $e(X_t,\pi^*(X_t))$ gives 
\$
\vh(\pi^*) \leq \frac{1}{T}  (T/C_*^3)^{1/4} \leq \frac{1}{(C_* T)^{3/4}}.
\$

Combining the three bounds, when $C_*\ge 1/T$, 
we have that with probability at least $1-\delta$ that
\$
V(\pi^*) \leq \max\bigg\{\frac{\log(1/\delta)}{\sqrt{C_*T}} ,~ \frac{1}{\sqrt{C_*T}},~
\frac{1  }{(C_*T)^{3/4}}   \bigg\} = \frac{2\log(1/\delta)}{\sqrt{C_*T}}.
\$
Further taking a union bound with the event  $\cL(\hat\pi;\cC,\Pi) \leq  2 \beta \cdot V(\pi^*)$ from Theorem~\ref{thm:fix_upper}, 
we know that with probability at least $1-\delta$, 
\$
\cL(\hat\pi;\cC,\Pi) &\leq  2c  \cdot \sqrt{  \ndim(\Pi) \log(TK^2) + \log(2/\delta) } 
\cdot \frac{\log(1/\delta)}{\sqrt{C_*T}}\\
&\leq 2c  \cdot \sqrt{\ndim(\Pi) \log(TK^2)   } 
\cdot \frac{\log(2/\delta)^{3/2}}{\sqrt{C_*T}}\\
&\leq  2c  \cdot \sqrt{ \frac{ \ndim(\Pi)\log(TK^2) }{C_* T}}
\cdot \log(2/\delta)^{3/2},
\$
which concludes the proof of Corollary~\ref{cor:fixed}.
\end{proof}

\subsection{Proof of Corollary~\ref{cor:general}}
\label{app:cor_general}

\begin{proof}[Proof of Corollary~\ref{cor:general}]
  Under the setup of Theorem~\ref{thm:fix_upper}, we know that with probability at least $1-\delta$,
  \$
  \big| \hat{Q}(\pi) - Q(\pi) \big| \leq \beta\cdot V(\pi),\quad \text{for all }\pi\in\Pi.
  \$
  Since $\hat\pi = \argmax_{\pi\in\Pi}\{\hat{Q}(\pi)-\beta\cdot V(\pi)\}$, 
  we see that 
  \$
  \hat{Q}(\hat\pi) &\geq 
  \sup_{\pi\in\Pi} \big\{ \hat{Q}(\pi)-\beta\cdot V(\pi) + \beta\cdot V(\hat\pi) \big\} \\ 
  &\geq \sup_{\pi\in\Pi} \big\{  {Q}(\pi) - \beta\cdot V(\pi) -\beta\cdot V(\pi)  \big\},
  \$
  where the second inequality uses the uniform bound above and the fact that 
  $V(\hat\pi)\geq 0$. Thus, we conclude the proof of Corollary 4.4. 
\end{proof}

\subsection{Proof of Theorem~\ref{thm:lower}}
\label{app:subsec_lower}
\begin{proof}[Proof of Theorem~\ref{thm:lower}]
Let $S=\{x_1,\dots,x_d\}\subseteq \cX$ with $d:=\ndim(\Pi)$ be the set 
that is N-shattered by $\Pi$ (c.f.~Definition 2.2). 
That is, there exists 
two functions $f_1,f_2\colon \cX\to \cA$ such that 
$f_1(x_i)\neq f_2(x_i)$ for all $i=1,\dots,d$, and 
for every subset $S_0\subseteq S$, there exists a policy $\pi\in \Pi$ 
such that $\pi(x)=f_1(x)$ for all $x\in S_0$ 
and $\pi(x)=f_2(x)$ for all $x\in S\backslash S_0$. 
We let $\cV=\{\pm 1\}^d$ be an index set of size $2^d$.

Let $\delta\in(0,1/4)$ be a constant to be decided later. 
We consider a subset of instances $\cR_0 (C_*,T,\cA,\Pi) = \{R_v\}_{v\in \cV}\subseteq \cR(C_*,T,\cA,\Pi)$ 
indexed by all $v\in \cV$.  
Each element $R_v = (\cC_v,e)\in \cR_{0} (C_*,T,\cA,\Pi)$ 
corresponds to $\cC_v=(\rho,\mu_v,\cA)$, where  $\rho$ is the uniform distribution over $S$, 
and the reward distribution $\mu_v$ is given by 
\$
\PP(Y(a) \given X=x_i) \sim 
\begin{cases}
&\textrm{Bern}(1/2),\quad a = f_1(x_i) \\
&\textrm{Bern}(1/2 + v_i \cdot \delta),\quad a = f_2(x_i), \\ 
&\textrm{Bern}(0),\quad \textrm{otherwise}. 
\end{cases}
\quad \forall x_i\in S.
\$
Since $S$ is N-shattered by $\Pi$, 
for every $v\in \cV$, there exists 
some $\pi_v^*\in \Pi$ such that 
$\pi_v^*(x_i)=f_1(x_i)$ for all $v_i <0$ and 
$\pi_v^*(x_i)=f_2(x_i)$ for all $v_i>0$, which is the optimal 
policy for $\cC_v$ by definition. 
Since $\delta\in(0,1/4)$, we know that 
for any $\pi\in\Pi$, 
\#\label{eq:L_single}
\cL(\pi;\cC_v,\Pi)  \geq \frac{1}{d}\sum_{i=1}^d \delta\cdot \ind\{\pi(x_i) \neq \pi_v^*(x_i)\} .
\#
Also, we define the fixed behavior policy by 
\$
e(x_i,a) = 
\begin{cases}
&C_*,\quad a = f_1(x_i) \\
&C_*,\quad a = f_2(x_i), \\ 
&(1-2C_*)/(K-2),\quad \textrm{otherwise}. 
\end{cases}
\quad \forall x_i\in S.
\$
For any $v\in \cV$, we let $\PP_v(\cdot)$ and $\EE_v[\cdot]$ 
denote the joint distribution of data under $(\cC_v,e)$. 
Then for any (data-dependent) policy $\hat\pi$, we have 
\#\label{eq:lb_1}
\sup_{R_v\in \cR_0(C^*,T)} \EE_{v}\big[ \cL(\hat\pi;\cC_v,\Pi)  \big]
& \geq \frac{1}{2^{d}} \sum_{v\in \cV} \EE_{v}\big[  \cL(\hat\pi;\cC_v,\Pi)  \big]
\geq  \frac{1}{2^{d}} \sum_{v\in \cV} \frac{1}{d}\sum_{i=1}^d \delta \cdot \EE_{v}\big[ \ind\{ \hat\pi(x_i) \neq \pi_{v}^*(x_i)  \big],
\#
where the last equality uses~\eqref{eq:L_single}. 
Now fix any $v\in \cV$, and for any $i\in[d]$, 
we denote $M_i[v]\in\{\pm 1\}^d$ as the vector
that only differs from $v$ in entry $i$, i.e., 
$(M_i[v])_i = -v_i$ and $(M_i[v])_j=v_j$ for all $j\neq i$. 
Thus 
\#\label{eq:lb_2}
\eqref{eq:lb_1} 
&\geq \frac{\delta}{d\cdot 2^{d}} \sum_{i=1}^d  \sum_{v\in \cV}   \EE_{v}\big[ \ind\{ \hat\pi(x_i) \neq \pi_{v}^*(x_i)  \big] \notag \\
&= \frac{\delta}{d\cdot 2^{d}} \sum_{i=1}^d \sum_{v\in\cV\colon v_i=1} \big\{  \EE_{v}\big[ \ind\{ \hat\pi(x_i) \neq \pi_{v}^*(x_i)  \big] +  \EE_{M_i[v]}\big[ \ind\{ \hat\pi(x_i) \neq \pi_{M_i[v]}^*(x_i)  \big] \big\} \notag \\ 
&= \frac{\delta}{d\cdot 2^{d}} \sum_{i=1}^d \sum_{v\in\cV\colon v_i=1} \big\{  \EE_{v}\big[ \ind\{ \hat\pi(x_i) \neq f_2(x_i)  \big] +  \EE_{M_i[v]}\big[ \ind\{ \hat\pi(x_i) \neq f_1(x_i)  \big] \big\} \notag \\ 
&\geq \frac{\delta}{d\cdot 2^{d}} \sum_{i=1}^d \sum_{v\in\cV\colon v_i=1} \big\{  \EE_{v}\big[ \ind\{ \hat\pi(x_i) \neq f_2(x_i)  \big] + 1 - \EE_{M_i[v]}\big[ \ind\{ \hat\pi(x_i) \neq f_2(x_i)  \big] \big\}\notag  \\ 
&\geq  \frac{\delta}{d\cdot 2^{d}} \sum_{i=1}^d \sum_{v\in\cV\colon v_i=1} \big( 1- \textrm{TV}(\PP_v, \PP_{M_i[v]})\big),
\#
where the second line is an equivalent decomposition of $v\in \cV$, 
the third line uses the fact that $\pi_v^*(x_i)=f_2(x_i)$ for $v_i>0$ 
and $\pi_{v^*(x_i)}=f_1(x_i)$ for $v_i<0$, 
and $\textrm{TV}(P,Q)$ is the Total-Variation distance between 
distributions $P,Q$. 
Using the relationship between TV distance and KL divergence (c.f.~\cite[Lemma 2.6]{tsybakov2004introduction}), 
\#\label{eq:tv_kl}
\big( 1- \textrm{TV}(\PP_v, \PP_{M_i[v]})\big) 
\geq \frac{1}{2} \exp\big(-\textrm{KL}(\PP_{v}\|\PP_{M_i[v]})\big),
\#
where $\textrm{KL} (\PP_{v }  \| \PP_{M_i[v]} )$
denotes the KL divergence of $\{X_t,Y_t,A_t\}$ under $R_{v }$ versus under $R_{M_i[v]}$.
By construction, for any $j$, the conditional log-likelihood of reward $y\in \{0,1\}$ is 
\$
&\log \PP_{v}(y\given X_t=x_j,A_t=a) \\
&= \log(1/2) \ind\{a=f_1(x_j)\}
+ \big(y\log(1/2+\delta) + (1-y)(1-\log(1/2+\delta) \big) \cdot\ind\{v_j>0,a=f_2(x_j)\} \\
&\quad + \big(y\log(1/2-\delta) + (1-y)(1-\log(1/2-\delta) \big) \cdot\ind\{v_i<0,a=f_2(x_j)\} \\ 
&=\log(1/2) \ind\{a=f_1(x_j)\}+ (1-y) \ind\{a=f_2(x_j)\} \\
&\quad + (2y-1) \log(1/2+\delta) \cdot\ind\{v_j>0,a=f_2(x_j)\}
+ (2y-1) \log(1/2-\delta) \cdot\ind\{v_j <0,a=f_2(x_j)\}.
\$
Note that $v$ and $M_i[v]$ only differ in entry $i$. 
We then have 
\$
\log\bigg(\frac{\PP_{v} }{\PP_{M_i[v]} }(x_j,y,a)\bigg)
&= (2y-1) \log \bigg( \frac{1/2+\delta}{1/2-\delta} \bigg) \ind\{a=f_2(x_j),j=i\}.
\$
Therefore, the KL divergence can be bounded as 
\#\label{eq:kl_bd}
\textrm{KL} (\PP_{v}  \| \PP_{v'}) &\leq \sum_{t=1}^T \EE\big[  \PP_v\big(A_t=f_2(x_i),X_t=x_i\big) 
\EE_v[2Y-1\given A_t,X_t] \big]\log \bigg( \frac{1/2+\delta}{1/2-\delta} \bigg) 
\leq  \frac{T}{d}\cdot  12\delta^2C_* ,
\#
where the last inequality uses the fact that $\log(1+2\delta)\leq 2\delta$ 
and $\log(1-2\delta) \geq -4\delta$ for $\delta\in(0,1/4)$.  
Combining~\eqref{eq:lb_2},~\eqref{eq:tv_kl}, and~\eqref{eq:kl_bd}, we have 
\$ 
\eqref{eq:lb_2} \geq \frac{\delta}{d\cdot 2^d} \sum_{i=1}^d \sum_{v\in \cV\colon v_i=1}
\frac{1}{2} \exp(- 12\delta^2C_* T /d)
= \delta \cdot \exp(- 12\delta^2C_* T /d).
\$
Finally, taking $\delta=\sqrt{d /(24C_*T)}$, we arrive at 
\#
&\sup_{(\cC,e)\in \cR (C_*,T,\cA,\Pi)} \EE_{v}\big[ \cL(\hat\pi;\cC ,\Pi)  \big]\\
&\geq 
\sup_{R_v\in \cR_0(C_*,T,\cA,\Pi)} \EE_{v}\big[ \cL(\hat\pi;\cC_v,\Pi)  \big]
\geq \frac{\exp(-1/2)}{\sqrt{24}}\sqrt{\frac{d}{C_*T}}\geq 0.12 \sqrt{\frac{d }{C_*T}},
\#
which concludes the proof of Theorem~\ref{thm:lower}. 
\end{proof}

\section{Proof of adaptive results}

\subsection{Proof of Lemma~\ref{lem:adapt_1}}
\label{app:lem_adapt_1}
\begin{proof}[Proof of Lemma~\ref{lem:adapt_1}]
  First, note that 
  $
  \log (896T^2\cdot V(\pi)^2) \leq 2\log(30T\cdot V(\pi))
  $. 
  We define the event of interest as 
  \$ 
   \cE &\defn \bigg\{\sup_{\pi \in \Pi} \Big|\frac{1}{T} \sum^T_{t=1}
  \frac{\hgamma_t(\pi) - \mu(X_t,\pi(X_t))}
  {V(\pi) \sqrt{\log (896T^2\cdot V(\pi)^2)}} \Big|\ge \xi\bigg\}\\
  &=
  \bigg\{  \Big| \frac{1}{T} \sum^T_{t=1}
  \frac{\hgamma_t(\pi^\dagger) - \mu(X_t,\pi^\dagger(X_t))}
  {V(\pi^\dagger)\sqrt{\log (896T^2\cdot V(\pi^\dagger)^2)}} \Big|\ge \xi\bigg\},
  \$
  where we let
  \$\pi^\dagger = \argmax_{\pi\in\Pi}~\bigg\{ \Big| \frac{1}{T} \sum^T_{t=1}
  \frac{\hgamma_t(\pi) - \mu(X_t,\pi(X_t))}
  {V(\pi)\sqrt{\log (896T^2\cdot V(\pi)^2)}} \Big| \bigg\}.
  \$
  Note that $\pi^\dagger$ is measurable with respect to $\{X_t,A_t,Y_t\}_{t=1}^T$. 
  Based on the symmetric copies, we also define 
\$
\vs'(\pi) = \frac{1}{T}  
\bigg( \sum^T_{t=1} \frac{\ind \{A_t' = \pi(X_t)\}}{e_t(X_t,\pi(X_t)\given \cH_t)^2}\bigg)^{1/2}.
\$
Next, we define two additional events 
  \$
  \cE_1 &= \bigg\{   \frac{1}{T} \Big|\sum^T_{t=1} \hgamma'_t(\pi^\dagger) - \mu(X_t,\pi^\dagger(X_t))\Big| 
  \ge 2 \vp(\pi^\dagger) \sqrt{ \log(896T^2\cdot\vp(\pi^\dagger)^2)} \bigg\},  \\ 
  \cE_2 & = \Big\{\vs'(\dpi)  \ge 2\cdot \max \big\{\vh(\dpi) ,\vp(\dpi)   \big\}\Big\}.
  \$
  The following lemma controls the 
  conditional probability of $\cE_1$ and 
  $\cE_2$, whose proof is in
  Appendix~\ref{app:proof_bd_events}.
  \begin{lemma}\label{lem:bd_events}
  Under the above notations, we have 
  \$
  \PP\big( \cE_1 \biggiven \{X_t,A_t,Y_t\}_{t=1}^T\big)\leq 1/4, \quad \textrm{and}\quad 
  \PP\big( \cE_2 \biggiven \{X_t,A_t,Y_t\}_{t=1}^T\big)\leq 1/4.
  \$
  \end{lemma}

Since $\cE$ is measurable with respect to $\{X_t,A_t,Y_t\}_{t=1}^T$, by
the tower property we 
have 
\$
\PP\big(\cE_1^c \cap \cE_2^c \cap \cE \biggiven \{X_t \}_{t=1}^T\big)
    &= \EE\Big[ \EE\big[ \ind_{\{\cE_1^c \cap \cE_2^c\}} \ind_{\{ \cE \}}  \biggiven \{X_t,A_t,Y_t\}_{t=1}^T \big] \Biggiven \{X_t\}_{t=1}^T \Big] \\ 
&=\EE\Big[ \PP\big(  \cE_1^c \cap \cE_2^c  \biggiven \{X_t,A_t,Y_t \}_{t=1}^T \big) 
\cdot \ind_{\{ \cE \}}  \Biggiven \{X_t\}_{t=1}^T \Big] \\ 
&\geq \frac{1}{2} \cdot \PP\big(\cE\given \{X_t\}_{t=1}^T\big),
\$
where the last line uses Lemma~\ref{lem:bd_events} with a union bound. 
We now proceed to bound 
the left-handed side. 
For any $\xi\geq 4$, 
on the event $\cE\cap \cE_1^c$, the triangle inequality implies that 
\#\label{eq:tail_1}
 \frac{1}{T} \Big| \sum^T_{t=1} \hgamma'_t(\pi^\dagger) - &\hgamma_t (\pi^\dagger)\Big| 
\ge \frac{1}{T} \Big|\sum^T_{t=1} \hgamma_t(\pi^\dagger) - \mu(X_t,\pi^\dagger(X_t))\Big| 
- \frac{1}{T} \Big|\sum^T_{t=1} \hgamma_t'(\pi^\dagger) - \mu(X_t,\pi^\dagger(X_t))\Big| \notag\\
&\geq \xi \cdot V(\dpi)\sqrt{\log (896T^2\cdot V(\pi^\dagger)^2)} - 2\vp (\dpi) \geq \frac{\xi}{2}\cdot V(\dpi)\sqrt{\log (896T^2\cdot V(\pi^\dagger)^2)},
\#
where we used the fact that $\xi \ge 4$ and $V(\dpi)\geq \vp(\dpi) \geq 1$. 
Since $V(\dpi)\geq \vs'(\dpi)/2$ 
on the event $\cE_2^c$, 
we know that on the event $\cE\cap \cE_1^c \cap \cE_2^c$, we additionally have 
\#\label{eq:tail_2}
\frac{\xi}{2}\cdot V(\dpi) \geq \frac{\xi}{4}\cdot \big( V(\dpi) + \vs'(\dpi)/2\big) 
\geq \frac{\xi}{8}\cdot \big( \vs(\dpi) + \vs'(\dpi) \big) .
\#
{Taking $a,b\in \RR^T$ 
whose $t$-th entries are given by  $a_t= \ind\{A_t=\dpi(X_t)\} /e_t(X_t,\dpi(X_t)\given \cH_t)$
and $b_t= \ind\{A_t'=\dpi(X_t)\} /e_t(X_t,\dpi(X_t)\given \cH_t) $, 
the triangle inequality implies 
\#\label{eq:tail_3}
\vs(\dpi) + \vs'(\dpi)
&= \frac{\|a\|_2  + \|b\|_2 }{T  } 
\geq \frac{\|a+b\|_2 }{ T  } 
=   \frac{1}{ T}  
\bigg( \sum^T_{t=1} \hat{U}_t(\pi)\bigg)^{1/2},
\# 
where 
recall that
\$
\hat{U}_t(\pi) = 
\frac{(\ind \{A_t = \pi(X_t)\}+\ind\{A_t'=\pi(X_t)\})^2}
{e_t(X_t,\pi(X_t)\given \cH_t)^2}.
\$
Combining~\eqref{eq:tail_1},~\eqref{eq:tail_2} and~\eqref{eq:tail_3}, we know that 
\#\label{eq:joint_bd}
&\PP\big(\cE_1^c \cap \cE_2^c \cap \cE \biggiven \{X_t\}_{t=1}^T\big) \notag \\
&\leq \PP\Bigg( \frac{1}{T} \Big|\sum^T_{t=1} \hgamma'_t(\pi^\dagger) - \hgamma_t(\pi^\dagger)  \Big| \geq  \frac{\xi}{ 8T  }\cdot 
\sqrt{  \sum^T_{t=1} \hat{U}_t(\dpi) } \sqrt{  \log \Big(56\sum^T_{t=1} \hat{U}_t(\dpi) } \Bigggiven \{X_t\}_{t=1}^T\Bigg) \notag \\ 
& \qquad \leq \PP\Bigg( \sup_{\pi\in\Pi}  \frac{ \big|\sum^T_{t=1} \hgamma'_t(\pi ) - \hgamma_t(\pi ) \big|}{ \sqrt{\sum^T_{t=1}{\hat{U}_t(\pi)}} \cdot \sqrt{\log \big(56  \sum^T_{t=1}{\hat{U}_t(\pi)} \big)} }
\geq \frac{\xi}{8}  \Bigggiven \{X_t\}_{t=1}^T  \Bigg).
\#

Next, note that the quantities $\hat\Gamma_t(\pi)$, $\hat\Gamma_t'(\pi)$, 
$\vp(\pi)$, and $\hat{U}_t(\pi)$ only depend on $\pi\in\Pi$ 
through the vector $\pi(X_{1:T}):=(\pi(X_1),\dots,\pi(X_T))$. 
For clarity we use the notation 
$\hat\Gamma_t(\pi(X_{1:T}))$, $\hat\Gamma_t'(\pi(X_{1:T}))$, 
$\vp(\pi(X_{1:T}))$, and $\hat{U}_t(\pi(X_{1:T}))$ interchangeably 
to emphasize such dependence. Recall that 
\$
\Pi(X) = \big\{ (\pi(X_1),\dots,\pi(X_T))\in \cA^{T}\colon \pi\in \Pi   \big\}.
\$
We then have 
\#\label{eq:nata_union}
\eqref{eq:joint_bd}
&= \PP\Bigg( \sup_{a_{1:T}\in\Pi(X)} \frac{\big|\sum^T_{t=1}  \hgamma'_t(a_{1:T} ) - \hgamma_t(a_{1:T}) \big|}{\sqrt{\sum^T_{t=1} \hat{U}_t(a_{1:T})}\cdot \sqrt{\log\big(56\sum^T_{t=1} \hat{U}_t(a_{1:T})\big)}}
\geq \frac{\xi}{8} \Bigggiven \{X_t\}_{t=1}^T  \Bigg) \notag \\ 
&\leq \sum_{a_{1:T}\in \Pi(X)}\PP\Bigg( \frac{\big|\sum^T_{t=1}  \hgamma'_t(a_{1:T} ) - \hgamma_t(a_{1:T}) \big|}{\sqrt{\sum^T_{t=1} \hat{U}_t(a_{1:T})}\cdot \sqrt{\log\big(56\sum^T_{t=1} \hat{U}_t(a_{1:T})\big)}}
\geq \frac{\xi}{8}  \Bigggiven \{X_t\}_{t=1}^T  \Bigg) \notag \\ 
&\leq \big|\Pi(X)\big| \cdot \sup_{a_{1:T}\in \Pi(X)}\PP\Bigg( \frac{\big|\sum^T_{t=1}  \hgamma'_t(a_{1:T} ) - \hgamma_t(a_{1:T}) \big|}{\sqrt{\sum^T_{t=1} \hat{U}_t(a_{1:T})}\cdot \sqrt{\log\big(56\sum^T_{t=1} \hat{U}_t(a_{1:T})\big)}}
\geq \frac{\xi}{8} \Bigggiven \{X_t\}_{t=1}^T  \Bigg) \notag \\ 
&\leq \big|\Pi(X)\big| \cdot \sup_{\pi\in\Pi}~ \PP\Bigg(  \frac{ \big|\sum^T_{t=1} 
  \hgamma'_t(\pi  ) - \hgamma_t(\pi )\big|}{    \sqrt{ \sum_{t=1}^T \hat{U}_t(\pi ) } \cdot\sqrt{\log\big(56\sum_{t=1}^T \hat{U}_t(\pi )\big)} } 
\geq \frac{\xi}{8}  \Bigggiven \{X_t\}_{t=1}^T  \Bigg),
\#
where the second line uses a union bound, 
and 
the last line is because 
every $a_{1:T}\in \Pi(X)$
corresponds to a fixed choice of $\pi\in\Pi$. 
This completes the proof of Lemma~\ref{lem:adapt_1}.} 
\end{proof}

\subsection{Proof of Lemma~\ref{lem:adapt_2}}
\label{app:lem_adapt_2}
\begin{proof}[Proof of Lemma~\ref{lem:adapt_2}] 
Fix any $\pi\in\Pi$. 
Since we always condition on $\{X_t\}_{t=1}^T$ 
and have fixed $\pi$, only for this proof, 
we suppress the dependence on $\pi$ and $X_t$, and write 
$\hat\Gamma_t$, $\hat\Gamma_t'$, $\hat{U}_t$, 
$\gamma_t( z)$, $\gamma_t'( z)$, and $\hat{u}_t( z)$ instead 
of the notations used in Section~\ref{subsec:proof_upper_adapt}. 

First, recall that $\cH_T = \{X_t,A_t,Y_t\}_{t=1}^{T-1}$, 
so $\cH_T\cup\{X_T\} \supseteq \{X_t\}_{t=1}^{T}$.   
By the tower property, we have 
\#\label{eq:peeling_first}
\notag &\PP\Bigg(    \frac{\big|\sum^T_{t=1} \{ \hgamma'_t  - \hgamma_t  \} \big|}{ 
  \sqrt{ \sum_{t=1}^T \hat{U}_t  } \cdot\sqrt{\log\big(56\sum_{t=1}^T \hat{U}_t \big)} } 
\geq \frac{\xi}{8}  
\Bigggiven \{X_t\}_{t=1}^T\Bigg) \notag \\ 
\notag &= \EE\Bigg[   \PP\bigg(    
  \frac{\big|\sum_{t=1}^{T-1}  \{ \hgamma'_t- \hgamma_t \} + \{ \hgamma'_T- \hgamma_T\}\big|}
{ \sqrt{ \sum^{T-1}_{t=1} \hat{U}_t + \hat{U}_T } \cdot \sqrt{\log\big(56 [ \sum^{T-1}_{t=1} \hat{U}_t + \hat{U}_T] \big)} }\geq    \frac{\xi}{8}   \bigggiven \cH_T \cup\{X_T\}\bigg) \Bigggiven \{X_t\}_{t=1}^T \Bigg] \notag \\ 
&= \EE\Bigg[   \PP\bigg(    
  \frac{\big|\sum_{t=1}^{T-1}  \{ \hgamma'_t - \hgamma_t \} + \eps_T\cdot\{ \hgamma'_T- \hgamma_T\}\big|}
{ \sqrt{ \sum^{T-1}_{t=1} \hat{U}_t + \hat{U}_T } \cdot \sqrt{\log\big(56 [ \sum^{T-1}_{t=1} \hat{U}_t + \hat{U}_T] \big)}   }  \geq  \frac{\xi}{8}  \bigggiven \cH_T\cup\{X_T\}\bigg) \Bigggiven \{X_t\}_{t=1}^T \Bigg] . 
\#
To see why the last line is true, simply note that 
conditional on $\{X_t,A_t,Y_t\}_{t=1}^{T-1}\cup\{X_T\}$, 
the two random vectors have the same distribution: 
\$
\big( \varepsilon_T \cdot \{ \hgamma'_T(\pi ) - \hgamma_T(\pi ) \}, ~ \hat{U}_T(\pi) \big) 
\stackrel{\textrm{d}}{=}\big(   \{ \hgamma'_T(\pi ) - \hgamma_T(\pi ) \}, ~ \hat{U}_T(\pi) \big) ~\Biggiven~ \{X_t,A_t,Y_t\}_{t=1}^{T-1}\cup\{X_T\}.
\$
The above is because 
conditional on $\{X_t,A_t,Y_t\}_{t=1}^{T-1}\cup\{X_T\}$, 
$(A_T,Y_T)$ and $(A_T',Y_T')$ are 
independent and have the same distribution by construction, 
and  that $\hat{U}_T= (\ind \{A_t = \pi(X_t)\}+\ind\{A_t'=\pi(X_t)\})^2 /
e_t(X_t,\pi(X_t)\given \cH_t)^2 $ is invariant to the permutation of $A_t$ and $A_t'$. {Hereafter for notational simplicity, we define 
\$ 
g(z) = \sqrt{z \cdot \log(56z)}.
\$

\noindent\textit{Symmetrization for the $t = T$.}
In the following, we are to apply the symmetrization trick 
to turn~\eqref{eq:peeling_first} into 
a tail probability over $\eps_T$, where
the symmetrization argument is more delicate than 
in the i.i.d.~case. 
Applying the tower property again to the inner probability in~\eqref{eq:peeling_first}, 
we have
\#\label{eq:peeling_first_2}
&\PP\bigg(    
  \frac{\big|\sum_{t=1}^{T-1}  \{ \hgamma'_t - \hgamma_t \} + \eps_T\cdot\{ \hgamma'_T- \hgamma_T\}\big|}
{ \sqrt{ \sum^{T-1}_{t=1} \hat{U}_t + \hat{U}_T } \cdot \sqrt{\log\big(56 [ \sum^{T-1}_{t=1} \hat{U}_t + \hat{U}_T] \big)} }  \geq  \frac{\xi}{8}  \bigggiven \cH_T\cup\{X_T\}\bigg) \notag \\
&
= \EE\big[  P_T(\cH_T, X_T, A_T,Y_T,A_T',Y_T',e_T ,\hat\mu_T).
\biggiven \cH_T\cup\{X_T\}\big] 
\#
Here, we define  
\$
& P_T( \cH_T, X_T, A_T,Y_T,A_T',Y_T',e_T ,\hat\mu_T) \\ 
& :=
   \PP_\eps\Bigg(    
  \frac{\big|\sum_{t=1}^{T-1}  \{ \hgamma'_t - \hgamma_t \} + \eps_T\cdot\{ \hgamma'_T- \hgamma_T\}\big|}
{ g \big( \sum^{T-1}_{t=1} \hat{U}_t + \hat{U}_T \big) }  \geq  \frac{\xi}{8}   
\Bigg) ,
\$
where $\PP_\eps$ denotes the probability over $\eps_T\sim \textrm{Unif}\{\pm 1\}$ 
while treating everything else as fixed values;
$e_T=e_T(\cdot,\cdot\given \cH_T)$ is the sampling rule at time $T$ 
and $\hat\mu_T$ is the estimator for $\mu$ 
that both depend  on $\cH_T$. 
More specifically, by the distribution of $\eps_T$, 
\$
&P_T( \cH_T, X_T, A_T,Y_T,A_T',Y_T',e_T ,\hat\mu_T) \\ 
&= \frac{1}{2} \cdot \ind \bigg\{ {\textstyle \frac{ |\sum_{t=1}^{T-1}  \{ \hgamma'_t - \hgamma_t \} + \{ \hgamma'_T- \hgamma_T\} |}
{ g( \sum^{T-1}_{t=1} \hat{U}_t + \hat{U}_T ) } }  \geq  \frac{\xi}{8} \bigg\}   + \frac{1}{2}\cdot \ind \bigg\{ {\textstyle \frac{ |\sum_{t=1}^{T-1}  \{ \hgamma'_t - \hgamma_t \} - \{ \hgamma'_T- \hgamma_T\} |}
  { g( \sum^{T-1}_{t=1} \hat{U}_t + \hat{U}_T ) } }  \geq  \frac{\xi}{8}   \bigg\},
\$
which is a random variable with respect to $\cH_T$ and $X_T,A_T,Y_T,A_T',Y_T',e_T ,\hat\mu_T$. 

Taking supremum over all possible realization of 
$A_T,A_T',Y_T,Y_T'$ and $e_T(\cdot,\cdot), \hat\mu_T(\cdot,\cdot)$ 
in the potentially larger set $\cZ_T := \cZ$, we have 
\$
&\PP\bigg(    
  \frac{\big|\sum_{t=1}^{T-1}  \{ \hgamma'_t - \hgamma_t \} + \eps_T\cdot\{ \hgamma'_T- \hgamma_T\}\big|}
{ g\big( \sum^{T-1}_{t=1} \hat{U}_t + \hat{U}_T  \big)  }  \geq  \frac{\xi}{8}   \bigggiven \cH_T\cup\{X_T\}\bigg) \\
&\leq 
\sup_{z_T\in \cZ_T} P_T(\cH_T, X_T, a_T,y_T,a_T',y_T',f_T , h_T) 
\Biggiven \cH_T\cup\{X_T\},
\$
where $z_T=(a_T,y_T,a_T',y_T',f_T,h_T)\in \cZ_T=\cZ$ is 
any realization of $A_T,A_T',Y_T,Y_T'$ and $e_T(\cdot,\cdot), \hat\mu_T(\cdot,\cdot)$. 
Now  we recall the definitions (suppressing the dependence on $\pi$ and $\{X_t\}_{t=1}^T$)
\$
\hat\gamma_T(  z_T) &= h_T(X_t,\pi(X_T) ) + \frac{\ind\{a_T=\pi(X_T)\}}{f_T(X_T,\pi(X_T) )} \cdot\big(y_T-h_T(X_T,\pi(X_T) )\big), \\
\hat\gamma_t'( z_T) &=h_T(X_T,\pi(X_T) ) + \frac{\ind\{a_T'=\pi(X_T)\}}{f_T(X_T,\pi(X_T) )} \cdot\big(y_T'-h_T(X_T,\pi(X_T) )\big), \\  
\hat{u}_T(  z_T) &=   \frac{(\ind \{a_T = \pi(X_T)\}+\ind\{a_T'=\pi(X_T)\})^2}{f_T(X_T,\pi(X_T) )^2}.
\$ 
Returning to~\eqref{eq:peeling_first} and~\eqref{eq:peeling_first_2} and recalling the definition of $P_T$, we know that 
\#\label{eq:peeling_T} 
\eqref{eq:peeling_first}
& \leq \EE\Bigg[ \sup_{z_T \in \cZ_T}\PP_{\varepsilon_T}\Bigg(    
  \frac{\big|\sum_{t=1}^{T-1}  \{ \hgamma'_t  - \hgamma_t  \} 
    + \varepsilon_T \cdot \{ \hat\gamma'_T(  z_T ) - \hat\gamma_T(  z_T ) \}\big|}
    {   g\big(\sum^{T-1}_{t=1} \hat{U}_t  + \hat{u}_T(  z_T)\big) }\geq \frac{\xi}{8}  
  \Bigg)\Bigggiven \{X_t\}_{t=1}^T \Bigg] . 
\#
where $\PP_\eps$ denotes the probability over $\eps_T$ 
while conditioning on everything else. Hence, for any $z_T\in \cZ_T$, 
the inner probability in~\eqref{eq:peeling_T} depends on $\{X_t,A_t,Y_t,A_t',Y_t'\}_{t=1}^{T-1}\cup\{X_T \}$. 
So far, we have turned~\eqref{eq:peeling_first} 
into an expectation of the tail probability for $\eps_T$. 
We will continue to do this for $t=T-1$. 


\noindent\textit{Recursive symmetrization for $t = T-1$.} 
We now continue the recursive symmetrization trick for $t=T-1$, 
and obtain an upper bound for  
\#\label{eq:recur_T-1}
\sup_{z_T\in \cZ_T}\PP_{\varepsilon_T}\Bigg(    
  \frac{\big|\sum_{t=1}^{T-1}  \{ \hgamma'_t  - \hgamma_t  \} 
    + \varepsilon_T \cdot \{ \hat\gamma'_T(  z_T ) - \hat\gamma_T(  z_T ) \}\big|}
    {  g\big(\sum^{T-1}_{t=1} \hat{U}_t  + \hat{u}_T(  z_T)\big) }\geq \frac{\xi}{8}  
  \Bigg),
\#
which 
is a random variable  determined by $\{ \hgamma'_t  - \hgamma_t\}_{t=1}^{T-1}\cup\{\hat{U}_t\}_{t=1}^{T-1}\cup\{X_T\}$.
For convenience we will write 
\$
D_T(\eps_T,z_T) := \varepsilon_T \cdot \{ \hat\gamma'_T( z_T )
- \hat\gamma_T( z_T ) \} .
\$ 
Similar to the previous step, we know by the construction of $(A_{T-1}',Y_{T-1}')$ that 
\$
\big( \varepsilon_{T-1} \cdot \{ \hgamma'_{T-1}  - \hgamma_{T-1} \}, ~ \hat{U}_{T-1} \big) 
\stackrel{\textrm{d}}{=}\big(   \{ \hgamma'_{T-1} - \hgamma_{T-1}  \}, ~ \hat{U}_{T-1}  \big) 
~\Biggiven~ \cH_{T-1}\cup\{X_{T-1},X_T\}.
\$
Since~\eqref{eq:recur_T-1} is a random variable 
that can be viewed as a function of $\{ \hgamma'_t  - \hgamma_t\}_{t=1}^{T-1}\cup\{\hat{U}_t\}_{t=1}^{T-1}\cup\{X_T\}$, the above equivalence in distribution 
leads to 
\#
\eqref{eq:recur_T-1} ~\stackrel{\textrm{d.}}{=}~
\sup_{z_T\in \cZ_T}\PP_{\varepsilon_T}\Bigg(    
  \frac{\big|\sum_{t=1}^{T-2}  \{ \hgamma'_t  - \hgamma_t  \} 
   + \eps_{T-1} \cdot (\hgamma'_{T-1}-\hgamma_{T-1}) +  D_T(\eps_T,z_T) \big|}
    { g\big(  \sum^{T-2}_{t=1} \hat{U}_t + \hat{U}_{T-1} + \hat{u}_T(  z_T)\big)}\geq \frac{\xi}{8}  
  \Bigg),
\#
where the above two distributions 
are both conditional on $\{X_t,A_t,Y_t\}_{t=1}^{T-2}\cup\{X_{T-1},X_T\}$.  
Define the $\sigma$-field $\cF_{T-1}:= \sigma(\{X_t,A_t,Y_t\}_{t=1}^{T-2}\cup\{X_{T-1},X_T\})$. 
We have 
\#\label{eq:peel_T-1}
& \EE\Bigg[ 
\sup_{z_T \in \cZ_T}\PP_{\varepsilon_T}\Bigg(    
  \frac{\big|\sum_{t=1}^{T-1}  \{ \hgamma'_t  - \hgamma_t  \} 
    +   D_T( \eps_T,z_T ) \big|}
    {  g\big(  \sum^{T-2}_{t=1} \hat{U}_t + \hat{U}_{T-1} + \hat{u}_T(  z_T)\big) }\geq \frac{\xi}{8}  
  \Bigg) \Bigggiven \cF_{T-1}\Bigg] \notag \\ 
&= \EE\Bigg[ 
  \sup_{z_T \in \cZ_T}\PP_{\varepsilon_T}\Bigg(    
    \frac{\big|\sum_{t=1}^{T-2}  \{ \hgamma'_t  - \hgamma_t  \} 
    + \varepsilon_{T-1} \cdot  (\hgamma'_{T-1}-\hgamma_{T-1}) 
      +   D_T( \eps_T,z_T ) \big|}
      { g\big(  \sum^{T-2}_{t=1} \hat{U}_t + \hat{U}_{T-1} + \hat{u}_T(  z_T)\big) }\geq \frac{\xi}{8}  
    \Bigg) \Bigggiven \cF_{T-1}\Bigg] \notag \\ 
&=\EE\Bigg[ 
 \EE_{\eps_{T-1}}\Bigg[  \sup_{z_T \in \cZ_T}\PP_{\varepsilon_T}\Bigg(    
    \frac{\big|\sum_{t=1}^{T-2}  \{ \hgamma'_t  - \hgamma_t  \} 
    + \varepsilon_{T-1} \cdot  (\hgamma'_{T-1}-\hgamma_{T-1}) 
      +   D_T( \eps_T,z_T ) \big|}
      {  g\big(  \sum^{T-2}_{t=1} \hat{U}_t + \hat{U}_{T-1} + \hat{u}_T(  z_T)\big) }\geq \frac{\xi}{8}  
    \Bigg) \Bigg] \Bigggiven \cF_{T-1}\Bigg],
\#
where the second line uses the equality in conditional distribution, 
and the third line uses the tower property with $\EE_{\eps_{T-1}}$ denoting
the expectation over $\eps_{T-1}$ while conditioning on everything else. 
In particular, we can rewrite the inner expectation as 
\#\label{eq:peel_T-1_2}
&\EE_{\eps_{T-1}}\Bigg[  \sup_{z_T \in \cZ_T}\PP_{\varepsilon_T}\Bigg(    
  \frac{\big|\sum_{t=1}^{T-2}  \{ \hgamma'_t  - \hgamma_t  \} 
  + \varepsilon_{T-1} \cdot  (\hgamma'_{T-1}-\hgamma_{T-1}) 
    +   D_T( \eps_T,z_T ) \big|}
    {  g\big(  \sum^{T-2}_{t=1} \hat{U}_t + \hat{U}_{T-1} + \hat{u}_T(  z_T)\big) }\geq \frac{\xi}{8}  
  \Bigg) \Bigg] \\
&= P_{T-1}\big(\cH_{T-2}, X_{(T-1):T}, A_{T-1},Y_{T-1},A_{T-1}',Y_{T-1}',e_{T-1},\hat\mu_{T-1}\big), \notag
\#
where $P_{T-1}(\cdots)$ is a fixed function, 
and  $e_{T-1}$ and $\hat\mu_{T-1}$ are the sampling propensity 
and regression estimator at time $T-1$. Therefore,
we can write 
\#\label{eq:peel_T-1_3}
\eqref{eq:peel_T-1_2} 
&= \EE\Big[ P_{T-1}\big(\cH_{T-2}, X_{(T-1):T}, A_{T-1},Y_{T-1},A_{T-1}',Y_{T-1}',e_{T-1},\hat\mu_{T-1}\big) \Biggiven \cF_{T-1}\Big] \\ 
&\leq \EE\bigg[ \sup_{z_{T-1}\in \cZ_{T-1}} P_{T-1}\big(\cH_{T-2}, X_{(T-1):T}, a_{T-1},y_{T-1},a_{T-1}',y_{T-1}',f_{T-1},h_{T-1}\big) \bigggiven \cF_{T-1}\bigg]
\#
where $z_{T-1}=(a_{T-1},y_{T-1},a_{T-1}',y_{T-1}',f_{T-1},h_{T-1}) \in \cZ_{T-1}:=\cZ$ 
denotes any possible realization of $(A_{T-1},Y_{T-1},A_{T-1}',Y_{T-1}',e_{T-1},\hat\mu_{T-1})$. 
By the definition of $P_{T-1}$,~\eqref{eq:peel_T-1_3} is upper bounded by 
\$ 
\sup_{z_{T-1}\in \cZ_{T-1}}\EE_{\eps_{T-1}}\Bigg[  \sup_{z_T \in \cZ_T}\PP_{\varepsilon_T}\Bigg(    
  \frac{\big|\sum_{t=1}^{T-2}  \{ \hgamma'_t  - \hgamma_t  \} 
  +  D_{T-1}(\eps_{T-1},z_{T-1})
    + D_T( \eps_T,z_T ) \big|}
    {  g\big(  \sum^{T-2}_{t=1} \hat{U}_t + \hat{u}_{T-1}( z_{T-1}) + \hat{u}_T(  z_T)\big) }\geq \frac{\xi}{8}  
  \Bigg) \Bigg],
\$
where for any $z_{T-1}=(a_{T-1},y_{T-1},a_{T-1}',y_{T-1}',f_{T-1},h_{T-1})\in \cZ_{T-1}:=\cZ$, we write 
\$
D_{T-1}(\eps_{T-1},z_{T-1}) &= \varepsilon_{T-1} \cdot \{ \hat\gamma'_{T-1}( z_{T-1} )
- \hat\gamma_{T-1}( z_{T-1} ) \}, \\ 
\textrm{and}\quad \hat{u}_{T-1}(z_{T-1}) &=  \frac{(\ind \{a_{T-1}  = \pi(X_{T-1})\}+\ind\{a_{T-1}'=\pi(X_{T-1})\})^2}{f_{T-1}(X_{T-1},\pi(X_{T-1}) )^2}.
\$
Returning to~\eqref{eq:peeling_T}, we have $\eqref{eq:peeling_T}$ upped bounded by
\#\label{eq:peeling_T_2}
\EE\Bigg[ \sup_{z_{T-1}\in \cZ_{T-1}}\EE_{\eps_{T-1}}\Bigg[  \sup_{z_T \in \cZ_T}\PP_{\varepsilon_T}\Bigg(    
  \frac{\big|\sum_{t=1}^{T-2}  \{ \hgamma'_t  - \hgamma_t  \} 
  +  D_{T-1}(\eps_{T-1},z_{T-1})
    +  D_T( \eps_T,z_T ) \big|}
    {  g\big(  \sum^{T-2}_{t=1} \hat{U}_t + \hat{U}_{T-1} + \hat{u}_T(  z_T)\big)}\geq \frac{\xi}{8}  
  \Bigg) \Bigg]  \Bigggiven \{X_t\}_{t=1}^T\Bigg] ,
\#
which completes the recursive symmetrization for $t=T-1$. 

\noindent\textit{Continuing the recursive symmetrization.} 
Using similar notations $D_t(\eps_t,z_t)$, $\hat{u}_t(z_t)$ for $z_t\in \cZ_t:=\cZ$, 
we can continue the above recursive symmetrization 
for $t=T-2,T-3,\dots,1$. 
This gives  
\# \label{eq:symmetrized}
\eqref{eq:peeling_T} &\leq 
\sup_{z_1\in \cZ_1} \EE_{\varepsilon_1} \sup_{z_2\in\cZ_2} \EE_{\varepsilon_2} 
\cdots \sup_{z_T\in \cZ_T} \EE_{\varepsilon_T} \ind \Bigg\{  
\sup_{\pi\in\Pi} \bigg|  \frac{ \sum_{t=1}^{T} D_t(\eps_t,z_t)  
}{ g \big(  \sum^{T}_{t=1} \hat{u}_t( z_t) )}  \bigg|  
\geq \frac{\xi}{8}   \Bigg\},
\#
where $\EE_{\eps_t}$ denotes the expectation over the 
randomness in $\eps_t$ while conditioning on everything else. 

We note that in~\eqref{eq:symmetrized}, 
for every $t>1$, the choice of $z_t\in \cZ_t$ that attains the supremum 
depends on the realization of $\eps_1,\dots,\eps_{t-1}$. 
Now let $z_1^*\in \cZ_1$ denote the (fixed) element in $\cZ_1$ that attains 
the supremum in~\eqref{eq:symmetrized}. 
It is fixed since we have conditioned on $\{X_t\}_{t=1}^T$ and 
fixed some $\pi$. 
Then  
the element that attains the second supremum $z_2\in \cZ_2$ depends on $\eps_1$; 
we define it as $z_2^*(\eps_1)\in\cZ_2$ to emphasize such dependence. 
Further, the element $z_3\in \cZ_3$ that achieves  the third supremum 
depends on $\eps_1,\eps_2$, and we denote it by $z_3^*(\eps_{1:2})$, etc. 
So on and so forth, for every $t\in[T]$, 
we can write $z_{t}^*(\eps_{1:(t-1)})$ that attain the supremum $z_t\in \cZ_t$ 
which depends on $\eps_1,\dots,\eps_{t-1}$. 
That is, we can use a $\cZ$-valued tree (c.f.~Definition~\ref{defn:tree}) 
to represent the $2^T-1$ elements 
$z_1^*,z_2^*(\pm 1), z_3^*(\{\pm 1\}^2),\dots,z_T^*(\{\pm 1\}^{T-1})$. 
Figure~\ref{fig:tree_opt} is a visualization.

\begin{figure} 
\centering  
\begin{tikzpicture}[->,>=stealth', thick, main node/.style={circle,draw}]
\node[circle, text=black, circle, draw=black, fill=black!5, scale=1.7, label=right:{$z_1^*$}] (1) at  (0,0) { };
\node[circle, text=black, circle, draw=black, fill=black!5, scale=1.7, label=left:{$z_2^*(-1)$}] (2) at  (-1.5,-1.5)  { }; 
\node[circle, text=black, circle, draw=black, fill=black!5, scale=1.7, label=right:{ $z_2^*(1)$}] (3) at  (1.5,-1.5)  { }; 

\node[circle, text=black, draw=black, fill=black!5, scale=1.7, label=below:{$z_3^*(-1,-1)$}] (4) at  (-3,-3) { }; 
\node[circle, text=black, draw=black, fill=black!5, scale=1.7, label=below:{$z_3^*(-1,1)$}] (5) at  (-1,-3) { }; 
\node[circle, text=black, draw=black, fill=black!5, scale=1.7, label=below:{$z_3^*(1,-1)$}] (6) at  (1,-3) { }; 
\node[circle, text=black, draw=black, fill=black!5, scale=1.7, label=below:{$z_3^*(1,1)$}] (7) at  (3,-3) { }; 

\draw[-] (1) edge [draw=black] (2);
\draw[-] (1) edge [draw=black] (3); 
\draw[-] (2) edge [draw=black] (4); 
\draw[-] (2) edge [draw=black] (5); 
\draw[-] (3) edge [draw=black] (6); 
\draw[-] (3) edge [draw=black] (7); 

\coordinate (a) at (-0.2,-0.4);
\coordinate (b) at (-1.1,-1.35);
\coordinate (c) at (-1.2,-1.7);
\coordinate (d) at (-0.9,-2.65);

\draw[->,magenta,dashed,very thick] (a) -- (b) node [] {};
\draw[->,magenta,dashed,very thick] (c) -- (d) node [] {};
\end{tikzpicture}  
\caption{The optimal $\cZ$-valued tree of depth $T=3$. 
The pink path shows a realization $\eps_1=-1$ and $\eps_2=1$.} 
\label{fig:tree_opt}
\end{figure}
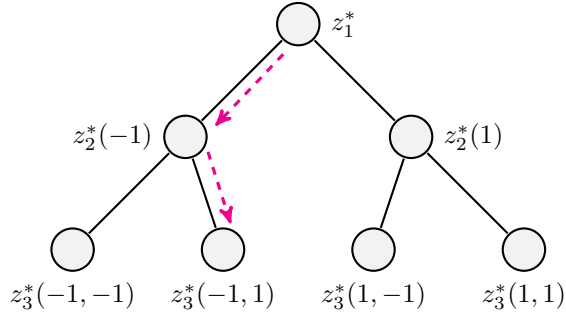

Recall that $\mathbb{T}_T(\cZ)$ is the set of all 
$\cZ$-valued trees of depth $T$, and 
the optimal tree 
\$
\bz^*=(z_1^*,z_2^*(\pm 1), z_3^*(\{\pm 1\}^2),\dots,z_T^*(\{\pm 1\}^{T-1})) \in \mathbb{T}_T(\cZ),
\$ 
as illustrated in Figure~\ref{fig:tree_opt}. 
With this notation, we can rewrite~\eqref{eq:symmetrized}
as
\#\label{eq:symmetrized_2}
\eqref{eq:symmetrized}
= \PP_{\eps}\Bigg( 
\bigg|\frac{\sum^T_{t=1} \eps_t \cdot
\big\{\hat{\gamma}_t(\pi; X_t,z^*_t(\eps_{1:(t-1)})) - \hat{\gamma}_t'(\pi;X_t,z^*_t(\eps_{1:(t-1)}))\big\}}
{  g\big( \sum^T_{t=1} \hat{u}_t(\pi;X_t,z_t^*(\eps_{1:(t-1)}))\big)}\bigg|
\ge \frac{\xi}{8}  \cdot \afactor \Bigg),
\#
where $\PP_{\eps}$ is over the randomness of $\eps_{1:T}$ 
treating everything else as fixed, and we have returned 
to the original notations for $\hat\gamma_t,\hat\gamma_t'$ and $\hat{u}_t$. 
We emphasize that, 
above, the only randomness in $z_t^*(\eps_{1:(t-1)})$ 
is in the realization of $\eps_{1:(t-1)}$, 
while the image set $z_t^*(\{\pm 1\}^{t-1})$ 
is fixed for all $t\in[T]$. Finally, taking the supremum over
all $\bz \in \mathbb{T}_T(\cZ)$
concludes our proof.}
\end{proof}

\subsection{Proof of Lemma~\ref{lem:adapt_3}}
\label{app:lem_adapt_3}

\begin{proof}[Proof of Lemma~\ref{lem:adapt_3}] 
As before, we condition on $\{X_t\}_{t=1}^T$ throughout
the proof.
For any $\bz = \{z_t(\{\pm 1\}^{t-1})\}_{t=1}^T \in \mathbb{T}_T(\cZ)$,
we define two processes
\$
&S_t \defn \sum^t_{s=1} \eps_s\cdot 
\big\{\hat{\gamma}_s \big(\pi;X_s,z_s(\eps_{1:(s-1)})\big) 
- \hat{\gamma}'_s\big(\pi;X_s,z_s(\eps_{1:(s-1)})\big)\big\}, \\ 
&M_t \defn \sum^t_{s=1} 4 \cdot \hat{u}_s\big(\pi;X_s,z_s(\eps_{1:(s-1)})\big), \quad t\in[T],
\$
and $S_0=M_0=0$. 
Next, for any $\lambda>0$, we define
\$
L_t(\lambda) \defn \exp\Big(\lambda S_t -\frac{\lambda^2}{2} M_t\Big),\quad t\in[T]. 
\$
It is clear that $\{S_t\}$ and $\{M_t\}$ are 
adapted to the filtration $\{\cG_t\}$ where 
we define $\cG_t \defn\sigma(\eps_{1:t})$. 
Also, we can check that for any $\lambda > 0$, $L_T(\lambda)$ 
is a super-martingale with respect to $\{\cG_T\}$
with $L_0(\lambda) = 1$.
To see this, we note that 
\#
&\EE[L_T(\lambda) \given \cG_{T-1}]
\notag \\ 
&= L_{T-1}(\lambda) \cdot \exp\Big(-\frac{\lambda^2}{2}
\cdot 4 \hat{u}_T\big(\pi;X_T,z_T(\eps_{1:(T-1)})\big)\Big) \notag \\
\notag &  
\quad \times \EE\bigg[\exp\Big(\lambda \eps_T \cdot \big\{\hat{\gamma}_T \big(\pi;X_T,z_T(\eps_{1:(T-1)})\big) 
- \hat{\gamma}'_T\big(\pi;X_T,z_T(\eps_{1:(T-1)})\big)\big\} \Big)\bigggiven \cG_{T-1}\bigg]\\
\notag & \leq  L_{T-1}(\lambda) \cdot \exp\Bigg(-\frac{\lambda^2}{2}
\cdot 4\hat{u}_T\big(\pi;X_T,z_T(\eps_{1:(T-1)})\big)\\
& \notag \qquad \qquad \qquad \qquad 
+ \frac{\lambda^2}{2} \cdot \Big\{\hat{\gamma}_T \big(\pi;X_T,z_T(\eps_{1:(T-1)})\big) 
- \hat{\gamma}'_T\big(\pi;X_T,z_T(\eps_{1:(T-1)})\big)\Big\}^2 \Bigg),
\#
where the last step is because $(e^t + e^{-t})/2 \le e^{t^2/2}$ for any $t \in \RR$.
Meanwhile, 
writing 
$z_{T}(\eps_{1:(T-1)}) = (a_T,a_T', y_T,y_T',f_T,h_T)$, 
it 
holds deterministically that 
\$
& \big\{\hat{\gamma}_T \big(\pi;X_T,z_T(\eps_{1:(T-1)})\big) 
- \hat{\gamma}'_T\big(\pi;X_T,z_T(\eps_{1:(T-1)})\big)\big\}^2\\
&= \bigg\{ \frac{y_T\ind\{a_T = \pi(X_T)\}}
{f_T(X_T,\pi(X_T))} -  \frac{y_T'\ind\{a_T' = \pi(X_T)\}}
{f_T(X_T,\pi(X_T))} \\
& \qquad \qquad \qquad + h_T(X_T,\pi(X_T)) \cdot
\Big(\frac{\ind\{a_T = \pi(X_T)\}}
{f_T(X_T,\pi(X_T))} - \frac{\ind\{a_T' = \pi(X_T)\}}
{f_T(X_T,\pi(X_T))}\Big)\bigg\}^2 \\
&\leq  2 \cdot \bigg\{y_T\cdot \frac{\ind\{a_T = \pi(X_T)\}}
{f_T(X_T,\pi(X_T))} -y_T'\cdot \frac{\ind\{a_T' = \pi(X_T)\}}
{f_T(X_T,\pi(X_T))} \bigg\}^2
+ 2 \hat{u}_T(\pi;X_T,z_{T}(\eps_{1:(T-1)})) \\ 
& \le 4\hat{u}_T(\pi;X_T,z_{T}(\eps_{1:(T-1)})),
\$
where the first inequality uses the fact that 
$(a+b)^2\leq 2a^2+2b^2$ for any $a,b\in\RR$. 
Collectively, we have 
\$
\EE\big[L_T(\lambda) \biggiven \cG_{T-1}\big] \le L_{T-1}(\lambda),
\$
and hence $L_T(\lambda)$ is a super-martingale with 
\$
\EE\big[L_T(\lambda)\big] = \EE\big[\exp\big(\lambda A-\lambda^2B^2/2\big)\big] \leq 1, 
\quad \textrm{where }A = S_T,~B= M_T^{1/2}.
\$
{Now invoking Lemma~\ref{lemma:self_normal} for $A,B$ 
and taking $y=1$, we know 
that for any $x\geq \sqrt{2}$, 
\#\label{eq:tail_bd}
\exp(-x^2/2)&\geq \PP\Bigg( \frac{|A|}{\sqrt{(B^2+1)\big( 1+ 0.5\log(B^2+1)  \big)}}  \geq x\Bigg) \notag \\ 
&\geq \PP\Bigg( \frac{|A|}{\sqrt{2B^2 \cdot \big( 1+ 0.5\log(2B^2)  \big)}}  \geq x\Bigg), \notag \\
&\geq \PP\Bigg( \frac{|A|}{\sqrt{ B^2 \cdot  \log (14B^2)  \big)}}  \geq x\Bigg),
\#
where the second line
uses the fact that $B^2+1 \leq 2B^2$ since $B\geq 1$, 
and the third line uses the fact that $1+\log(2B^2)/2 = \log(2e^2B^2)/2\leq  \log(14B^2)/2 $. 
In~\eqref{eq:tail_bd}, 
plugging in the expression of $A$ and $B$
we obtain 
\#\label{eq:bd_almost_done}
\PP_\eps\bigg( \frac{|\sum^T_{t=1} \eps_t\cdot 
\{\hat{\gamma}_t (\pi;X_t,z_t^*(\eps_{1:t-1})) - \hat{\gamma}'_t(\pi;X_t,z_t^*(\eps_{1:t-1}))\}|  }
{  2 \sqrt{ \sum^T_{t=1}  \hat{u}_t (\pi;X_t,z_t^*(\eps_{1:(t-1)}))  } \sqrt{\log [56\sum^T_{t=1}  \hat{u}_t (\pi;X_t,z_t^*(\eps_{1:(t-1)}))]}   } 
\ge x \Big) \le e^{-x^2/2} 
\#
for any $x\geq \sqrt{2}$, 
which concludes the proof of Lemma~\ref{lem:adapt_3}.}
\end{proof}

\subsection{Proof of Lemma~\ref{lem:adapt_4}}
\label{app:lem_adapt_4}

\begin{proof}[Proof of Lemma~\ref{lem:adapt_4}]
Following exactly the same arguments as in the proof of 
Lemma~\ref{lemma:bounded_part}, we have that 
\$
\sup_{\pi\in\Pi} |\sum_{t=1}^T \mu(X_t,\pi(X_t))-Q(\pi)|
\leq \sqrt{T} \cdot \sqrt{8 (\ndim(\Pi)\log(TK^2)+ \log(16/\delta))}
\$ 
with probability at least $1- \delta/2$. 
Noting that $V(\pi)\geq\sqrt{T}$ deterministically completes the proof of 
Lemma~\ref{lem:adapt_4}. 
\end{proof}

\subsection{Proof of Corollary~\ref{cor:adaptive}}
\label{app:cor_adapt}

\begin{proof}[Proof of Corollary~\ref{cor:adaptive}]
From Theorem~\ref{thm:adapt_upper}, it holds with probability at least $1-\delta$ 
that $\cL(\hat\pi;\cC,\Pi) \leq 2\beta \cdot V(\pi^*)$ if we choose any 
$\beta\geq  67 \cdot (\log T)^{\alpha/2} \cdot \sqrt{\ndim(\Pi)\log(TK^2) + \log(16/\delta)}$. 
We now proceed to bound $V(\pi^*)$ by the three terms in Equation (7) separately. 

\noindent\textit{Bounding $\vs(\pi^*)$.} 
We consider martingale differences 
$D_t = \frac{\ind\{A_t=\pi^*(X_t)\}}{e_t(X_t,\pi^*(X_t)\given \cH_t)^2} - \frac{1}{e_t(X_t,\pi^*(X_t)\given \cH_t)}$ 
and the filtration $\cF_{t-1} = \sigma(\cH_t,X_t)$. Then 
$\{D_t\}$ is adapted to $\{\cF_t\}$ and $\EE[D_t\given \cF_{t-1}]=0$. 
Moreover, Assumption~\ref{assump:poly} implies 
\$
M_t := \sum_{j=1}^t \EE[D_j^2\given \cF_{j-1}] \le \sum_{j=1}^t \frac{1}{e_t(X_t,\pi^*(X_t)\given \cH_t)^3} 
\leq \sum_{j=1}^t g_t^{-3} \leq \frac{\bar{c}^{-3}}{1+3\gamma} (T+1)^{1+3\gamma}  
\$
for all $t\in[T]$. Also, $D_t\leq g_t^{-2} \leq R:=T^{2\gamma}$. 
Then invoking Freedman's inequality (c.f.~Lemma~\ref{lem:bernstein}) 
with $\sigma^2 = \frac{\bar{c}^{-3}}{1+3\gamma} (T+1)^{1+3\gamma}$ and 
$R = T^{2\gamma}$, for any $x\geq 0$, we have 
\$
\PP\bigg( \sum_{t=1}^T \frac{\ind\{A_t=\pi^*(X_t)\}}{e_t(X_t,\pi^*(X_t)\given \cH_t)^2} - 
\sum_{t=1}^T \frac{1}{e_t(X_t,\pi^*(X_t)\given \cH_t)} \geq x   \bigg)
\leq \exp\bigg(  \frac{-x^2/2}{\sigma^2 + Rx/3} \bigg).
\$
Now we take $x=2\max\{1,\sqrt{\bar{c}^{-3}/(1+3\gamma)}\}\cdot (T+1)^{(1+3\gamma)/2}\cdot \log(1/\delta)$. 
Then since $(1+3\gamma)/2\geq 2 \gamma$, we know 
$x\geq 2\sigma\log(1/\delta)$ and $x\geq 2R\log(1/\delta) $. Hence 
\$
\frac{-x^2/2}{\sigma^2 + Rx/3} \leq \frac{- 2\log(1/\delta)^2}{1 + 2/3\cdot \log(1/\delta) } 
\leq -\log(1/\delta).
\$
That is, with probability at least $1-\delta$, we have 
\$
\vs(\pi^*)&\leq \frac{1}{T} \Bigg\{  \sum_{t=1}^T \frac{1}{e_t(X_t,\pi^*(X_t)\given \cH_t)} + 2\max\bigg\{1,\frac{\bar{c}^{-3/2}}{\sqrt{1+3\gamma}}\bigg\}\cdot (T+1)^{(1+3\gamma)/2}\cdot \log(1/\delta) \Bigg\}^{1/2} \\
          &\leq  \frac{1}{T} \Bigg\{ \sum_{t=1}^T \bar{c}^{-1}\cdot t^{\gamma} + 2\max\bigg\{1,\frac{c^{3/2}}{\sqrt{1+3\gamma}}\bigg\}\cdot (T+1)^{(1+3\gamma)/2}\cdot \log(1/\delta)  \Bigg\}^{1/2} \\ 
&\leq \frac{1}{T} \Bigg\{ \frac{\bar{c}^{-1}\cdot (T+1)^{1+\gamma}}{1+\gamma}+ 2\max\bigg\{1,\frac{\bar{c}^{-3/2}}{\sqrt{1+3\gamma}}\bigg\}\cdot (T+1)^{(1+3\gamma)/2}\cdot \log(1/\delta)  \Bigg\}^{1/2}  \\ 
&\leq   \sqrt{\frac{2\bar{c}^{-1}  + 4\max\{1,\bar{c}^{-3/2}\} \cdot \log(1/\delta) \cdot  T^{(\gamma-1)/2}}{T^{1-\gamma}}  }.
\$
Here the fourth line uses the fact that $(T+1)^{1+\gamma}/(1+\gamma)\leq T^{1+\gamma} \cdot \frac{2^{1+\gamma}}{1+\gamma}\leq 2 T^{1+\gamma}$  
and 
\$
(T+1)^{(1+3\gamma)/2}\leq T^{(1+3\gamma)/2} \cdot \frac{2^{(1+3\gamma)/2}}{\sqrt{1+3\gamma}} \leq 
2 T^{(1+3\gamma)/2}
\$ for $\gamma\in(0,1)$. 

\noindent\textit{Bounding $\vp(\pi^*)$.} 
By the lower bound for $e_t(X_t,\pi^*(X_t)\given \cH_t)$, we have the deterministic upper bound
\$
\vp(\pi^*)\leq \frac{1}{T} \sqrt{ \sum_{t=1}^T \bar{c}^{-1}\cdot t^{\gamma} }
\leq \sqrt{ \frac{\bar{c}^{-1}}{T^2} \cdot \frac{(T+1)^{\gamma+1}}{\gamma+1}   }
\leq \sqrt{ \frac{2 }{\bar{c} \cdot T^{1-\gamma}}  }.
\$

\noindent\textit{Bounding $\vh(\pi^*)$.} 
Finally, we have the deterministic upper bound
\$
\vh(\pi^*) \leq \frac{1}{T} \bigg(  \sum_{t=1}^T \bar{c}^{-3} \cdot t^{3\gamma}  \bigg)^{1/4}
= \bigg(   \frac{1}{\bar{c}^{3}\cdot T^4} \cdot \frac{(T+1)^{1+3\gamma}}{1+3\gamma}  \bigg)^{1/4}
\leq \bigg(   \frac{2  \cdot T^{1+3\gamma} }{\bar{c}^{3}\cdot T^4}  \bigg)^{1/4}
\leq    \frac{2  }{\bar c^{3/4} \cdot T^{3(1-\gamma)/4}}  .
\$

Putting the three terms together, 
we know that with probability at least $1-\delta$, 
\$
V(\pi^*)&\leq \max\Bigg\{  \sqrt{\frac{2/\bar{c}  + 4\max\{1,\bar{c}^{-3/2}\} \cdot \log(1/\delta) \cdot  T^{(\gamma-1)/2}}{T^{1-\gamma}}} ,~  \sqrt{ \frac{2\bar{c}^{-1}}{T^{1-\gamma}}  },~ 
\frac{2\bar{c}^{-3/4}  }{T^{3(1-\gamma)/4}} \Bigg\}\\ 
&\leq T^{-(1-\gamma)/2} \cdot 
\max\Bigg\{  \sqrt{ 2/\bar{c}  + 4\max\{1,\bar c^{-3/2}\} \cdot \log(1/\delta) \cdot  T^{(\gamma-1)/2} } ,~ 
  2\bar c^{-3/4}  T^{-(1-\gamma)/4}  \Bigg\} \\ 
&\leq T^{-(1-\gamma)/2} \cdot 
\max\Big\{ 2/\sqrt{\bar c}, 4\max\{1,\bar c^{-3/4} \}\cdot \log(1/\delta)\cdot T^{-\frac{1-\gamma}{4}}  \Big\} \\ 
&\leq 4 T^{-(1-\gamma)/2} \cdot 
\max\Big\{  1/\sqrt{\bar c},  \max\{1,\bar c^{-3/4} \}\cdot \log(1/\delta)\cdot T^{-\frac{1-\gamma}{4}}  \Big\} \\ 
&\leq \frac{4 \log(1/\delta)}{ T^{(1-\gamma)/2} \cdot  \max\{1,\bar c^{ 3/4} \}}.
\$
Here we used the fact that $\sqrt{x}\leq \max\{1,x^{3/4}\}$ for $x>0$ 
and $T^{-(1-\gamma)/4}\leq 1$. 
Further applying a union bound, we know that with probability at least $1-2\delta$, 
since $TK^2\geq 8/\delta$, 
\$
\cL(\hat\pi;\cC,\Pi) &\leq 2 c 
\cdot (\log T)^{\alpha/2} \cdot \sqrt{\ndim(\Pi)\log(TK^2) + \log(8/\delta)} \cdot \frac{4 \log(1/\delta)}{ T^{(1-\gamma)/2} \cdot  \max\{1,\bar c^{ 3/4} \}} \\ 
&\leq 2c \cdot (\log T)^{\alpha/2} \cdot \sqrt{2 \ndim(\Pi)\log(TK^2)  } \cdot \frac{4 \log(1/\delta)}{ T^{(1-\gamma)/2} \cdot  \max\{1,\bar c^{ 3/4} \}}  \\ 
&\leq 12c \cdot \sqrt{ \frac{\ndim (\Pi)}{T^{1-\gamma}} } \cdot \frac{(\log (TK^2))^{(1+\alpha)/2} \log(1/\delta)}{\max\{1,\bar c^{ 3/4} \}},
\$
which concludes the proof of Corollary~\ref{cor:adaptive}. 
\end{proof}

\subsection{Proof of Theorem~\ref{thm:lower_adapt}}
\label{app:subsec_lower_adapt}

\begin{proof}[Proof of Theorem~\ref{thm:lower_adapt}]
This proof is similar to that of Theorem~\ref{thm:lower}, and also 
similar to that in~\cite{zhan2021policy}. 

Let $S=\{x_1,\dots,x_d\}\subseteq \cX$ with $d:=\ndim(\Pi)$ be the set 
that is N-shattered by $\Pi$. 
Then we know that there exists 
two functions $f_1,f_2\colon \cX\to \cA$ such that 
$f_1(x_i)\neq f_2(x_i)$ for all $i=1,\dots,d$, and 
for every subset $S_0\subseteq S$, there exists a policy $\pi\in \Pi$ 
such that $\pi(x)=f_1(x)$ for all $x\in S_0$ 
and $\pi(x)=f_2(x)$ for all $x\in S\backslash S_0$. 
We let $\cV=\{\pm 1\}^d$ be an index set of size $2^d$.

Let $\delta\in(0,1/4)$ be a constant to be decided later. 
We consider a subset of instances $\cR_0(\bar{c},\gamma,T) = \{R_v\}_{v\in \cV}\subseteq \cR(\bar{c},\gamma,T)$ 
indexed by all $v\in \cV$.  
Each element $R_v = (\cC_v,e)\in \cR_{0}(\bar{c},\gamma,T)$ 
corresponds to $\cC_v=(\rho,\mu_v,\cA)$, where  $\rho$ is the uniform distribution over $S$, 
and the reward distribution $\mu_v$ is given by 
\$
\PP(Y(a) \given X=x_i) \sim 
\begin{cases}
&\textrm{Bern}(1/2),\quad a = f_1(x_i) \\
&\textrm{Bern}(1/2 + v_i \cdot \delta),\quad a = f_2(x_i), \\ 
&\textrm{Bern}(0),\quad \textrm{otherwise}. 
\end{cases}
\quad \forall x_i\in S.
\$
Since $S$ is N-shattered by $\Pi$, 
for every $v\in \cV$, there exists 
some $\pi_v^*\in \Pi$ such that 
$\pi_v^*(x_i)=f_1(x_i)$ for all $v_i <0$ and 
$\pi_v^*(x_i)=f_2(x_i)$ for all $v_i>0$, which is the optimal 
policy for $\cC_v$ by definition. 
Since $\delta\in(0,1/4)$, similar to the proof of Theorem~\ref{thm:fix_upper}, we know 
\# 
\cL(\pi;\cC_v,\Pi)  \geq \frac{1}{d}\sum_{i=1}^d \delta\cdot \ind\{\pi(x_i) \neq \pi_v^*(x_i)\} .
\#
Also, we define the fixed behavior policy by 
\$
e_t(x_i,a) = 
\begin{cases}
& \bar{c} \cdot t^{-\gamma},\quad a = f_1(x_i) \\
&\bar{c} \cdot t^{-\gamma},\quad a = f_2(x_i), \\ 
&(1-2\bar{c} \cdot t^{-\gamma})/(K-2),\quad \textrm{otherwise}. 
\end{cases}
\quad \forall x_i\in S.
\$
For any $v\in \cV$, we let $\PP_v(\cdot)$ and $\EE_v[\cdot]$ 
denote the joint distribution of data under $(\cC_v,e)$. 
Now fix any $v\in \cV$, and for any $i\in[d]$, 
we denote $M_i[v]\in\{\pm 1\}^d$ as the vector
that only differs from $v$ in entry $i$, i.e., 
$(M_i[v])_i = -v_i$ and $(M_i[v])_j=v_j$ for all $j\neq i$. 
Then, following exactly the same arguments as in~\eqref{eq:lb_1} and~\eqref{eq:lb_2} in the 
proof of Theorem~\ref{thm:fix_upper}, we know that 
for any (data-dependent) policy $\hat\pi$, we have 
\# 
\sup_{R_v\in \cR_0(\bar{c},\gamma,T)} \EE_{v}\big[ \cL(\hat\pi;\cC_v,\Pi)  \big]
& \geq \frac{1}{2^{d}} \sum_{v\in \cV} \EE_{v}\big[  \cL(\hat\pi;\cC_v,\Pi)  \big]
\geq  \frac{1}{2^{d}} \sum_{v\in \cV} \frac{1}{d}\sum_{i=1}^d \delta \cdot \EE_{v}\big[ \ind\{ \hat\pi(x_i) \neq \pi_{v}^*(x_i)  \big] \\
&\geq \frac{\delta}{d\cdot 2^{d}} \sum_{i=1}^d  \sum_{v\in \cV}   \EE_{v}\big[ \ind\{ \hat\pi(x_i) \neq \pi_{v}^*(x_i)  \big] \notag \\ 
&\geq  \frac{\delta}{d\cdot 2^{d}} \sum_{i=1}^d \sum_{v\in\cV\colon v_i=1} \big( 1- \textrm{TV}(\PP_v, \PP_{M_i[v]})\big) \\ 
&\geq \frac{\delta}{d\cdot 2^{d}} \sum_{i=1}^d \sum_{v\in\cV\colon v_i=1} \frac{1}{2} \exp\big(-\textrm{KL}(\PP_{v}\|\PP_{M_i[v]})\big),
\# 
where $\textrm{KL} (\PP_{v }  \| \PP_{M_i[v]} )$
denotes the KL divergence of $\{X_t,Y_t,A_t\}$ under $R_{v }$ versus under $R_{M_i[v]}$.
Again, following the arguments in the proof of Theorem~\ref{thm:lower}, we know that  
\$
\log\bigg(\frac{\PP_{v} }{\PP_{M_i[v]} }(x_j,y,a)\bigg)
&= (2y-1) \log \bigg( \frac{1/2+\delta}{1/2-\delta} \bigg) \ind\{a=f_2(x_j),j=i\}.
\$
Therefore, by our construction, the KL divergence can be bounded as 
\#\label{eq:kl_bd_adapt}
\textrm{KL} (\PP_{v}  \| \PP_{v'}) &\leq \sum_{t=1}^T \EE\big[  \PP_v\big(A_t=f_2(x_i),X_t=x_i\big) 
\EE_v[2Y-1\given A_t,X_t] \big]\log \bigg( \frac{1/2+\delta}{1/2-\delta} \bigg) \\ 
&\leq  \frac{12\delta^2 \bar{c}}{d}\cdot \sum_{t=1}^T t^{-\gamma} 
\leq \frac{12\delta^2 \bar{c}}{d}\cdot \frac{T^{1-\gamma}}{1-\gamma},
\#
where we use the fact that $\log(1+2\delta)\leq 2\delta$ 
and $\log(1-2\delta) \geq -4\delta$ for $\delta\in(0,1/4)$.  
Combining the above results, we have 
\$ 
\sup_{R_v\in \cR_0(\bar{c},\gamma,T)} \EE_{v}\big[ \cL(\hat\pi;\cC_v,\Pi)  \big] 
\geq \frac{\delta}{d\cdot 2^d} \sum_{i=1}^d \sum_{v\in \cV\colon v_i=1}
\frac{1}{2} \exp\bigg( - \frac{12\delta^2 \bar{c}}{d}\cdot \frac{T^{1-\gamma}}{1-\gamma}\bigg)
= \delta \cdot \exp\bigg( - \frac{12\delta^2 \bar{c}}{d}\cdot \frac{T^{1-\gamma}}{1-\gamma}\bigg).
\$
Finally, taking $\delta=\sqrt{\frac{(1-\gamma)d}{24\cdot \bar{c}\cdot T^{1-\gamma}}}\le 1/4$, we arrive at 
\#
\sup_{R_v\in \cR (\bar{c},\gamma,T)} \EE_{v}\big[ \cL(\hat\pi;\cC_v,\Pi)  \big] 
&\geq 
\sup_{R_v\in \cR_0(\bar{c},\gamma,T)} \EE_{v}\big[ \cL(\hat\pi;\cC_v,\Pi)  \big]\\ 
&\geq \frac{\exp(-1/2)}{\sqrt{24 \bar{c}/(1-\gamma)}}\sqrt{\frac{d}{ T^{1-\gamma}}}\geq 0.12 \sqrt{\frac{(1-\gamma) d }{\bar{c} T^{1-\gamma}}},
\#
which concludes the proof of Theorem\ref{thm:lower_adapt}. 
\end{proof}

\section{Proof of supporting lemmas}

\subsection{Proof of Lemma~\ref{lemma:copy_moments}}
\label{app:lem_copy_moment}

\begin{proof}[Proof of Lemma~\ref{lemma:copy_moments}]
  Fix any $\pi\in \Pi$ and any $\eta > 0$. 
  Since $(Y_t',A_t')$ are independent copies of $(Y_t,A_t)$ conditional on $X_t$, 
  it is easy to check that 
  \#\label{eq:unbiased}
  \EE\big[ \,\hat\Gamma_t'(\pi) - \mu(X_t,\pi(X_t))\biggiven \{X_t,A_t,Y_t\}_{t=1}^T\big] = 0,
  \#
  and $\{\hat\Gamma_t'(\pi)\}_{t=1}^T$ 
  are mutually independent conditional on $\{X_t,A_t,Y_t\}_{t=1}^T$. 
  Then by Chebyshev's inequality, 
  \$
  & \PP\bigg(\Big|\sum^T_{t=1} \hat \Gamma_t'(\pi) -  \mu(X_t,\pi(X_t))\Big| \ge \eta T \bigggiven
  \{X_t,A_t,Y_t\}_{t=1}^T\bigg) \\ 
  & \leq \frac{1}{\eta^2 T^2} 
  \EE\Bigg[\bigg(\sum_{t=1}^T \big\{\hat \Gamma'_t(\pi) - \mu(X_t,\pi(X_t))\big\}\bigg)^2
  \bigggiven \{X_t,A_t,Y_t\}_{t=1}^T\Bigg] \\ 
  &=  \frac{1}{\eta^2 T^2}\sum^T_{t=1}\EE\Big[\big\{ \hat \Gamma'_t(\pi) - \mu(X_t,\pi(X_t))\big\}^2
  \Biggiven \{X_t,A_t,Y_t\}_{t=1}^T\Big] .
  \$
  In view of~\eqref{eq:unbiased}, 
  each term in the last summation can be bounded as 
  \#\label{eq:var_bound}
  & \EE\Big[\big\{\hat \Gamma'_t(\pi) - \mu(X_t,\pi(X_t))\big\}^2
  \biggiven \{X_t,A_t,Y_t\}_{t=1}^T\Big]= \Var\big( \hat \Gamma'_t(\pi) 
  \biggiven X_t \big)\\
  & = \Var\bigg( \frac{\ind\{A_t' = \pi(X_t)\}}{e(X_t,\pi(X_t))}
   \big(Y_t' - \hmu(X_t,\pi(X_t))\big)
  \bigggiven X_t \bigg) \\ 
  & \leq \EE\bigg[  \frac{\ind\{A_t' = \pi(X_t)\}}{e(X_t,\pi(X_t))^2}
   \big(Y_t' - \hmu(X_t,\pi(X_t))\big)^2
  \bigggiven X_t \bigg] \leq  \frac{1}{e(X_t,\pi(X_t))},
  \#
  where the last inequality uses  that fact that 
  $Y_t'\in[0,1]$ and $\hat\mu(X_t,\pi(X_t))\in[0,1]$. 
  We then have 
  \$
  \PP\bigg(\Big|\sum^T_{t=1} \hat \Gamma_t'(\pi) - \mu(X_t,\pi(X_t))\Big| \ge \eta T \bigggiven
  \{X_t,A_t,Y_t\}_{t=1}^T\bigg) 
  \leq \frac{1}{\eta^2 T^2}\sum_{t=1}^T  \frac{1}{e(X_t,\pi(X_t))}
  = \frac{\vp(\pi)^2}{\eta^2}.
  \$ 
  
  Taking $\eta = 2 \vp(\pi)$ (this is measurable 
  with respect to $\{X_t,A_t,Y_t\}_{t=1}^T$)
  in the above inequality leads to  
  \#
  \label{eq:copy_bound_1}
  \PP\bigg( \frac{1}{T} \Big|\sum^T_{t=1} \hgamma'_t(\pi) - \mu(X_t,\pi(X_t))\Big| 
  \ge {2}\vp(\pi) \Biggiven \{X_t,A_t,Y_t\}_{t=1}^T\bigg)
  \le \frac{1}{4}.
  \#
  We then proceed to prove the second inequality.
  For any $\pi \in \Pi$ and any $\zeta > 0$, 
  \$
  & \PP\bigg(\sum^T_{t=1} \frac{\ind\{A'_t = \pi(X_t)\}}{e(X_t,\pi(X_t))^2} 
  \ge \zeta + \vp(\pi)^2T^2 \bigggiven \{X_t,A_t,Y_t\}_{t=1}^T\bigg)\\
  &\leq \PP\Bigg(\sum^T_{t=1} \frac{\ind\{A'_t = \pi(X_t)\}}{e(X_t,\pi(X_t))^2}
    - \EE\bigg[\frac{\ind\{A_t' = \pi(X_t)\}}{e(X_t,\pi(X_t))^2} \bigggiven \{X_t,A_t,Y_t\}_{t=1}^T\bigg]
  \ge \zeta \Bigggiven \{X_t,A_t,Y_t\}_{t=1}^T\Bigg)\\
  &\leq \frac{1}{\zeta^2} \sum^T_{t=1} \Var\Bigg(\frac{\ind\{A_t' = \pi(X_t)\}}{e(X_t,\pi(X_t))^2} 
  \Bigggiven \{X_t,A_t,Y_t\}^T_{t=1}\Bigg)\\
  & \leq \frac{1}{\zeta^2} \sum^T_{t=1} \EE\Bigg[\bigg(\frac{\ind\{A_t' = \pi(X_t)\}}{e(X_t,\pi(X_t))^2}\bigg)^2
  \Bigggiven \{X_t,A_t,Y_t\}^T_{t=1}\Bigg]\\
  &=  \frac{1}{\zeta^2} \sum^T_{t=1} \frac{1}{e(X_t,\pi(X_t))^3}
  = \frac{\vh(\pi)^4T^4}{\zeta^2},
  \$
  where the second  inequality is 
  because $\{A_t'\}_{t=1}^T$ are mutually independent
  conditional on $\{X_t,A_t,Y_t\}_{t=1}^T$. Taking 
  $\zeta = 2 \vh(\pi)^2T^2$, we have
  \$
  \PP\Big(V_{\rm s}'(\pi)^2 \ge 4\cdot\max \big(\vp(\pi)^2,  \vh(\pi)^2\big)
  \Biggiven \{X_t,A_t,Y_t\}_{t=1}^T\Big) \le \frac{1}{4},
  \$
  which concludes the proof of Lemma~\ref{lemma:copy_moments}.
  \end{proof}

\subsection{Proof of Lemma~\ref{lem:bd_events}}
\label{app:proof_bd_events}
\begin{proof}[Proof of Lemma~\ref{lem:bd_events}]
Fix $\pi \in \Pi$ and $\eta >0$. By Chebyshev's inequality,
we have
\$
&\PP\bigg(\frac{1}{T} \Big|\sum^T_{t=1}\hgamma'_t(\pi) - \mu(X_t,\pi(X_t))\Big| \ge \eta \bigggiven \xat\bigg)\\
\le ~&\frac{1}{\eta^2T^2} \EE\bigg[\Big(\sum^T_{t=1} \hgamma_t'(\pi) - \mu(X_t,\pi(X_t))\Big)^2 \bigggiven \xat\bigg].
\$
Since $\{A_t',Y_t'\}_{t=1}^T$ are mutually independent 
conditional on $\xat$, the right-hand side of the above
can be written as
\$
& \frac{1}{\eta^2 T^2}\sum^T_{t=1} 
\EE\Big[\big(\hgamma_t'(\pi) - \mu(X_t,\pi(X_t))\big)^2 \Biggiven \xat\Big] \\
  = \, & \frac{1}{\eta^2 T^2} \sum^T_{t=1}
\Var\Big(\frac{\ind\{A_t' = \pi(X_t)\}}
{e_t(X_t,\pi(X_t) \given \cH_t)} \Biggiven \xat\Big)\\
  \le \, & \frac{1}{\eta^2 T^2} \sum^T_{t=1}
\EE\Big[\frac{\ind\{A_t' = \pi(X_t)\}}
{e_t(X_t,\pi(X_t) \given \cH_t)^2} \Biggiven \xat\Big] \\
\le \, & \frac{1}{\eta^2 T^2} \sum^T_{t=1} 
\frac{1}{e_t(X_t,\pi(X_t)\given \cH_t)}
\le \frac{\vp(\pi)^2}{\eta^2},
\$
where the step is because
$\EE[\hgamma_t(\pi) \given \xat] = \mu(X_t,\pi(X_t))$
and the last step uses the fact that 
$g_t(x,\pi(x)) \le e_t(x,\pi(x) \given \cH_t)$.
Taking $\eta = 2\vp(\pi)$, we have
\$
\PP\bigg(\frac{1}{T} \Big|\sum^T_{t=1}\hgamma'_t(\pi) - \mu(X_t,\pi(X_t))\Big| 
\ge 2\vp(\pi) \bigggiven \xat\bigg)
\le \frac{1}{4}.
\$
We proceed to bound the conditional probability of $\cE_2$.
For any $\pi \in \Pi$ and $\zeta >0$ that may depend on $\xat$, 
since $(\zeta+\vp(\pi))^2\geq \zeta^2 + \vp(\pi)^2$, we have 
\#
\label{eq:mg_moment_1}
\notag & \PP\big(\vs'(\pi) \ge \zeta + \vp(\pi) \biggiven \xat\big)\\
& \leq  \PP\Bigg(\sum^T_{t=1} \frac{\ind\{A_t' = \pi(X_t)\}}{e_t(X_t,\pi(X_t) \given \cH_t)^2} 
\ge \zeta^2 T^2 + T^2 \vp(\pi)^2 \Bigggiven \xat\Bigg)
\#
Note that by construction
\$
\EE\bigg[\frac{\ind\{A_t' = \pi(X_t)\}}{e_t(X_t,\pi(X_t)\given \cH_t)^2} \bigggiven \xat \bigg]
= \sum^T_{t=1} \frac{1}{e_t(X_t,\pi(X_t)\given \cH_t)} = T^2 \vp(\pi)^2
.
\$
Then by Chebyshev's inequality, we have 
\$
\eqref{eq:mg_moment_1}
\, &\le \PP\Bigg(\sum^T_{t=1} \frac{\ind\{A_t' = \pi(X_t)\}}{e_t(X_t,\pi(X_t))^2} 
  - \EE\bigg[ \frac{\ind\{A_t'=\pi(X_t)\}}{e_t(X_t,\pi(X_t))^2} \Biggiven \xat\bigg] 
\ge \zeta^2 T^2 \Bigggiven \xat\Bigg)\\
\le \, & \frac{1}{\zeta^4 T^4}
\cdot \Var\bigg(\sum^T_{t=1} \frac{\ind\{A_t' = \pi(X_t)\}}{e_t(X_t,\pi(X_t))^2} \Biggiven \xat\bigg) \\
= \, & \frac{1}{\zeta^4T^4}
\cdot \sum^T_{t=1} \Var\bigg(\frac{\ind\{A_t' = \pi(X_t)\}}{e_t(X_t,\pi(X_t))^2} \Biggiven \xat\bigg)\\
\le \, & \frac{1}{\zeta^4T^4} \cdot \sum^T_{t=1} \frac{1}{e_t(X_t,\pi(X_t))^3}
\le \frac{\vh(\pi)^4}{\zeta^4},
\$
where in the next-to-last step we use the independence 
between $\{A_t'\}_{t=1}^T$ conditional on $\xat$.
Taking $\zeta = 2\vh(\pi) $, we have
\$
\PP\Big(\vs'(\pi)\ge 2 \cdot \max \big\{\vh(\pi), \vp(\pi) \big\} \Biggiven \xat\Big)
\le \frac{1}{4},
\$
for any $\pi\in\Pi$. Now that $\pi^\dagger$ 
only depends on $\xat$, hence the two inequalities also hold for $\pi^\dagger$. 
Therefore, we conclude the proof of Lemma~\ref{lem:bd_events}. 
\end{proof}

\section{Auxiliary lemmas}

The following  Natarajan's lemma  characterizes the 
number of distinct realizations of $\pi(x)$ for $\pi\in\Pi$, 
which allows us to turn the uniform bound over $\Pi$ 
to that over a  finite set with polynomially growing  size. 

\begin{lemma}[\cite{natarajan1989learning}]
\label{lemma:nat_lemma}
Let $\Pi \defn \{\pi: S \mapsto \cA\}$ for a finite 
set $S\in \cX^{|S|}$. Then
\$
|\Pi| \le |\cX|^{\ndim(\Pi)} |\cA|^{2\ndim(\Pi)}.
\$
\end{lemma}

The following self-normalized concentration inequality is 
from Corollary 2.2 of~\cite{de2004self}. 
\begin{lemma}[Self-normalized concentration]
\label{lemma:self_normal}
Suppose two random variables $A$ and $B>0$ satisfies that 
$\EE[\exp(\lambda A-\lambda^2 B^2/2)]\leq 1$ for all $\lambda\in \RR$. 
Then for all $x\geq \sqrt{2}$ and all $y>0$, it holds that 
\$
\PP\bigg( \frac{|A|}{\sqrt{(B^2+y)(1+0.5\log(1 + B^2/y))}} \geq x \bigg) \leq \exp(-x^2/2).
\$
\end{lemma}


The lemma below describes Hoeffding's inequality~\citep{hoeffding1994probability}.
\begin{lemma}[Hoeffding's inequality]
\label{lem:hoeffding}
Let $X_1,\ldots,X_n$ be independent random variables
on $\RR$ such that $a_i \le X_i \le b_i$ almost surely.
If $S_n = \sum^n_{i=1}X_i$, then for all $x > 0$,
\$
\PP(S_n - \EE[S_n] \ge x) \le \exp\Big(-\frac{2x^2}{\sum^n_{i=1}(b_i - a_i)^2}\Big).
\$
\end{lemma}

The next lemma is adapted from 
Freedman's inequality~\citep{freedman1975tail}, with Bernstein's inequality as a special case. 

\begin{lemma}[Freedman's inequality]
\label{lem:bernstein}
Consider a real-valued martingale $\{Y_k\}_{k\geq 0}$ 
adapted to a filtration $\{\cF_k\}$, and its 
different sequence $\{X_k = Y_k-Y_{k-1}\}_{k\geq 1}$. Assume 
the difference sequence is uniformly bounded: 
$X_k\leq R$ almost surely for $k\geq 1$. 
Define the predictable quadratic variation process of the martingale: 
$M_k:=\sum_{j=1}^k \EE[X_j^2\given \cF_{j-1}]$. 
Then for all $x\geq 0$ and $\sigma^2>0$, we have 
\$
\PP\big( \exists~k\geq 0\colon Y_k \geq x~~\textrm{and}~~ M_k\leq \sigma^2   \big)
\leq \exp\bigg(  \frac{-x^2/2}{\sigma^2 + Rx/3} \bigg).
\$
As a special case, if  $Y_0=0$ 
and $Y_k = \sum_{j=1}^{k}X_k$ for $k=1,\dots,n$, 
where $\{X_k\}_{k=1}^n$ are independent mean-zero random variables, 
it becomes Bernstein's inequality with $\sigma^2 = \sum_{k=1}^n \Var(X_k)$. 
\end{lemma}

\end{document}